\documentclass[11pt]{article}

\usepackage{fullpage}
\usepackage{microtype}
\usepackage{graphicx}
\usepackage{booktabs} % for professional tables

\usepackage{amsmath,amsfonts,bm}

\def\eqref#1{equation~\ref{#1}}
\def\1{\bm{1}}

\def\eps{{\epsilon}}

\def\vx{{\bm{x}}}

\DeclareMathAlphabet{\mathsfit}{\encodingdefault}{\sfdefault}{m}{sl}
\SetMathAlphabet{\mathsfit}{bold}{\encodingdefault}{\sfdefault}{bx}{n}

\newcommand{\E}{\mathbb{E}}

\newcommand{\R}{\mathbb{R}}

\newcommand{\KL}{D_{\mathrm{KL}}}
\newcommand{\Var}{\mathrm{Var}}

\DeclareMathOperator*{\argmax}{arg\,max}
\DeclareMathOperator*{\argmin}{arg\,min}

\usepackage{inputenc}
\usepackage{hyperref}
\usepackage{url}
\usepackage{natbib}
\usepackage{algorithm}
\usepackage[noend]{algorithmic}
\usepackage{amsmath,amsthm,amsfonts,amssymb}
\usepackage{amsmath}
\usepackage{hyperref}
\usepackage{color}
\hypersetup{
	colorlinks = true,
	citecolor = blue,
	linkcolor = blue
}
\usepackage{mathrsfs}
\usepackage{enumitem}
\usepackage{bm}
\usepackage{multirow}
\usepackage{booktabs}
\usepackage{makecell}
\usepackage{graphicx}
\usepackage{caption}
\usepackage{thmtools}
\usepackage{thm-restate}
\usepackage{dsfont}
\usepackage{cleveref}

\newcommand{\defeq}{\mathrel{\mathop:}=}

\newcommand{\trans}{^{\top}}

\newcommand{\norm}[1]{\|{#1} \|}

\newcommand{\Exp}[2]{\mathop\mathbb{E}_{#1}\left[#2\right]}

\newcommand{\pr}[1]{\mathop\mathbb{P}\left[#1\right]}
\newcommand{\prob}[2]{\mathop\mathbb{P}_{#1}\left[#2\right]}

\newcommand{\tlO}{\mathcal{\tilde{O}}}

\newcommand{\N}{\mathbb{N}}
\newcommand{\cF}{\mathcal{F}}
\newcommand{\cM}{\mathcal{M}}

\newcommand{\cD}{\mathcal{D}}

\newcommand{\cH}{\mathcal{H}}
\newcommand{\cT}{\mathcal{T}}

\newcommand{\cG}{\mathcal{G}}

\newcommand{\cn}{\kappa}

\newcommand{\vsigma}{{\bm\sigma}}

\renewcommand{\eps}{\varepsilon}
\newcommand{\one}[1]{\mathds{1}\left[#1\right]}
\renewcommand{\KL}[2]{KL\left( #1,#2\right)}

\def\vx{{\bm{x}}}

\usepackage{times}

\usepackage{subcaption}

\usepackage{diagbox}

\usepackage{tikz} 
\usetikzlibrary{shapes.misc}

\usepackage{hyperref}

\newtheorem{theorem}{Theorem}[section]
\newtheorem{lemma}[theorem]{Lemma}

\declaretheoremstyle[bodyfont=\normalfont]{rm}

\newtheorem{proposition}[theorem]{Proposition}
\theoremstyle{definition}
\newtheorem{definition}[theorem]{Definition}

\newtheorem{assumption}{Assumption}

\newcommand{\cS}{\mathcal{S}}
\newcommand{\cA}{\mathcal{A}}
\newcommand{\cX}{\mathcal{X}}

\newcommand{\Data}{\mathcal{D}}
\newcommand{\DData}{\widetilde{\mathcal{D}}}
\newcommand{\bfun}{f}
\newcommand{\empbfun}{\hat{f}}
\newcommand{\bg}{{g}}

\newcommand{\bFun}{\cF}
\newcommand{\bG}{{\cG}}
\renewcommand{\P}{\mathbb{P}}
\newcommand{\ind}{\mathds{1}}
\newcommand{\Berr}{\mathcal{E}}
\newcommand{\empR}{\widehat{\mathcal{R}}}
\newcommand{\popR}{\mathcal{R}}

\newcommand{\bcT}{{\mathcal{T}}}
\newcommand{\bmu}{{\mu}}

\newcommand{\FQI}{\mathsf{FQI}}

\newcommand{\best}{\dagger}

\newcommand{\empBerr}{\hat{L}_{\text{B}}}
\newcommand{\empDS}{\hat{L}_{\text{DS}}}
\newcommand{\empMM}{\hat{L}_{\text{MM}}}

\newcommand{\Uniform}{\text{Unif}}

\usepackage{thmtools} 
\usepackage{thm-restate}

\begin{document}

\begin{center} {\bf \LARGE Risk Bounds and Rademacher Complexity in \\ Batch Reinforcement Learning}

\vspace{1.5em}

\begin{tabular}{ccc}
	Yaqi Duan & Chi Jin & Zhiyuan Li \\
	Princeton University & Princeton University & Princeton University \\
	\texttt{yaqid@princeton.edu} & \texttt{chij@princeton.edu} & \texttt{zhiyuanli@cs.princeton.edu}
\end{tabular}

\vspace{1em}

\today

\end{center}

\vspace{1em}

\begin{abstract}
This paper considers batch Reinforcement Learning (RL) with general value function approximation.
Our study investigates the minimal assumptions to reliably estimate/minimize Bellman error, and characterizes the generalization performance by (local) Rademacher complexities of general function classes, which makes initial steps in bridging the gap between statistical learning theory and batch RL. Concretely, we view the Bellman error as a surrogate loss for the optimality gap, and prove the followings: (1) In double sampling regime, the excess risk of Empirical Risk Minimizer (ERM) is bounded by the Rademacher complexity of the function class. (2) In the single sampling regime, sample-efficient risk minimization is not possible without further assumptions, regardless of algorithms. However, with completeness assumptions, the excess risk of FQI and a minimax style algorithm can be again bounded by the Rademacher complexity of the corresponding function classes.  (3) Fast statistical rates can be achieved by using tools of local Rademacher complexity. Our analysis covers a wide range of function classes, including finite classes, linear spaces, kernel spaces, sparse linear features, etc.
\end{abstract}

% !TEX root = main.tex

\section{Introduction}
\label{sec:intro}

Statistical learning theory, since its introduction in the late 1960’s, has become one of the most important frameworks in machine learning, to study problems of inference or function estimation from a given collection of data \citep{hastie2009elements,vapnik2013nature,james2013introduction}. The development of statistical learning has led to a series of new  popular algorithms including support vector machines \citep{cortes1995support,suykens1999least}, boosting \citep{freund1996experiments,schapire1999brief}, as well as many successful applications in fields such as computer vision \citep{szeliski2010computer,forsyth2012computer}, speech recognition \citep{juang1991hidden,jelinek1997statistical}, and bioinformatics \citep{baldi2001bioinformatics}.

Notably, in the area of supervised learning, a considerable amount of effort has been spent on obtaining sharp risk bounds. These are valuable, for instance, in the problem of model selection---choosing a model of suitable complexity. Typically, these risk bounds characterize the excess risk---the suboptimality of the learned function compared to the best function within a given function class, via proper complexity measures of that function class. After a long line of extensive research \citep{vapnik2013nature,vapnik2015uniform,bartlett2005local,bartlett2006convexity}, risk bounds are proved under very weak assumptions which do not require realizability---the prespecified function class contains the ground-truth. 
The complexity measures for general function classes have also been developed, including but not limited to metric entropy \citep{dudley1974metric}, VC dimension \citep{vapnik2015uniform} and Rademacher complexity \citep{bartlett2002rademacher}. (See e.g. \cite{wainwright2019high} for a textbook review.)

Concurrently, batch reinforcement learning \citep{lange2012batch,levine2020offline}---a branch of Reinforcement Learning (RL) that learns from offline data, has been independently developed. This paper considers the value function approximation setting, where the learning agent aims to approximate the optimal value function from a restricted function class that encodes the prior knowledge. Batch RL with value function approximation provides an important foundation for the empirical success of modern RL, and leads to the design of many popular algorithms such as DQN \citep{mnih2015human} and Fitted Q-Iteration with neural networks \citep{riedmiller2005neural,fan2020theoretical}.

Despite being a special case of supervised learning, batch RL also brings several unique challenges due to the additional requirement of learning the rich temporal structures within the data. Addressing these unique challenges has been the main focus of the field so far \citep{levine2020offline}. Consequently, the field of statistical learning and batch RL have been developed relatively in parallel. In contrast to the mild assumptions required and the generic function class allowed in classical statistical learning theory, a majority of batch RL results \citep{munos2008finite,antos2008learning,lazaric2012finite,chen2019information} remain under rather strong assumptions which rarely hold in practice, and are applicable only to a restricted set of function classes. This raises a natural question: can we bring the rich knowledge in statistical learning theory to advance our understanding in batch RL?

This paper makes initial steps in bridging the gap between statistical learning theory and batch RL. We investigate the minimal assumptions required to reliably estimate or minimize the Bellman error, and characterize the generalization performance of batch RL algorithms by (local) Rademacher complexities of general function classes. Concretely, we establish conditions when the Bellman error can be viewed as a surrogate loss for the optimality gap in values. 
We then bound the excess risk measured in Bellman errors. We prove the followings:
\begin{itemize}
\item In the double sampling regime, the excess risk of a simple Empirical Risk Minimizer (ERM) is bounded by the Rademacher complexity of the function class, under almost no assumptions. 
\item In the single sampling regime, without further assumptions, no algorithm can achieve small excess risk in the worse case unless the number of samples scales up polynomially with respect to the number of states. 
\item In the single sampling regime, under additional completeness assumptions, the excess risks of Fitted Q-Iteration (FQI) algorithm and a minimax style algorithm can be again bounded by the Rademacher complexity of the corresponding function classes. 
\item Fast statistical rates can be achieved by using tools of local Rademacher complexity.
\end{itemize}
Finally, we specialize our generic theory to concrete examples, and show that our analysis covers a wide range of function classes, including finite classes, linear spaces, kernel spaces, sparse linear features, etc.

% !TEX root = main.tex

\subsection{Related Work}

We restrict our discussions in this section to the RL results under function approximation.

\paragraph{Batch RL}

There exists a stream of literature regarding finite sample guarantees for batch RL with value function approximation. Among the works, fitted value iteration \citep{munos2008finite} and policy iteration \citep{antos2008learning,farahmand2008regularized,lazaric2012finite,farahmand2016regularized,le2019batch} are canonical and popular approaches.
When using a linear function space, the sample complexity for batch RL is shown to depend on the dimension \citep{lazaric2012finite}.
When it comes to general function classes, several complexity measures of function class such as metric entropy and VC dimensions have been used to bound the performance of fitted value iteration and policy iteration \citep{munos2008finite,antos2008learning, farahmand2016regularized}.

Throughout the existing theoretical studies of batch RL, people commonly use concentrability, realizability and completeness assumptions to prove polynomial sample complexity.
\citet{chen2019information} justify the necessity of low concentrability and hold a debate on realizability and completeness. \citet{xie2020batch} develop an algorithm that only relies on the realizability of optimal Q-function and circumvents completeness condition. However, they use a stronger concentrability assumption and the error bound has a slower convergence rate.
While the analyses in \citet{chen2019information} and \citet{xie2020batch} are restricted to discrete function classes with a finite number of elements,
\citet{wang2020statistical} investigate value function approximation with linear spaces. It is shown that data coverage and realizability conditions are not sufficient for polynomial sample complexity in the linear case.

\paragraph{Off-policy evaluation}
Off-policy evaluation (OPE) refers to the estimation of value function given offline data \citep{precup2000eligibility,precup2001off,xie2019towards,uehara2020minimax,kallus2020double,yin2020near,uehara2021finite}, which can be viewed as a subroutine of batch RL.
Combining OPE with policy improvement leads to policy-iteration-based or actor-critic algorithms \citep{dann2014policy}.
OPE is considered as a simpler problem than batch RL and
its analyses cannot directly translate to guarantees in batch RL.

\paragraph{Online RL}
RL in online mode is in general a more difficult problem than batch RL. The role of value function approximation in online RL remains largely unclear. It requires better tools to measure the capacity of function class in an online manner.
In the past few years, there are some investigations in this direction, including using Bellman rank \citep{jiang2017contextual} and Eluder dimension \citep{wang2020reinforcement} to characterize the hardness of RL problem.

\subsection{Notation} For any integer $K > 0$, let $[K]$ be the collection of $1,2,\ldots,K$. We use $\ind[\cdot]$ to denote the indicator function. For any function $q(\cdot)$ and any measure $\rho$ over the domain of $q$, we define norm $\norm{\cdot}_{\rho}$ where $\norm{q}^2_{\rho} \defeq \E_{x\sim \rho} q^2(x)$. Let $\rho_1$ be a measure over $\mathcal{X}_1$ and $\rho_2(\cdot \mid x_1)$ be a conditional distribution over $\mathcal{X}_2$. Define $\rho_1 \times \rho_2$ as a joint distribution over $\mathcal{X}_1 \times \mathcal{X}_2$, given by $(\rho_1 \times \rho_2)( x_1, x_2) := \rho_1(x_1) \rho_2(x_2 \mid x_1)$.
For any finite set $\cX$, let $\Uniform(\cX)$ define a uniform distribution over $\cX$.

% !TEX root = main.tex

\section{Preliminaries}
\label{sec:prelim}

We consider the setting of episodic Markov decision process ${\rm MDP}(\cS, \cA, H, \P, r)$, where $\cS$ is the set of states which possibly has infinitely many elements; $\cA$ is a finite set of actions with $|\cA| = A$; $H$ is the number of steps in each episode; $\P_h ( \cdot \mid s, a) $ gives the distribution over the next state if action $a$ is taken from state $s$ at step $h\in [H]$; and $r_h \colon \cS \times \cA \to [0,1]$ is the deterministic reward function at step $h$. \footnote{While we study deterministic reward functions for notational simplicity, our results generalize to randomized reward functions. Note that we are assuming that rewards are in $[0,1]$ for normalization.}

In each episode of an MDP, we start with \textbf{a fixed initial state $s_1$}. Then, at each step
$h \in [H]$, the agent observes state $s_h \in \cS$, picks an action
$a_h \in \cA$, receives reward $r_h(s_h, a_h)$, and then transitions
to the next state $s_{h+1}$, which is drawn from the distribution
$\P_h(\cdot \mid s_h, a_h)$. Without loss of generality, we assume there is a terminating state $s_{\text{end}}$ which the environment will \emph{always} transit to % $s_{\text{end}}$ 
at step $H+1$, and the episode terminates when $s_{\text{end}}$ is reached.

A (non-stationary, stochastic) policy $\pi$ is a collection of $H$
functions $\big\{ \pi_h: \cS \rightarrow \Delta_\cA \big\}_{h\in
	[H]}$, where $\Delta_\cA$ is the probability simplex over action set
$\cA$. We denote $\pi_h(\cdot \mid s)$ as the action
distribution for policy $\pi$ at state $s$ and time $h$.  Let
$V^\pi_h \colon \cS \to \mathbb{R}$ denote the value function at
step $h$ under policy $\pi$, which % so that $V^\pi_h(s)$ 
gives the expected
sum of remaining rewards received under policy $\pi$, starting from
$s_h = s$, until the end of the episode. That is,
\begin{equation*}
V^\pi_h(s) \defeq \E_{\pi}\left[ \sum_{h' = h}^H r_{h'}(s_{h'}, a_{h'}) \Biggm| s_h = s\right] .
\end{equation*}
Accordingly, the action-value function $Q^\pi_h:\cS \times \cA \to \mathbb{R}$ at step $h$ is defined as,
\begin{equation*}
Q^\pi_h(s,a) \defeq \E_{\pi}\left[ \sum_{h' = h}^H r_{h'}(s_{h'}, a_{h'}) \Biggm| s_h = s, a_h = a\right].
\end{equation*}

Since the action spaces, and the horizon, are all finite,
there always exists (see, e.g., \cite{puterman2014markov}) an optimal
policy $\pi^\star$ which gives the optimal value $V^\star_h(s) =
\sup_{\pi} V_h^\pi(s)$ for all $s\in \cS$ and $h\in [H]$.  

For notational convenience, we take shorthands $\P_h, \P^{\pi}_h, \P^\star_h$ as follows, where $(s, a)$ is the state-action pair for the current step, while $(s', a')$ is the state-action pair for the next step,
\begin{align*} 
[\P_h V](s, a) \defeq & \E \big[V(s') \bigm| s,a \big], \\
[\P^\pi_h Q](s, a) \defeq & \E_{\pi} \big[Q(s', a') \big| s,a\big],  \\ 
[\P^\star_h Q](s, a) \defeq & \E \big[ \max_{a'} Q(s', a') \bigm| s,a \big]. 
\end{align*}
We further define Bellman operators $\cT_h^\pi, \cT_h^\star: \R^{\cS \times \cA} \rightarrow \R^{\cS \times \cA}$  for $h\in[H]$ as
\begin{align*}
& (\cT_h^\pi Q)(s, a) \defeq (r_h + \P_h^\pi Q)(s, a), \\ & (\cT_h^\star Q)(s, a) \defeq (r_h + \P_h^\star Q)(s, a).
\end{align*}
Then the Bellman equation and the Bellman optimality equation can be written as:
\begin{equation*}
Q^\pi_h(s, a) = (\cT_h^\pi Q^\pi_{h+1})(s, a), \ Q^\star_h(s, a) = (\cT_h^\star Q^\star_{h+1})(s, a).
\end{equation*}

The objective of RL is to find a near-optimal policy, where the sub-optimality is measured by $V_1^{\star}(s_1) - V_1^{\pi}(s_1)$. Accordingly, we have the following definition of $\epsilon$-optimal policy.
\begin{definition}[$\epsilon$-optimal policy]
	We say a policy $\pi$ is \textbf{$\epsilon$-optimal} if $V_1^{\star}(s_1) - V_1^{\pi}(s_1) \le \epsilon$.
\end{definition}

\subsection{(Local) Rademacher complexity}
In this paper, we leverage Rademacher complexity to characterize the complexity of a function class.
For a generic real-valued function space $\cF \subseteq \R^{\cX}$ and $n$ fixed data points $X = \{x_1, \ldots, x_n\} \in \cX^n$, the empirical Rademacher complexity is defined as
\begin{equation*}
\empR_{X}(\cF) := \E\Bigg[\sup_{f \in \cF} \frac{1}{n} \sum_{i=1}^n  \sigma_{i} f(x_{i}) \Biggm| X \Bigg],
\end{equation*}
where $\sigma_i \sim {\rm Uniform}(\{-1,1\})$ are {\it i.i.d.} Rademacher random variables and 
the expectation is taken with respect to the uncertainties in $\{\sigma_i\}_{i=1}^n$.
Let $\rho$ be the underlying distribution of $x_i$. We further define a population Rademacher complexity $\popR^{\rho}_n(\cF) := \E_{\rho} [\empR_{X}(\cF)]$ with expectation taken over data samples $X$. 
Intuitively, $\popR^{\rho}_{n}(\cF)$ measures the complexity of $\cF$ by the extent to which functions in the class $\cF$ correlate with random noise $\sigma_{i}$.

This paper further uses the tools of local Rademacher complexity to obtain results with fast statistical rate. For a generic real-valued function space $\cF \subseteq \R^{\cX}$, and data distribution $\rho$. Let $T$ be a functional $T:\cF \rightarrow \R^{+}$, we study the local Radmacher complexity in the form of 
\begin{equation*}
\popR^{\rho}_n(\{f\in \cF \mid T(f) \le r\}).
\end{equation*}
A crucial quantity that appears in the generalization error bound using local Rademacher complexity is the critical radius \citep{bartlett2005local}. We define as follows.

\begin{definition}[Sub-root function]
	A function $\psi: \R^+ \rightarrow \R^+$ is sub-root if it is nondecreasing, and $r\rightarrow \psi(r)/\sqrt{r}$ is nonincreasing for $r >0$.
\end{definition}

\begin{definition}[Critical radius of local Radmacher complexity] \label{def:critrad} The critical radius of the local Radmacher complexity $\popR^\rho_n(\{f\in \cF \mid T(f) \le r\})$ is the infimum of the set $\mathfrak{B}$, where set $\mathfrak{B}$ is defined as follows: for any $r^\star \in \mathfrak{B}$, there exists a sub-root function $\psi$ such that $r^\star$ is the fixed point of $\psi$, and for any $r \ge r^\star$ we have 
	\begin{equation} \label{eq:critical_radius}
	\psi(r) \ge \popR^{\rho}_n(\{f\in \cF \mid T(f) \le r\}).
	\end{equation}
\end{definition}
We typically obtain an upper bound of this critical radius by constructing one specific sub-root function $\psi$ % with its fixed point $r^\star$ 
satisfying (\ref{eq:critical_radius}).

\vspace{.5em}

\section{Batch RL with Value Function Approximation} \label{sec:batch_RL_setting}
This paper focuses on the offline setting where the data in form of tuples $\Data = \{(s, a, r, s', h)\}$ are collected beforehand, and are given to the agent. In each tuple, $(s, a)$ are the state and action at the $h^{\textrm{th}}$ step, $r$ is the resulting reward, and $s'$ is the next state sampled from $\P_h(\cdot |s, a)$. For each $h\in[H]$, we have access to $n$ data, that are {\it i.i.d} sampled with marginal distribution $\mu_h$ over $(s, a)$ at the $h^{\textrm{th}}$ step. We denote $\mu = \mu_1 \times \mu_2 \times \ldots \times \mu_H$. 
For each $h\in[H]$, 
we further denote the marginal distribution of $s'$ in tuple $(s, a, s', h)$ as $\nu_h$ , and let $\nu = \nu_1 \times \nu_2 \times \ldots \times \nu_H$. Throughout this paper, we will consistently use $\mu$ and $\nu$ to only denote the probability measures defined above.

We assume data distribution $\mu$ is well-behaved and satisfies the following assumption.
\begin{assumption}[Concentrability]\label{as:coverage}
	Given a policy $\pi$, let $P_h^{\pi}$ denote the marginal distribution at time step $h$, starting from $s_1$ and following $\pi$. There exists a parameter $C$ such that
	\begin{equation*}
	\sup_{(s,a,h) \in \cS \times \cA \times [H]}\frac{{\rm d} P_h^{\pi}}{{\rm d} \mu_h}(s, a) \le C \qquad \text{for any policy $\pi$}.
	\end{equation*}
\end{assumption}
Assumption \ref{as:coverage} requires that for any state-action pair $(s, a)$, if there exists a policy $\pi$ that reaches $(s, a)$ with some descent amount of probability, then the chance that sample $(s, a)$ appears in the dataset would not be low.
Intuitively, Assumption \ref{as:coverage} ensures that the dataset $\Data$ is representative for all the ``reachable'' state-action pairs. The assumption is frequently used in the literature of batch RL, e.g. equation (7) in \citet{munos2003error}, Definition 5.1 in \citet{munos2007performance}, Proposition 1 in \citet{farahmand2010error}, Assumption 1 in \citet{chen2019information}, etc.
We remark that Assumption \ref{as:coverage} here is the only assumption of this paper regarding the properties of the batch data.

We consider the setting of value function approximation, where at each step $h$ we use a function $f_h$ in class $\cF_h$ to approximate the optimal $Q$-value function. For notational simplicity, we denote $f := (f_1, \cdots, f_H) \in \cF$ with $\cF :=  \cF_1 \times \cdots \times \cF_H$.  Since no reward is collected in the $(H+1)^{\text{th}}$ steps, we will always use the convention that $f_{H+1} =0$
and $\cF_{H+1} = \{0\}$.
We assume $f_h \in [-H,H]$ for any $f_h \in \cF_h$.
Each $f \in \cF$ induces a greedy policy $\pi_{f} = \{\pi_{f_h}\}_{h=1}^{H}$ where
\begin{equation*}
\pi_{f_h}(a \mid s) = \ind\Big[a = \argmax_{a'} f_h(s, a')\Big].
\end{equation*}

In valued-based batch RL, we take the offline dataset $\Data$ as input and output an estimated optimal $Q$-value function $f$ and the associated policy $\pi_f$.
We are interested in the performance of $\pi_f$, which is measured by suboptimality in values, i.e., $V_1^{\star}(s_1) - V_1^{\pi_f}(s_1)$. However, this gap is highly nonsmooth in $f$, which is similar to the case of supervised learning where the $0-1$ losses for classification tasks are also highly nonsmooth and intractable. To mitigate this issue, a popular approach is to use a surrogate loss---the Bellman error. 
\begin{definition}[Bellman error] \label{def_Berr}
	Under data distribution $\mu$, we define the \emph{Bellman error} of function $f= (f_1, \cdots, f_H)$ as
	\begin{equation} 
	\Berr(f) := \frac{1}{H} \sum_{h=1}^H \| f_h - \mathcal{T}_h^{\star} f_{h+1} \|_{\mu_h}^2 .
	\end{equation}
\end{definition}
Bellman error $\Berr(f)$ appears in many classical RL algorithms including Bellman risk minimization (BRM) \citep{antos2008learning}, least-square temporal difference (LSTD) learning \citep{bradtke1996linear,lazaric2012finite}, etc.

The following lemma shows that under Assumption \ref{as:coverage}, one can control the suboptimality in values by the Bellman error.
\begin{restatable}[Bellman error to value suboptimality]{lemma}{gap}
%\begin{lemma} (Bellman error to value suboptimality)
	\label{lem:surrogate} 
	Under Assumption \ref{as:coverage}, for any $f \in \cF$, we have that ,
	\begin{equation} \label{eq:err_decomp}
	V_1^{\star}(s_1) - V_1^{\pi_{f}}(s_1) \leq 2H \sqrt{C \cdot \Berr(f)},
	\end{equation}
	where $C$ is the concentrability coefficient in Assumption \ref{as:coverage}.
%\end{lemma}
\end{restatable}
Therefore, the Bellman error $\Berr(f)$ is indeed a surrogate loss for the suboptimality of $\pi_f$ under mild conditions. In the next two sections, we will focus on designing efficient algorithms that minimize the Bellman error.
% !TEX root = main.tex

\section{Results for Double Sampling Regime}
\label{sec:double}
As a starting point for Bellmen error minimization, we consider an empirical version of $\Berr(f)$ computed from samples.
A natural choice of this empirical proxy is as follows
\begin{equation}
\empBerr(f) := \frac{1}{nH} \sum_{(s, a, r, s', h) \in \Data} \big(f_h(s, a) - r - V_{f_{h+1}}(s')\big)^2,
\end{equation}
where $V_{f_{h+1}}(s) := \max_{a \in \cA} f_{h+1}(s,a)$.
Unfortunately, the estimator $\empBerr$ is biased due to the error-in variable situation \citep{bradtke1996linear}.
In particular, we have the following decomposition.
\begin{equation} \label{eq:decomp}
\Berr(f) = \E_{\mu}\empBerr(f)  - \frac{1}{H}\sum_{h=1}^H \E_{\mu_h} \Var_{s' \sim \P_h(\cdot| s, a)}(V_{f_{h+1}}(s')).
\end{equation}

That is, the Bellman error and the expectation of $\empBerr$ differ by a variance term. This variance term is due to the stochastic transitions in the system, which is non-negligible even when $f$ approximates the optimal value function $Q^\star$.  A direct fix of this problem is to estimate the variance by double samples, where two independent samples of $s_{h+1}$ are drawn when being in state $s_h$ \citep{baird1995residual}.

Formally, in this section, we consider the setting where for any $(s, a, r, s', h)$ in dataset $\Data$, there exists a paired tuple $(s, a, r, \tilde{s}', h)$ which share the same state-action pair $(s, a)$ at step $h$, while $s', \tilde{s}'$ being two independent samples of the next state. Such data can be collected for instance if a simulator is avaliable, or the system allows an agent to revert back to the previous step. For simplicity, we denote this dataset as $\DData = \{(s, a, r, s', \tilde{s}', h)\}$ without placing additional constraints.

We construct the following empirical risk, which further estimates the variance term in (\ref{eq:decomp}) via double samples,
\begin{equation*}
\empDS(f)
:= \frac{1}{nH} \sum_{(s, a, r, s', \tilde{s}', h) \in \DData}\bigg[ \big(f_h(s, a) - r - V_{f_{h+1}}(s')\big)^2 - \frac{1}{2} \left(V_{f_{h+1}}(s') - V_{f_{h+1}}(\tilde{s}')\right)^2 \bigg].
\end{equation*}
We can show that, for any fixed $f \in \cF$, $\E \empDS(f) = \Berr(f)$, i.e., $\empDS$ is an unbiased estimator of the Bellman error.
Our algorithm for this setting is simply the Empirical Risk Minimizer (ERM), and we prove the following guarantee.

\begin{restatable}[]{theorem}{ds}
%\begin{theorem} 
	\label{thm:doublesample}
	There exists an absolute constant $c>0$, with probability at least $1-\delta$, the ERM estimator $\hat{f}=\argmin_{f \in \cF} \empDS(f)$ satisfies the following:
	\begin{align*}
	\Berr(\hat{f}) \le \min_{f \in \cF} \Berr(f) + c H^2 \sqrt{\frac{\log(1/\delta)}{n}}  + c \sum_{h=1}^H \big( \mathcal{R}^{\mu_h}_n(\mathcal{F}_h) + \mathcal{R}^{\nu_{h}}_n( V_{\mathcal{F}_{h+1}}) \big).
	\end{align*}
%\end{theorem}
\end{restatable}

Here, we use shorthand $V_{\cF_{h+1}} := \{ V_{f_{h+1}} \mid f_{h+1} \in \cF_{h+1} \}$ for any $h \in [H]$.
Theorem \ref{thm:doublesample} asserts that, in the double sampling regime, simple ERM has its excess risk $\Berr(\hat{f}) - \min_{f \in \cF} \Berr(f)$ upper bounded by the Rademacher complexity of function class $\{\cF_h\}_{h=1}^H$, $\{V_{\cF_{h+1}}\}_{h=1}^H$ and a small concentration term that scales as $\tlO(1/\sqrt{n})$. 

Most importantly, we remark that Theorem \ref{thm:doublesample} holds without any assumption on the input data distribution or the properties of the MDP. Function class $\cF$ can also be completely misspecificed in the sense the optimal value function $Q^\star$ may be very far from $\cF$. This allows Theorem \ref{thm:doublesample} to be widely applicable to a large number of applications.

However, a major limitation of Theorem \ref{thm:doublesample} is its reliance on double samples. Double samples are not available in most dynamical systems that have no simulators or can not be reverted back to the previous step. In next section, we analyze algorithms in the standard single sampling regime.
% !TEX root = main.tex

\section{Results for Single Sampling Regime}
\label{sec:single}

% Overview about each subsection.
In this section, we focus on batch RL in the standard single sampling regime, where each tuple $(s,a,r,s',h)$ in dataset $\cD$ has a single next step $s'$ following $(s,a)$. We first present a sample complexity lower bound for minimizing the Bellman error, showing that in order to acheive an excess risk that does not scale polynomially with respect to the number of states,
it is inevitable to have additional structural assumptions on function class $\cF$ and the MDP. Then we analyze fitted Q-iteration (FQI) and a minimax estimator respectively, under different completeness assumptions. In addition to Rademacher complexity upper bounds similar to Theorem~\ref{thm:doublesample}, we also utilize localization techniques and prove bounds with faster statistical rate in these two schemes.

\subsection{Lower bound}

%Raise the question, whether we can achieve similar positive result.
Recall that when double samples are available, the excess risk of ERM estimator is controlled by Rademacher complexities of function classes (Theorem~\ref{thm:doublesample}). 
In the single sampling regime, one natural question to ask is whether there exists an algorithm with a similar guarantee (i.e. the excess risk is upper bounded by certain complexity measure of the function class).
% entirely depending on the capacity of function class? 
Unfortunately, without further assumptions, the answer is negative.

%\begin{theorem}
\begin{restatable}[]{theorem}{lb}
  \label{theorem:lb}
	Let $\mathfrak{A}$ be an arbitrary algorithm that takes any dataset $\cD$ and function class $\cF$ as input and outputs an estimator $\hat{f} \in \cF$. For any $S \in \N^+$ and sample size $n \geq 0$, there exists an $S$-state, single-action MDP paired with a function class $\cF$ with $|\cF| = 2$ such that the $\hat{f}$ output by algorithm $\mathfrak{A}$ satisfies
	\begin{equation} \label{eq:thm_lb}
	\E \Berr(\hat{f}) \ge  \min_{f \in \cF} \Berr(f) + \Omega\left( \min\left\{1, 
	\frac{S^{1/2}}{n}\right\}\right).
	\end{equation} 
	Here, the expectation is taken over the randomness in $\cD$.
%\end{theorem}
\end{restatable}

Theorem~\ref{theorem:lb} reveals a fundamental difference between the single sampling regime and the double sampling regime. The lower bound in inequality~(\ref{eq:thm_lb}) depends polynomially on $S$---the cardinality of state space, which is considered to be intractably large in the setting of function approximation. In batch RL with single sampling, despite the use of function class $\cF$, the hardness of Bellman error minimization is still determined by the size of state space.
This also suggests that minimizing Bellman error
in the single sampling regime, is intrinsically different from the classic supervised learning due to the additional temporal correlation structure presented within the data.

We remark that unlike most lower bounds of similar type in prior works \citep{sutton2018reinforcement, sun2019model}, which only apply to certain restrictive classes of algorithms, Theorem \ref{theorem:lb} is completely information-theoretic, and applies to any algorithm.

To circumvent the hardness result in Theorem~\ref{theorem:lb}, additional structural assumptions are necessary.
In the following, we provide statistical gurantees for two batch RL algorithms, where different completeness assumptions on $\cF$ are used.

\subsection{Fitted Q-iteration (FQI)}\label{subsec:FQI}

\begin{algorithm}[t]
	\caption{FQI}
	\begin{algorithmic}[1] \label{alg:FQI_simp}
		\STATE \textbf{initialize} $\hat{f}_{H+1} \gets 0$.
		\FOR{$h = H, H-1, \ldots, 1$}
		\STATE $ \hat{f}_h \gets \argmin_{f_h \in \cF_h} \hat{\ell}_{h}(f_h,\hat{f}_{h+1}) := \frac{1}{n} \sum_{(s,a,r,s',h) \in \cD_h} \big( f_h(s,a) - r - V_{\hat{f}_{h+1}}(s') \big)^2.$
		\ENDFOR
		\STATE \textbf{return} $\hat{f} = (\hat{f}_1, \ldots, \hat{f}_H)$.
	\end{algorithmic}
\end{algorithm}

We consider the classical FQI algorithm. We assume that function class $\cF = \cF_1 \times \ldots \times \cF_H$ is (approximately) closed under the optimal Bellman operators $\cT_1^{\star}, \ldots, \cT_H^{\star}$, which is commonly adopted by prior analyses of FQI \citep{munos2008finite,chen2019information}.

\begin{assumption} % [Approximate Closure] 
	\label{as:approx_comp_FF}
	There exists $\epsilon>0$ such that, for all $h\in[H]$, $\sup_{f_{h+1} \in \cF_{h+1}}\inf_{f_h \in \cF_h} \norm{f_h - \mathcal{T}_h^{\star} f_{h+1} }^2_{\mu_h} \leq \epsilon$.
\end{assumption}

The FQI algorithm is closely related to approximate dynamic programming \citep{bertsekas1995neuro}. It starts by setting $\hat{f}_{H+1} := 0$ and then recursively computes Q-value functions at $h = H, H-1, \ldots, 1$. Each iteration in FQI is a least squares regression problem based on data collected at that time step. For $h \in [H]$, we denote $\cD_h$ as set of data at the $h^{\text{th}}$ step. The details of FQI are specified in Algorithm \ref{alg:FQI_simp}.

In the following Theorem~\ref{theorem:FQI_RC}, we upper bound the excess risk of the output of FQI in terms of Rademacher complexity.

\begin{restatable}[FQI, Rademacher complexity]{theorem}{FQIRC}
%\begin{theorem}[FQI, Rademacher complexity]
 \label{theorem:FQI_RC}
	There exists an absolute constant $c>0$, under Assumption \ref{as:approx_comp_FF}, with probability at least $1 - \delta$, the output of FQI $\hat{f}$ satisfies
	\begin{equation} 
	\label{eq:FQI_RC} 
	\Berr(\hat{f}) \leq  \epsilon + c \sum_{h=1}^H \popR^{\mu_h}_n (\cF_h) + c H^2\sqrt{\frac{ \log(H/\delta)}{n}}.
	\end{equation}
%\end{theorem}
\end{restatable}
We remark that Assumption~\ref{as:approx_comp_FF} immediately implies that $\min_{f \in \cF} \Berr(f) \leq \epsilon$. Therefore, although the minimal Bellman error $\min_{f \in \cF} \Berr(f)$ does not explicitly appear on the right hand side, inequality~(\ref{eq:FQI_RC}) is still a variant of excess risk bound.
%can still reduce to an excess risk bound

For typical parametric function classes, the Rademacher complexity scales as $n^{-1/2}$ (see Section \ref{sec:examples}). Therefore, Theorem \ref{theorem:FQI_RC} guarantees that the excess risk decrease as $n^{-1/2}$, up to a constant error $\epsilon$ due to the approximate completeness (in Assumption \ref{as:approx_comp_FF}).
However, since Bellman error is the average of squared $L^2$-norms (Definition~\ref{def_Berr}), one may expect a faster statistical rate in this setting, similar to the case of linear regression.
For this reason, we take advantage of the localization techniques and develop sharper error bounds in Theorem~\ref{theorem:FQI_localRC}.

\begin{restatable}[FQI, local Rademacher complexity]{theorem}{FQILRC}
%\begin{theorem}[FQI, local Rademacher complexity]
 \label{theorem:FQI_localRC}
	There exists an absolute constant $c>0$, under Assumption \ref{as:approx_comp_FF}, with probability at least $1 - \delta$, the output of FQI $\hat{f}$ satisfies 
	\begin{align} 
	 \label{eq:FQI_localRC}
	\Berr(\hat{f}) \leq \epsilon +  c\sqrt{\epsilon \cdot \Delta} + c \Delta~, \qquad
	\text{~where~} \Delta := H \sum_{h=1}^H r_h^{\star} +  H^2 \frac{\log(H/\delta)}{n}~. 
	\end{align}  
	Here $r^\star_h$ is the critical radius of local Rademacher complexity $\popR^{\mu_h}_n ( \{ f_h \in \cF_h  ~|~ \norm{f_h - f^{\best}_h}_{\mu_h}^2 \leq r \})$ with $f^{\best}_h := \argmin_{f_h \in \cF_h} \|f_h - \cT_h^{\star} \hat{f}_{h+1}\|_{\mu_h}$.
%\end{theorem}
\end{restatable}
On the RHS of inequality (\ref{eq:FQI_localRC}), the first term $\epsilon$ measures model misspecification. The other two terms $c ( \sqrt{\epsilon \cdot \Delta} + \Delta )$ can be viewed as statistical errors since $\Delta \rightarrow 0$ as sample size $n \rightarrow \infty$. 
For typical parametric function classes, the critical radius of the local Rademacher complexity scales as $n^{-1}$ (see Section \ref{sec:examples}), which decreases much faster than standard Rademacher complexity. That is, Theorem \ref{theorem:FQI_localRC} indeed guarantees faster statistical rate comparing to Theorem \ref{theorem:FQI_RC}.

Finally, we remark that $f_h^{\best}$ in Theorem~\ref{theorem:FQI_localRC} depends on $\hat{f}_{h+1}$ and therefore is random. We will show later in Section \ref{sec:examples} for many examples, the critical radius can be upper bounded independent of the choice of $f_h^{\best}$, in which case the randomness in $f_h^{\best}$ does not affect the final results.

\subsection{Minimax Algorithm} \label{sec:minimax}

The (approximate) completeness of $\cF$ in Assumption~\ref{as:approx_comp_FF} can be stringent sometimes. For instance, if there is a new function $f_h$ attached to $\cF_h$, for the sake of completeness, we need to enlarge $\cF_{h-1}$ by adding several approximations of $\cT_{h-1}^{\star} f_h$. The same goes for $\cF_{h-2}, \ldots, \cF_1$. After amplifying the function classes one by one for each step, we may obtain an exceedingly large $\cF$.

To avoid the issue above posted by the completeness assumptions on $\cF$, 
we introduce a new function class $\cG = \cG_1 \times \cdots \times \cG_H$, where $\cG_h$ consists of functions mapping from $\cS \times \cA$ to $[-H,H]$.
We assume that for each $f_{h+1} \in \cF_{h+1}$, one can always find a good approximation of $\cT_h^{\star}f_{h+1}$ in this helper function class $\cG_h$. 
\begin{assumption}
	\label{as:approx_gcomp_FF}
	There exists $\epsilon > 0$ such that, for all $h\in [H]$, $\sup_{f_{h+1} \in \cF_{h+1}}\inf_{g_h \in \cG_h} \| g_h - \cT_h^{\star} f_{h+1} \|^2_{\mu_h} \leq \epsilon$.
\end{assumption}
According to (\ref{eq:decomp}), we can approximate Bellman error $\Berr(f)$ by subtracting the variance term from $\empBerr(f)$. If $g_h$ is close to $\cT_h^{\star} f_{h+1}$, then $\big( g_h(s,a) - r - V_{f_{h+1}}(s') \big)^2$ averaged over data provides a good estimator of the variance term.
Following this intuition, we define a new loss
\begin{multline*}
\empMM(f, g)
:= \frac{1}{nH} \!\!\!\! \sum_{(s, a, r, s', h) \in \Data} \!\!\!\!\Big[ \big(f_h(s, a) - r - V_{f_{h+1}}(s')\big)^2  - \big(g_h(s,a) - r - V_{f_{h+1}}(s')\big)^2 \Big].
\end{multline*}
The minimax algorithm \citep{antos2008learning,chen2019information} then computes
\begin{equation*}
\hat{f} := \argmin_{f\in \cF} \max_{g\in \cG} \empMM(f, g).
\end{equation*}

Now we are ready to state our theoretical guarantees for the minimax algorithms.

\begin{restatable}[Minimax algorithm, Rademacher complexity]{theorem}{MMRC}
%\begin{theorem}[Minimax algorithm, Rademacher complexity] 
	\label{theorem:Minimax_RC}
	There exists an absolute constant $c>0$, under Assumption \ref{as:approx_gcomp_FF}, with probability at least $1 - \delta$, the minimax estimator $\hat{f}$ satisfies:
	\begin{align*}
	\Berr(\hat{f}) \le \min_{f \in \cF} \Berr(f) + \epsilon + c H^2 \sqrt{\frac{\log(1/\delta)}{n}} 
	+ c \sum_{h=1}^H \big(\popR^{\mu_h}_n(\cF_h) + \popR^{\mu_h}_n(\cG_h) + \popR^{\nu_h}_n(V_{\cF_{h+1}})\big).
	\end{align*}
%\end{theorem}
\end{restatable}
As is shown in Theorem~\ref{theorem:Minimax_RC}, the excess risk is simultaneously controlled by the Rademacher complexities of $\{\cF_h\}_{h=1}^H$, $\{\cG_h\}_{h=1}^H$ and $\{V_{\cF_{h+1}}\}_{h=1}^H$.

Similar to the results for FQI, we can also develop risk bounds with faster statistical rate using the localization techniques.
For technical reasons that will be soon discussed, 
we introduce the following assumption, which can be viewed a variant of the concentrability coefficient in Assumption~\ref{as:coverage} under different initial distributions.

\begin{assumption}\label{as:Ctilde} For any policy $\pi$ and $h \in [H]$, let $P_{h,t}^{\pi}$ (or $\widetilde{P}_{h,t}^{\pi}$) denote the marginal distribution at $t > h$, starting from $\mu_h$ at time step $h$ (or from $\nu_h \times \Uniform(\cA)$ at $h+1$) and following $\pi$. There exists a parameter $\widetilde{C}$ such that
	\[ \sup_{\begin{subarray}{c} (s,a) \in \cS \times \cA \\ h \in [H], t > h \end{subarray}} \bigg( \frac{{\rm d}P_{h,t}^{\pi}}{{\rm d} \mu_t} \vee \frac{{\rm d}\widetilde{P}_{h,t}^{\pi}}{{\rm d}\mu_t}\bigg)(s,a) \ \leq \widetilde{C} \quad \text{for any policy $\pi$}. \]
\end{assumption}
For notational convenience, we define
\begin{align*}
f^{\best} = (f_1^{\best}, \ldots, f_H^{\best}) := \argmin\nolimits_{f \in \cF} \Berr(f) \qquad \text{and} \qquad g_h^{\best} := \argmin\nolimits_{g_h \in \cG_h} \| g_h - \cT_h^{\star} f_{h+1}^{\best} \|_{\mu_h}.
\end{align*}
Now we are ready to state the excess risk bound of the minimax algorithm in terms of local Rademacher complexity as follows.
\begin{restatable}[Minimax algorithm, local Rademacher complexity]{theorem}{MMLRC}
%\begin{theorem}[Minimax algorithm, local Rademacher complexity]
	\label{theorem:Minimax_LRC}
	There exists an absolute constant $c>0$, under Assumptions \ref{as:approx_gcomp_FF} and \ref{as:Ctilde}, with probability at least $1 - \delta$, the minimax estimator $\hat{f}$ satisfies:
	\begin{align} 
	% \label{eq:FQI_localRC}
	&\Berr(\hat{f}) \! \leq \! \min_{f \in \cF} \Berr(f)\!+\! \epsilon +  c\!\sqrt{\big( \min_{f \in \cF} \Berr(f)\!+\! \epsilon\big) \Delta} + c \Delta~, \label{eq:FQI_LRC} \\
	& \begin{aligned} \Delta := H^3 \sum_{h=1}^H & \left[\widetilde{C}\big(r_{f, h}^{\star} + r_{g, h}^{\star} + \tilde{r}_{f, h}^{\star}\big) + \sqrt{\widetilde{C} r_{g, h}^{\star}\epsilon} \right]  +  H^2 \frac{\log(H/\delta)}{n}~. \end{aligned} \notag
	\end{align} 
	where $\widetilde{C}$ is the concentrability coefficient in Assumption \ref{as:Ctilde}, and $r_{f, h}^{\star}, r_{g, h}^{\star}, \tilde{r}_{f, h}^{\star}$ are the critical radius of the following local Rademacher complexities respectively:
	\begin{align*}
	&\popR^{\mu_h}_n \big( \big\{ f_h \in \cF_h  ~\big|~ \norm{f_h - f^{\best}_h}_{\mu_h}^2 \leq r \big\} \big)~, \\
	&\popR^{\mu_h}_n \big( \big\{ g_h \in \cG_h  ~\big|~ \norm{g_h - g_h^{\best}}_{\mu_h}^2 \leq r \big\}\big)~, \\
	&\popR^{\nu_h}_n \big( \big\{ V_{f_{h+1}} ~\big|~ f_{h+1} \in \cF_{h+1}, \norm{f_{h+1} - f^{\best}_{h+1}}^2_{\nu_h \times \Uniform(\cA)} \le r \big\} \big)~.
	\end{align*}
%\end{theorem}
\end{restatable}

Similar to Theorem~\ref{theorem:FQI_localRC}, our upper bound in (\ref{eq:FQI_LRC}) can also be viewed as a combination of model misspecification error ($\min_{f \in \cF} \Berr(f) + \epsilon$) and statistical error ($c \sqrt{\big( \min_{f \in \cF} \Berr(f) + \epsilon \big) \Delta} + c\Delta$). As $n \rightarrow \infty$, the model misspecification error is nonvanishing and the statistical error tends to zero. Again for typical parametric function classes, the critical radius of the local Rademacher complexity scales as $n^{-1}$ (see Section \ref{sec:examples}), and Theorem \ref{theorem:Minimax_LRC} claims the excess risk of the minimax algorithm also decreases as $n^{-1}$ except a constant model misspecification error $\epsilon$.

Intuitively, Assumption \ref{as:Ctilde} is required in Theorem \ref{theorem:Minimax_LRC} to allow that $\Berr(f)$ close to $\Berr(f^{\best})$ implies $f_h$ in the neighborhood of $f_h^{\best}$ for each step $h\in [H]$. We conjecture such additional assumption is unavoidable if we would like to upper bound the excess risk using the local Rademacher complexity of $\cF_h$, $\cG_h$ and $V_{\cF_{h+1}}$ for the minimax algorithm. 

In Appendix \ref{app:proof_thm_MM}, we present an alternative version of Theorem~\ref{theorem:FQI_localRC}, which does not require Assumption \ref{as:Ctilde} but bound the excess risk using the local Rademacher complexity of a composite function class depending on the loss, $\cF$, and $\cG$. The alternative version recovers the sharp result in \cite{chen2019information} when the function classes $\cF$ and $\cG$ both have finite elements.

Finally, our upper bounds for the minimax algorithm contain Radermacher complexities of the function class $V_{\cF}$. 
We can conveniently control them using the Radermacher complexities of function class $\cF$ as follows.

\begin{proposition} \label{lemma:VF}
	Let $\cF$ be a set of functions over $\cS \times \cA$ and $\rho$ be a measure over $\cS$.
	We have the following inequality,
	\[ \popR_n^{\rho} (V_{\cF}) \leq \sqrt{2} A \popR_n^{\rho \times \Uniform(\cA)}( \cF), \]
	where $A$ is the cardinality of the set $\cA$.
\end{proposition}
% !TEX root = batchrl-fa_icml.tex

\section{Examples}
\label{sec:examples}

Below we give four examples of function classes, each with an upper bound on Rademacher complexity, as well as the critical radius of the local Rademacher complexity. Throughout this section we use notation $r^{\star, \rho}_n(\cF, f_o)$ to denote the critical radius of local Rademacher complexity $\popR^{\rho}_n ( \{ f \in \cF  ~|~ \norm{f - f_o}^2_{\rho} \leq r \})$.

\paragraph{Function class with finite element.} First, we consider the function class $\cF$ with $|\cF| < \infty$. Under the normalization that $f \in [0, H]$ for any $f \in \cF$, we have the following.
\begin{restatable}[]{proposition}{finite}
%\begin{proposition}
	\label{prop:R_bound_finite}
For function class $\cF$ defined above, for any data distribution $\rho$ and any anchor function $f_o \in \cF$:
\begin{align*}
 \mathcal{R}^{\rho}_n ( \mathcal{F} ) \leq 2 H \max \bigg\{ \sqrt{\frac{\log|\mathcal{F}|}{n}}, \frac{\log |\mathcal{F}|}{n} \bigg\}~, \qquad r^{\star, \rho}_n(\cF, f_o) \le \frac{2H \log|\mathcal{F}|}{n}~.
\end{align*}
%\end{proposition}
\end{restatable}

\paragraph{Linear functions.} Let $\phi: \mathcal{S} \times \mathcal{A} \rightarrow \mathbb{R}^d$ be the feature map to a $d$-dimensional Euclidean space, and consider the function class $\cF  \subset \{ w\trans \phi ~|~ w\in \R^d, \norm{w} \le H  \}$. Under the normalization that $\norm{\phi(s, a)} \le 1$ for any $(s, a)$, we have
\begin{restatable}[]{proposition}{linear}
%\begin{proposition}
	\label{prop:R_bound_linear}
For linear function class $\cF$ defined above, for any data distribution $\rho$ and any anchor function $f_o \in \cF$:
\begin{align*}
 \mathcal{R}^{\rho}_n ( \mathcal{F} ) \leq H \sqrt{\frac{2d}{n}}~, \qquad r^{\star, \rho}_n(\cF, f_o) \le   \frac{2d}{n}.
\end{align*}
%\end{proposition}
\end{restatable}

\paragraph{Functions in RKHS.} Consider a Reproducing Kernel Hilbert Space (RKHS) $\cH$ associated with a positive kernel $k: (\cS \times \cA) \times (\cS \times \cA) \rightarrow \R$. % \yaqi{
Suppose that $k\big((s,a),(s,a)\big) \leq 1$ for any $(s,a) \in \cS \times \cA$. Consider the function class $\cF \subseteq \{f \in \cH ~|~ \norm{f}_{\mathcal{K}} \le H \}$, %} 
%Consider the function class $\cF \subset \{ f\in \cH ~|~ \norm{f}_{\mathcal{K}} \le H, \norm{f}_\rho \le H\}$, 
here $\norm{\cdot}_{\mathcal{K}}$ denotes the RKHS norm. Define an integral operator $\mathscr{T}: L^2(\rho) \rightarrow L^2(\rho)$ as 
\begin{equation*}
\mathscr{T} f := \mathbb{E}_{(s,a)\sim\rho}\big[ k\big( \cdot, (s,a) \big) f(s,a) \big].
\end{equation*}
	Suppose that $\mathbb{E}_{(s,a) \sim \rho}\big[k\big( (s,a), (s,a) \big)\big] < +\infty$. % and $\mathscr{T}$ is a trace-class operator. 
	Let $\big\{ \lambda_{i}(\mathscr{T}) \big\}_{i=1}^{\infty}$ be the eigenvalues of $\mathscr{T}$, arranging in a nonincreasing order. Then
\begin{restatable}[]{proposition}{kernel}
%\begin{proposition}
\label{prop:R_bound_kernel}
For kernel function class $\cF$ defined above, for any data distribution $\rho$ and any anchor function $f_o \in \cF$:
\begin{align*}
 \mathcal{R}^{\rho}_n ( \mathcal{F} ) \leq H \sqrt{\frac{2}{n} \sum_{i=1}^{\infty} 1 \wedge \big( 4 \lambda_i(\mathscr{T}) \big)},\qquad r^{\star, \rho}_n(\cF, f_o) \le   2 \min_{j \in \mathbb{N}} \left\{ \frac{j}{n} + H \sqrt{\frac{2}{n} \sum_{i=j+1}^{\infty} \lambda_i(\mathscr{T})} \right\}.
\end{align*}
%\end{proposition}
\end{restatable}

\paragraph{Sparse linear functions.} Let $\phi: \mathcal{S} \times \mathcal{A} \rightarrow \mathbb{R}^d$ be the feature map to a $d$-dimensional Euclidean space, and consider the function class $\cF  \subset \{ w\trans \phi ~|~ w\in \R^d, \norm{w}_0 \le s  \}$. Assume that when $(s, a) \sim \rho$, $\phi(s, a)$ satisfies a Gaussian distribution with covariance $\Sigma$. Assume $\norm{f}_\rho \le H$ for any $f \in \cF$. Furthermore, denote $\cn_s(\Sigma)$ to be the upper bound such that $\cn_s(\Sigma) \ge \lambda_{\max}(M)/\lambda_{\min}(M)$ for any matrix $M$ that is a $s \times s$ principal submatrix of $\Sigma$. Then
\begin{restatable}[]{proposition}{sparse}
%\begin{proposition}
\label{prop:R_bound_sparse}
There exists an absolute constant $c>0$, for sparse linear function class $\cF$ defined above, assume the data distribution $\rho$ satisfies the conditions specified above, then for any anchor function $f_o \in \cF$:
\begin{align*}
 \mathcal{R}^{\rho}_n ( \mathcal{F} ) \leq cH \sqrt{\kappa_s(\Sigma)} \sqrt{\frac{s \log d}{n}}, \qquad r^{\star, \rho}_n(\cF, f_o) \le   c^2 \kappa_s(\Sigma) \cdot \frac{ s \log d}{n}.
\end{align*}
%\end{proposition}
\end{restatable}

\paragraph{End-to-end results.} Finally, to obtain an end-to-end result that upper bounds the suboptimality in values for specific function classes listed above, we can simply combine (a) the result that upper bound the value suboptimality using the Bellman error (Lemma \ref{lem:surrogate}); (b) the results that upper bound the Bellman error in terms of (local) Rademacher complexity (Theorems \ref{thm:doublesample}, \ref{theorem:FQI_RC}-\ref{theorem:Minimax_LRC}); (c) the upper bounds of (local) Rademacher complexity for specific function classes (Propositions \ref{prop:R_bound_finite}-\ref{prop:R_bound_sparse}).

% !TEX root = main.tex

\section{Conclusion}
\label{sec:conclu}

This paper studies batch RL with general value function approximation from the lens of statistical learning theory.
We identify the intrinsic difference between batch reinforcement learning and classical supervised learning (Theorem \ref{theorem:lb}) due to the additional temporal correlation structure presented in the RL data. Under mild conditions, this paper also provides upper bounds on the generalization performance of several popular batch RL algorithms in terms of the (local) Rademacher complexities of general function classes. We hope our results shed light on the future research in further bridging the gap between statistical learning theory and RL.

\bibliographystyle{abbrvnat}
\bibliography{batchrl_ref}

\appendix

% !TEX root = main.tex

\section{Proof of Results for Double Sampling (Theorem~\ref{thm:doublesample})}

Throughout the supplementary materials, we omit the subscript $\rho$ in population Rademacher complexty $\popR_n^{\rho}(\cdot)$ if the distribution is clear from the context. 

In this part, we prove \Cref{thm:doublesample} in \Cref{sec:double}.
We first define some auxiliary notations to simplify the writing. We divide the dataset $\DData$ into $\DData = \DData_1 \cup \ldots \cup \DData_H$, where $\DData_h$ consists of $n$ independent sample tuples collected at the $h^{\text{th}}$ time step.
For $f_h,g_h \in \cF_h$, denote
\[ \ell_{\text{DS}}(g_h,f_h)(s,a,r,s',\tilde{s}') := \big( g_h(s,a) - r - V_{f_h}(s') \big)^2 - \frac{1}{2}\big( V_{f_h}(s') - V_{f_h}(\tilde{s}') \big)^2. \]
Define an expected value $\E_{\mu_h} \ell_{\text{DS}}(g_h,f_h) := \E \big[ \ell_{\text{DS}}(g_h,f_h)(s,a,r,s', \tilde{s}') \big]$ with $(s,a) \sim \mu_h$, $r = r_h(s,a)$, $s',\tilde{s}' \overset{i.i.d.}{\sim} \P_h(\cdot \, | \, s_h, a_h)$ and its empirical version $\hat{\ell}_{\text{DS}}(g_h,f_h) := \frac{1}{n} \sum_{(s,a,r,s',\tilde{s}',h) \in \DData_h} \ell_{\text{DS}}(g_h,f_h)(s, a, r, s', \tilde{s}')$. It is easy to see that $\E_{\mu_h} \ell_{\text{DS}}(g_h,f_h) = \| g_h - \cT_h^{\star} f_h \|_{\mu_h}^2$.
For any $\bfun = (f_1, \ldots, f_H) \in \bFun$, we have
\[ L_{\text{DS}}(\bfun) : = \frac{1}{H}\sum_{h=1}^H \ell_{\text{DS}}(f_h, f_{h+1}), \quad \E_{\mu} L_{\text{DS}}(\bfun) = \Berr(\bfun) \quad \text{and} \quad \empDS(\bfun) := \frac{1}{H} \sum_{h=1}^H \hat{\ell}_{\text{DS}}(f_h, f_{h+1}), \]
where $f_{H+1} := 0$. Note that the loss function $\empDS(\bfun)$ is an empirical estimation of $\Berr(\bfun)$.

Theorem~\ref{thm:doublesample} provides an upper error bound for the BRM estimator
$\empbfun = \argmin_{\bfun \in \bFun} \empDS(\bfun)$, of which the proof is given below.

\ds*

\begin{proof}[Proof of Theorem \ref{thm:doublesample}]
	We apply the uniform concentration inequalites in Lemma \ref{lemma:RC}. Let $\bfun^{\best}$ be a minimizer of the Bellman error within the function class $\bFun$, {\it i.e.} $\bfun^{\best} \in \argmin_{\bfun \in \bFun} \Berr(\bfun)$. By noting that $L_{\text{DS}}(\bfun) \in \big[-2H^2,4H^2\big]$, we have with probabliity at least $1 - \delta$,
	\begin{equation} \label{eq:double_RC_1} \E_{\mu} L_{\text{DS}} (\empbfun) - \E_{\mu} L_{\text{DS}} (\bfun^{\best}) \leq \big( \empDS(\empbfun) - \empDS(\bfun^{\best}) \big) + 2 \popR_n \big(\big\{ L_{\text{DS}}(\bfun) \bigm| \bfun \in \bFun \big\} \big) + 6H^2 \sqrt{\frac{2 \log(2/\delta)}{n}}. \end{equation}
	We use the relations $\E_{\mu} L_{\text{DS}}(\empbfun) = \Berr(\empbfun)$, $\E_{\mu} L_{\text{DS}}(\bfun^{\best}) = \Berr(\bfun^{\best}) = \min_{\bfun \in \bFun} \Berr(\bfun)$ and $\empDS(\empbfun) \leq \empDS(\bfun^{\best})$ and reduce \cref{eq:double_RC_1} to
	\begin{align} \label{eq:double_0} \Berr(\empbfun) \leq \min_{\bfun \in \bFun} \Berr(\bfun) + 2 \popR_n \big(\big\{ L_{\text{DS}} \bigm| \bfun \in \bFun \big\} \big) + 6 H^2 \sqrt{\frac{2 \log(2/\delta)}{n}}. \end{align}
	It then remains to simplify the form of Rademacher complexity $\popR_n \big(\big\{L_{\text{DS}}(\bfun) \bigm| \bfun \in \bFun \big\} \big)$.
	
	Due to the sub-additivity of Rademacher complexity, we have
	\begin{align} \label{eq:double_3} \popR_n\big( \big\{ L_{\text{DS}}(\bfun) \bigm| \bfun \in \bFun \big\} \big) \leq \frac{1}{H} \sum_{h=1}^H \popR_n\big( \{ \ell_{\text{DS}}(f_h,f_{h+1}) \mid f_h \in \cF_h, f_{h+1} \in \cF_{h+1} \} \big). \end{align}
	In order to tackle the term $\popR_n\big( \{ \ell_{\text{DS}}(f_h,f_{h+1}) \mid f_h \in \cF_h, f_{h+1} \in \cF_{h+1} \} \big)$ on the right hand side, we apply the vector-form contraction property of Rademacher complexity in \Cref{lemma:contraction_RC}.
	By letting \begin{align*} \tilde{\phi}_{h,1} := f_h(s,a), \quad \tilde{\phi}_{h,2} := r_h + V_{f_{h+1}}(s') \quad \text{and} \quad \tilde{\phi}_{h,3} := r_h + V_{f_{h+1}}(\tilde{s}'), \end{align*}
	we can write
	\begin{align*} \ell_{\text{DS}}(f_h,f_{h+1}) = \frac{1}{2} \big(\tilde{\phi}_{h,1}, \tilde{\phi}_{h,2}, \tilde{\phi}_{h,3}\big)^{\top} \tilde{\boldsymbol{A}} \left( \begin{array}{c} \tilde{\phi}_{h,1} \\ \tilde{\phi}_{h,2} \\ \tilde{\phi}_{h,3} \end{array} \right) \qquad \text{with } \tilde{\boldsymbol{A}} = \left( \begin{array}{ccc} 2 & -2 & 0 \\ -2 & 1 & 1 \\ 0 & 1 & -1 \end{array} \right).
	\end{align*}
	Since the spectral norm $\| \tilde{\boldsymbol{A}} \|_2 \leq 4$ and $\big\| \big(\tilde{\phi}_{h,1}, \tilde{\phi}_{h,2}, \tilde{\phi}_{h,3}\big)^{\top} \big\|_2 \leq \sqrt{3} H$ due to the boundedness of $f_h$ and $\cT_h^{\star} f_{h+1}$, we find that $\ell_{\text{DS}}(f_h,f_{h+1})$ is ($4\sqrt{3}H$)-Lipschitz with respect to the vector $\big(\tilde{\phi}_{h,1}, \tilde{\phi}_{h,2}, \tilde{\phi}_{h,3}\big)^{\top}$.
	\Cref{lemma:contraction_RC} then implies
	\begin{align} \label{eq:double_1} \mathcal{R}_n \big(\big\{ \ell_{\text{DS}}(f_h,f_{h+1}) \bigm| f_h \in \cF_h, f_{h+1} \in \cF_{h+1} \big\} \big) \leq 10 H \Big( \popR_n \big(\big\{ \tilde{\phi}_{h,1} \big\} \big) + \popR_n \big\{ \tilde{\phi}_{h,2} \big\} \big) + \popR_n \big\{ \tilde{\phi}_{h,3} \big\} \big) \Big). \end{align}
	Recalling that $s'$ and $\tilde{s}'$ are {\it i.i.d.} conditioned on $(s, a)$, we use the sub-additivity of Rademacher complexity and find that
	\begin{equation} \label{eq:double_4} \begin{aligned} & \popR_n \big(\big\{ \tilde{\phi}_{h,1} \big\} \big) \leq \popR_n^{\mu_h} (\cF_h) \\ & \popR_n \big(\big\{ \tilde{\phi}_{h,2} \big\} \big) = \popR_n \big(\big\{ \tilde{\phi}_{h,3} \big\} \big) \leq \popR_n(\{ r_h \}) \!+\! \popR_n^{\nu_h} (V_{\cF_{h+1}} ), \end{aligned} \end{equation}
	where $\nu_h$ is the marginal distribution of $s'$ in the $h^{\text{th}}$ step.
	Note that $\{ r_h \}$ is a singleton, therefore, $\popR_n(\{r_h\}) = 0$.
	It follows from \cref{eq:double_1,eq:double_4} that 
	\begin{equation} \label{eq:double_5} \begin{aligned} & \mathcal{R}_n \big(\big\{ \ell_{\text{DS}}(f_h,f_{h+1}) \bigm| f_h \in \cF_h, f_{h+1} \in \cF_{h+1} \big\} \big) \leq 10 H \big( \popR_n^{\mu_h} (\cF_h) + 2 \popR_n^{\nu_h} ( V_{f_{h+1}} ) \big). \end{aligned} \end{equation}
	
	Combining \cref{eq:double_3,eq:double_5}, we learn that
	\begin{align} \label{eq:double_6}
		\mathcal{R}_n \big(\big\{ L_{\text{DS}}(\bfun) \bigm| \bfun \in \bFun \big\} \big)
		\leq 10 \sum_{h=1}^H \big( \popR_n^{\mu_h} ( \cF_h ) + 2 \popR_n^{\nu_h} (V_{\cF_{h+1}}) \big).
	\end{align}
	Plugging \cref{eq:double_6} into \cref{eq:double_0}, we finish the proof.
	
\end{proof}
% !TEX root = main.tex

\section{Proof of Results for FQI (Theorems~\ref{theorem:FQI_RC}~and~\ref{theorem:FQI_localRC})}

	In this section, we analyze the FQI estimator defined in \Cref{alg:FQI_simp}. For any $f_h \in \cF_h$ and $f_{h+1} \in \cF_{h+1}$, we denote
	\begin{align} \label{eq:def_ell} \ell(f_h,f_{h+1})(s,a,r,s') := \big(f_h(s,a) - r - V_{f_{h+1}}(s')\big)^2, \end{align}
	therefore, $\hat{\ell}_h(f_h,f_{h+1}) := \frac{1}{n} \sum_{(s,a,r,s',h) \in \cD_h} \ell(f_h,f_{h+1})(s,a,r,s')$.
	Note that each iteration in FQI solves an empirical loss minimization problem $\hat{f}_h := \argmin_{f_h \in \cF_h} \hat{\ell}_h(f_h,\hat{f}_{h+1})$.
	The empirical loss $\hat{\ell}_h(f_h,\hat{f}_{h+1})$ approximates
	\begin{align*} 
		\E_{\mu_h} \ell(f_h,\hat{f}_{h+1}) = & \E\big[ \ell(f_h,\hat{f}_{h+1}) \bigm| (s,a) \sim \mu_h, s' \sim \P_h(\cdot \mid s,a) \big] \\ = & \| f_h - \cT_h^{\star} \hat{f}_{h+1} \|_{\mu_h}^2 + \E_{\mu_h} {\rm Var}_{s' \sim \P_h(\cdot|s,a)}(V_{\hat{f}_{h+1}}(s')).
	\end{align*}
	Recall that
	\begin{equation} \label{def_tilde_f_h} f^{\best}_h = \argmin_{f_h \in \cF_h}\|f_h - \cT_h^{\star} \hat{f}_{h+1} \|_{\mu_h}. \end{equation}
	$f_h^{\best}$ minimizes $\E_{\mu_h}\ell(f_h,\hat{f}_{h+1})$.
	
	In the sequel, we develop upper bounds for Bellman error $\Berr(\empbfun)$ based on (local) Rademathcer complexities.

	\subsection{Analyzing FQI with Rademacher Complexity (Theorem~\ref{theorem:FQI_RC})} \label{app:proof_thm_FQI_RC}
	
	\FQIRC*
	
	\begin{proof}[Proof of \Cref{theorem:FQI_RC}]
		By \Cref{lemma:RC}, 
		with probability at least $1 - \delta$, for any $f_h \in \cF_h$,
		\begin{equation} \label{eq:FQI_main_1} \begin{aligned} \E_{\mu_h} \ell(f_h, \hat{f}_{h+1}) - \E_{\mu_h} \ell(f^{\best}_h, \hat{f}_{h+1}) \leq & \big( \hat{\ell}_h(f_h,\hat{f}_{h+1}) - \hat{\ell}_h(f^{\best}_h,\hat{f}_{h+1}) \big) \\ & + 2 \popR_n \big(\big\{ \ell(f_h,\hat{f}_{h+1}) - \ell(f^{\best}_h,\hat{f}_{h+1}) \, \big| \, f_h \in \cF_h \big\}\big) + 4H^2 \sqrt{\frac{2\log(2/\delta)}{n}}, \end{aligned} \end{equation}
		where $f^{\best}_h$ is defined in \cref{def_tilde_f_h} and we have used $\ell(f_h,\hat{f}_{h+1}) - \ell(f^{\best}_h,\hat{f}_{h+1}) \in [-2H^2,2H^2]$.
		
		Specifically, we take $f_h = \hat{f}_h$ in \cref{eq:FQI_main_1}. Due to the optimality of $\hat{f}_h$, we have $\hat{\ell}_h(\hat{f}_h,\hat{f}_{h+1}) \leq \hat{\ell}_h(f^{\best}_h,\hat{f}_{h+1}) $. We further use the relation
		\begin{align} \label{eq:FQI0} \| \hat{f}_h - \cT_h^{\star} \hat{f}_{h+1} \|_{\mu_h}^2 = \big(\E_{\mu_h} \ell (\hat{f}_h, \hat{f}_{h+1}) - \E_{\mu_h} \ell(f_h^{\best}, \hat{f}_{h+1})\big) + \| f_h^{\best} - \cT_h^{\star} \hat{f}_{h+1} \|_{\mu_h}^2. \end{align}
		and \Cref{as:approx_comp_FF}. It follows that
		\begin{equation} \label{eq:RC} \big\| \hat{f}_h - \cT_h^{\star} \hat{f}_{h+1} \big\|_{\mu_h}^2 \leq 2 \popR_n \big(\big\{ \ell(f_h,\hat{f}_{h+1}) - \ell(f^{\best}_h,\hat{f}_{h+1}) \, \big| \, f_h \in \cF_h \big\}\big)  + 4H^2 \sqrt{\frac{2\log(2/\delta)}{n}} + \epsilon.  \end{equation}
		
		We now simplify the Rademacher complexity term in \cref{eq:RC}. Due to the symmmetry of Rademacher random variables, we have $\popR_n \big(\big\{ \ell(f_h,\hat{f}_{h+1}) - \ell(f^{\best}_h,\hat{f}_{h+1}) \, \big| \, f_h \in \cF_h \big\}\big) = \popR_n \big(\big\{ \ell(f_h,\hat{f}_{h+1}) \, \big| \, f_h \in \cF_h \big\}\big)$. We also note that the loss function $\ell$ is ($4H$)-Lipschitz in its first argument. In fact, since $|f_h| \leq H$ for all $f_h \in \cF_h$ and $r + V_{\hat{f}_{h+1}}(s') \in [-H, H]$, it holds that for any $f_h, f_h' \in \cF_h$,
		\begin{equation} \label{Lip} \begin{aligned} & \big|\ell(f_h,\hat{f}_{h+1})(s,a,r,s') - \ell(f_h',\hat{f}_{h+1})(s,a,r,s')\big| \\ = & |f_h(s,a) - f_h'(s,a)| \big|f_h(s,a) + f_h'(s,a) - 2r - 2V_{\hat{f}_{h+1}}(s')\big| \\ \leq & 4 H |f_h(s,a) - f_h'(s,a)|. \end{aligned} \end{equation}
		According to the contraction property of Rademacher complexity (see Lemma \ref{lemma:contraction_RC_0}), we have
		\begin{equation} \label{RC_2} \popR_n \big(\big\{ \ell(f_h,\hat{f}_{h+1}) - \ell(f^{\best}_h,\hat{f}_{h+1}) \, \big| \, f_h \in \cF_h \big\}\big) = \popR_n \big(\big\{ \ell(f_h,\hat{f}_{h+1}) \, \big| \, f_h \in \cF_h \big\}\big) \leq 2H \popR_n^{\mu_h} (\cF_h). \end{equation}
		
		Plugging \cref{RC_2} into \cref{eq:RC} and applying union bound, we find that with probability at least $1-\delta$,
		\[ \mathcal{E}(\empbfun) = \frac{1}{H} \sum_{h=1}^H \big\| \hat{f}_h - \cT_h^{\star} \hat{f}_{h+1} \big\|_{\mu_h}^2 \leq 8 \sum_{h=1}^H \popR_n^{\mu_h}(\cF_h) + 4H^2 \sqrt{\frac{2\log(2H/\delta)}{n}} + \epsilon, \]
		which completes the proof.
	\end{proof}

	\subsection{Analyzing FQI with Local Rademacher Complexity (Theorem~\ref{theorem:FQI_localRC})} \label{app:proof_thm_FQI_LRC}
	
	\FQILRC*

	\begin{proof}[Proof of \Cref{theorem:FQI_localRC}]
		Recall that we have shown in \cref{Lip} that $\ell(f,g)$ is ($4H$)-Lipchitz in its first argument $f$. Under Assumption \ref{as:approx_comp_FF}, for $f_h^{\best}$ shown in \cref{def_tilde_f_h}, we have
		\[ \begin{aligned} & {\rm Var}\big[\ell(f_h,\hat{f}_{h+1}) - \ell(f^{\best}_h, \hat{f}_{h+1})\big] \\ \leq & \mathbb{E}\big[ \big(\ell(f_h,\hat{f}_{h+1}) - \ell(f^{\best}_h, \hat{f}_{h+1})\big)^2 \big] \leq 16 H^2 \mathbb{E} \Big[ \big| f_h(s_h,a_h) - f^{\best}_h(s_h,a_h) \big|^2 \Big] \\ = & 16 H^2 \big\| f_h - f^{\best}_h \big\|_{\mu_h}^2 \leq 32 H^2 \Big( \big\| f_h - \cT_h^{\star} \hat{f}_{h+1} \big\|_{\mu_h}^2    + \big\| f^{\best}_h - \cT_h^{\star} \hat{f}_{h+1} \big\|_{\mu_h}^2 \Big) \\ = & 32 H^2 \Big[ \Big( \big\| f_h  -  \cT_h^{\star} \hat{f}_{h+1} \big\|_{\mu_h}^2    - \big\| f^{\best}_h  -  \cT_h^{\star} \hat{f}_{h+1} \big\|_{\mu_h}^2 \Big) + 2\big\| f^{\best}_h  -  \cT_h^{\star} \hat{f}_{h+1} \big\|_{\mu_h}^2 \Big] \\ \leq & 32 H^2 \Big( \mathbb{E}_{\mu_h}\big[\ell(f_h,\hat{f}_{h+1}) - \ell(f^{\best}_h, \hat{f}_{h+1})\big] + 2 \epsilon \Big). \end{aligned} \]
		When applying \Cref{theorem:LRC}, we are supposed to take a sub-root function larger than
		\[ \psi_{\FQI}(r) := 32H^2 \popR_n \Big\{ \ell(f_h,\hat{f}_{h+1}) - \ell(f^{\best}_h,\hat{f}_{h+1}) \, \Big| \, f_h \in \cF_h, 32 H^2 \Big( \mathbb{E}\big[\ell(f_h,\hat{f}_{h+1}) - \ell(f^{\best}_h, \hat{f}_{h+1})\big] + 2 \epsilon \Big) \leq r \Big\}. \]
		Note that
		\begin{align*}
		\psi_{\FQI}(r)
		\leq & 32H^2 \popR_n \Big( \Big\{ \ell(f_h,\hat{f}_{h+1}) - \ell(f^{\best}_h,\hat{f}_{h+1}) \, \Big| \, f_h \in \cF_h, 16 H^2 \big\| f_h - f^{\best}_h \big\|_{\mu_h}^2 \leq r \Big\} \Big) \\
		\leq & 128 H^3 \popR_n \Big( \Big\{ f_h - f^{\best}_h \, \Big| \, f_h \in \cF_h, 16 H^2 \big\| f_h - f^{\best}_h \big\|_{\mu_h}^2 \leq r \Big\} \Big) \\
		= & 128 H^3 \popR_n \Big( \Big\{ f_h \in \cF_h \, \Big| \, 16 H^2 \big\| f_h - f^{\best}_h \big\|_{\mu_h}^2 \leq r \Big\} \Big) \leq 128 H^3 \psi_h\Big(\frac{r}{16H^2} \Big)
		\end{align*}
		where $\psi_h$ is a sub-root function satisfying $\psi_h(r) \geq \popR_n \big(\big\{ f_h \in \cF_h \bigm| \|f_h - f_h^{\best}\|_{\mu_h}^2 \leq r \big\}\big)$ and the positive fixed point $r_h^{\star}$ of $\psi_h$ is the corresponding critical radius.
		In the second inequality, we have used the contraction property of Rademacher complexity (see \Cref{lemma:contraction_RC_0}) and the Lipschitz continuity of $\ell$. The equality in the last line is due to the symmetry of Rademacher random variables. According to \Cref{lemma:subroot}, the positive fixed point of $128H^3 \psi_h\big(\frac{r}{16H^2} \big)$ is upper bounded by $1024 H^4 r_h^{\star}$.
		
		We apply \cref{eq:theoremLRC1} in \Cref{theorem:LRC} and use the \cref{eq:FQI0} and $\hat{\ell}_h(\hat{f}_h,\hat{f}_{h+1}) \leq \hat{\ell}_h(f^{\best}_h,\hat{f}_{h+1})$. It follows that for a fixed parameter $\theta$, with probability at least $1-\delta$,
		\begin{align*} & \big\| \hat{f}_h - \cT_h^{\star} \hat{f}_{h+1} \big\|_{\mu_h}^2 - \big\| f^{\best}_h - \cT_h^{\star} \hat{f}_{h+1} \big\|_{\mu_h}^2 \\ \leq & c H^2 r_h^{\star} + \frac{c H^2\log(1/\delta)}{n} + c (\theta-1)\bigg( H^2 r_h^{\star} + c \frac{H^2\log(1/\delta)}{n} \bigg) + \frac{2\epsilon}{\theta-1}, \end{align*}
		where $c > 0$ is a universal constant.
		By union bound and \Cref{as:approx_comp_FF}, we have
		\[ \Berr(\hat{f}) \leq \epsilon + c H \sum_{h=1}^H r_h^{\star} + c H^2\frac{\log(H/\delta)}{n} + c (\theta-1)\bigg( H \sum_{h=1}^H r_h^{\star} + \frac{H^2\log(1/\delta)}{n} \bigg) + \frac{2\epsilon}{\theta-1}. \]
		We further take $\theta := 1 + \frac{\sqrt{\epsilon}}{2H} \big( \frac{1}{H} \sum_{h=1}^H r_h^{\star} + \frac{\log(H/\delta)}{n} \big)^{-\frac{1}{2}}$ and find that
		\begin{align*} \Berr(\hat{f}) \leq \epsilon + c H \sum_{h=1}^H r_h^{\star} + c H^2 \frac{\log(H/\delta)}{n} + c \sqrt{\epsilon \Big( H \sum_{h=1}^H r_h^{\star} + H^2\frac{\log(1/\delta)}{n}\Big)}, \end{align*}
		which completes the proof.
	\end{proof}
	
% !TEX root = main.tex

\section{Proof of Results for Minimax Algorithm (Theorems~\ref{theorem:Minimax_RC}, \ref{theorem:Minimax_LRC}~and~\ref{theorem:Minimax_LRC'})} \label{app:proof_thm_MM}

In this part, we prove the statistical guarantees for minimax algorithm in \Cref{sec:minimax}.

\paragraph{Notations} We first introduce some notations that will be used later in the analyses.
For any vector-valued function $\bfun = (f_1, \ldots, f_H) \in L^2(\mu_1) \times \ldots \times L^2(\mu_H)$, we denote $\| \bfun \|_{\bmu} := \sqrt{\frac{1}{H} \sum_{h=1}^H \| f_h \|_{\mu_h}^2}$ for short. Parallel to the optimal Bellman operator $\cT_h^{\star}$, we define $\cT_h^{\best}$ and $\widehat{\cT}_h$ as \[ \cT_h^{\best} f_{h+1} := \argmin_{g_h \in \cG_h} \| g_h - \cT_h^{\star} f_{h+1} \|_{\mu_h} \quad \text{and} \quad \widehat{\cT}_h f_{h+1} := \argmin_{g_h \in \cG_h} \frac{1}{n} \sum_{(s,a,r,s',h) \in \cD_h} \big( g_h(s,a) - r - V_{f_{h+1}}(s') \big)^2. \]
Let $\bcT^{\star}$, $\bcT^{\best}$, $\widehat{\bcT}$ be their vector form, given by
\begin{align}  \label{eq:def_T}
	\begin{aligned}
		\bcT^{\star} \bfun := & (\cT_1^{\star}f_2, \ldots, \cT_H^{\star}f_{H+1}), \\ \bcT^{\best} \bfun := & (\cT_1^{\best}f_2, \ldots, \cT_H^{\best}f_{H+1}), \\ \widehat{\bcT} \bfun := & (\widehat{\cT}_1 f_2, \ldots, \widehat{\cT}_H f_{H+1}),
	\end{aligned}
\end{align}
for any $\bfun \in \bFun$.

Similar to the definition of $\ell$ in \cref{eq:def_ell}, for any $g_h \in \cG_h \cup \cF_h$ and $f_{h+1} \in \cF_{h+1}$, we take $$\ell(g_h, f_{h+1})(s,a,r,s') = \big( g_h(s,a) - r - V_{f_{h+1}}(s') \big)^2.$$ For any $\bfun \in \cF$, $\bg \in \cF \cup \cG$ and $\{ (s_h, a_h, r_h, s_h') \}_{h=1}^H \in (\cS \times \cA \times \mathbb{R} \times \cS)^H$, let
\[ \ell(\bg, \bfun)(\cdot) := \frac{1}{H} \sum_{h=1}^H \ell(g_h,f_{h+1})(s_h,a_h,r_h,s_h') = \frac{1}{H} \sum_{h=1}^H \big( g_h(s_h,a_h) - r_h - V_{f_{h+1}}(s_h') \big)^2. \]
Denote 
\begin{align*} \E_{\bmu} \ell(\bg,\bfun) := \frac{1}{H} \sum_{h=1}^H \E_{\mu_h} \ell(g_h,f_{h+1}) = & \frac{1}{H} \sum_{h=1}^H \E \big[ \ell(g_h, f_{h+1})(s,a,r,s') \bigm| (s,a) \sim \mu_h, s' \sim \P_h(\cdot \mid s,a) \big] \\ = & \| \bg - \bcT^{\star} \bfun \|_{\bmu}^2 + \frac{1}{H} \sum_{h=1}^H \E_{\mu_h} {\rm Var}_{s' \sim \P_h(\cdot \mid s,a)} (V_{f_{h+1}}(s')) \end{align*}
\[ \text{and} \qquad \hat{\ell}(\bg,\bfun) := \frac{1}{H} \sum_{h=1}^H \hat{\ell}_h(g_h,f_{h+1}) = \frac{1}{nH} \sum_{(s,a,r,s',h) \in \cD} \big( g_h(s,a) - r - V_{f_{h+1}}(s') \big)^2. \]

The loss function in minimax algorithm then can be written as
\begin{align} \label{eq:def_LMM} \begin{aligned} & L_{\text{MM}}(\bfun, \bg) := \ell(\bfun, \bfun) - \ell(\bg, \bfun), \\ & \E_{\bmu} L_{\text{MM}}(\bfun, \bg) := \E_{\bmu} \ell(\bfun, \bfun) - \E_{\bmu} \ell(\bg, \bfun) \\ & \hat{L}_{\text{MM}}(\bfun, \bg) := \hat{\ell}(\bfun, \bfun) - \hat{\ell}(\bg, \bfun). \end{aligned} \end{align}
Note that $\E_{\bmu} L_{\text{MM}}(\bfun, \bg) = \| \bfun - \bcT^{\star} \bfun \|_{\bmu}^2 - \| \bg - \bcT^{\star} \bfun \|_{\bmu}^2 = \Berr(f) - \| \bg - \bcT^{\star} \bfun \|_{\bmu}^2$.

With our newly-defined notations, we formulate the minimax estimator as
\begin{equation} \label{eq:Minimax} \empbfun = \argmin_{\bfun \in \bFun} \max_{\bg \in \bG} \hat{L}_{\text{MM}}(\bfun,\bg) = \argmin_{\bfun \in \bFun} \hat{L}_{\text{MM}}(\bfun, \widehat{\bcT} \bfun). \end{equation}

In the analysis of minimax algorithm, we take $\bfun^{\best}$ as the function in $\bFun$ that minimizes the Bellmen risk, {\it i.e.}
\[ \bfun^{\best} := \argmin_{\bfun \in \bFun} \Berr(\bfun). \]

\paragraph{Main results}

\MMRC*

\MMLRC*

Aside from \Cref{theorem:Minimax_RC,theorem:Minimax_LRC}, we also have an alternative statistical guarantee for $\Berr(\hat{f})$ using local Rademacher complexity for composite function $L_{\text{MM}}(\bfun, \bcT^{\best}\bfun)$. See \Cref{theorem:Minimax_LRC'} below.
\begin{theorem}[Minimax algorithm, local Rademacher complexity, alternaltive] \label{theorem:Minimax_LRC'}
	There exists an absolute constant $c>0$, under \Cref{as:approx_gcomp_FF}, with probability at least $1 - \delta$, the minimax estimator $\hat{f}$ satisfies:
	\begin{align} 
	% \label{eq:FQI_localRC}
	&\Berr(\hat{f}) \! \leq \! \min_{f \in \cF} \Berr(f)\!+\! \epsilon +  c\!\sqrt{\big( \min_{f \in \cF} \Berr(f)\!+\! \epsilon\big) \Delta} + c \Delta~, \label{eq:FQI_LRC'} \\
	& \begin{aligned} \Delta := H^2 r_L^{\star} + H \sum_{h=1}^H r_{g,h}^{\star} + H^2 \frac{\log(H/\delta)}{n} ~. \end{aligned} \notag
	\end{align} 
	where $r_L^{\star}$ and $r_{g, h}^{\star}$ are the critical radius of the following local Rademacher complexities respectively:
	\begin{align*}
	&\popR^{\mu_h}_n \big( \big\{ L_{\text{MM}}(\bfun, \bcT^{\best} \bfun)  ~\big|~ \bfun \in \bFun, \E \big[ L_{\text{MM}}(\bfun, \bcT^{\best} \bfun)^2 \big] \leq r \big\} \big)~, \\
	&\popR^{\mu_h}_n \big( \big\{ g_h \in \cG_h  ~\big|~ \norm{g_h - g_h^{\best}}_{\mu_h}^2 \leq r \big\}\big)~.
	\end{align*}
\end{theorem}
In contrast to \Cref{theorem:Minimax_LRC}, \Cref{theorem:Minimax_LRC'} does not rely on the additional \Cref{as:Ctilde}.
In general, \Cref{theorem:Minimax_LRC'} provides a tighter upper bound for $\Berr(\empbfun)$ than \Cref{theorem:Minimax_LRC} when the function class $\big\{ L_{\text{MM}}(\bfun, \bcT^{\best} \bfun)  ~\big|~ \bfun \in \bFun \big\}$ has a clear structure and $r_L^{\star}$ is easy to estimate. For instance, this is the case if both $\bfun$ and $\bG$ have finite elements. Based on \Cref{theorem:Minimax_LRC'}, we can recover the sharp results for finite function classes in \citet{chen2019information}.

\Cref{as:approx_gcomp_FF} used in our analysis of minimax algorithm can be relaxed to: 
\[ \text{``There exist constants $\epsilon > 0$ and $\zeta \in [0,1)$ such that $\inf_{\bg \in \bG} \| \bg - \bcT^{\star} \bfun \|_{\bmu}^2 \leq \epsilon + \zeta \Berr(\bfun)$ for any $\bfun \in \bFun$.''} \]
In this way, we only need a high-quality approximation of $\bcT^{\star} \bfun$ in $\bG$ when $\bfun$ lies within a neighborhood of the optimal Q-function. We can easily generalize our analyses to this case. However, in order to avoid unnecessary clutter, we stick to the current \Cref{as:approx_gcomp_FF}.

\paragraph{Proof outline} Our analyses in this section are devoted to the proofs of \Cref{theorem:Minimax_RC,theorem:Minimax_LRC,theorem:Minimax_LRC'}.
\begin{enumerate}
	\item We first translate the estimation of $\Berr(\empbfun)$ into deriving uniform concentration bounds for $\hat{L}_{\text{MM}}(\bfun, \bcT^{\best} \bfun) - \hat{L}_{\text{MM}}(\bfun^{\best}, \bcT^{\best} \bfun^{\best})$ and $\hat{\ell}(\bg, \bfun) - \hat{\ell}(\bcT^{\best}\bfun^{\best}, \bfun^{\best})$ (\Cref{lemma:Minimax_ErrorDecomp} in \Cref{app:proof_minimax_decomp}). The error decomposition lemma is shared among the proofs of \Cref{theorem:Minimax_RC,theorem:Minimax_LRC,theorem:Minimax_LRC'}.
	\item We then develop the desired uniform concentration bounds using Rademacher complexities (\Cref{app:proof_minimax_Rad}) and local Rademacher complexities (\Cref{app:proof_minimax_locRad}) separately. In particular, when tackling $\hat{L}_{\text{MM}}(\bfun, \bcT^{\best} \bfun) - \hat{L}_{\text{MM}}(\bfun^{\best}, \bcT^{\best} \bfun^{\best})$, we have two alternative analyses involving local Rademacher complexities of different types of function classes. One leads to \Cref{theorem:Minimax_LRC} and the other results in \Cref{theorem:Minimax_LRC'}.
	\item In \Cref{app:proof_minimax_conclude}, we integrate the error decomposition result and uniform concentration bounds, and finish the proofs of theorems.
\end{enumerate}

\subsection{Error Decomposition} \label{app:proof_minimax_decomp}
We provide a decomposition of the Bellman error $\Berr(\empbfun)$ and upper bound the error using some uniform concentration inequalities.
\begin{lemma}[Error decomposition] \label{lemma:Minimax_ErrorDecomp}
	Suppose there exist $\alpha > 0$ and $Err_f, Err_g > 0$ such that the following concentration inequailities hold simultaneously.
	\begin{enumerate}
		\item For any $\bfun \in \bFun$, \begin{equation} \label{eq:bound_f} 
		\E_{\bmu} L_{\text{MM}} \big(\bfun, \bcT^{\best}\bfun\big) - \E_{\bmu} L_{\text{MM}} \big(\bfun^{\best},\bcT^{\best}\bfun^{\best}\big) \leq \alpha  \big( \hat{L}_{\text{MM}} \big(\bfun, \bcT^{\best}\bfun\big) - \hat{L}_{\text{MM}} \big(\bfun^{\best},\bcT^{\best}\bfun^{\best}\big) \big) + Err_f.
		\end{equation}
		\item For any $\bg \in \bG$,
		\begin{equation} \label{eq:bound_ga}
		\E_{\bmu} \ell \big(\bg, \bfun^{\best} \big) - \E_{\bmu} \ell \big(\bcT^{\best}\bfun^{\best}, \bfun^{\best} \big) \leq \alpha \big( \hat{\ell} \big(\bg, \bfun^{\best} \big) - \hat{\ell} \big(\bcT^{\best}\bfun^{\best}, \bfun^{\best} \big) \big) + Err_g.
		\end{equation}
	\end{enumerate}
	Then under Assumption \ref{as:approx_gcomp_FF}, the Bellman error satisfies
	\begin{equation} \label{eq:minimax_err_decomp} \Berr(\empbfun) \leq \min_{\bfun \in \bFun} \Berr(\bfun) + Err_f + Err_g + \epsilon. \end{equation}
\end{lemma}
\begin{proof}
	By definition of function $\E_{\bmu} L_{\text{MM}}(f,g)$ in \cref{eq:def_LMM}, we find that for any $\bfun \in \bFun$,
	\begin{align*}
	\E_{\bmu} L_{\text{MM}} \big(\bfun, \bcT^{\best}\bfun\big) = & \E_{\bmu} \ell (\bfun,\bfun) - \E_{\bmu} \ell (\bcT^{\best}\bfun,\bfun) = \| \bfun - \bcT^{\star} \bfun \|_{\bmu}^2 - \| \bcT^{\best} \bfun - \bcT^{\star} \bfun \|_{\bmu}^2 =  \Berr(\bfun) - \| \bcT^{\best} \bfun - \bcT^{\star} \bfun \|_{\bmu}^2.
	\end{align*}
	We learn from \Cref{as:approx_gcomp_FF} that $\|\bcT^{\best} \bfun -\bcT^{\star} \bfun \|_{\bmu}^2 \leq \epsilon$ for any $\bfun \in \bFun$, therefore,
	\begin{equation} \label{eq:Lmu} \begin{aligned} \E_{\bmu} L_{\text{MM}} \big(\bfun, \bcT^{\best}\bfun\big) - \E_{\bmu} L_{\text{MM}} \big(\bfun^{\best}, \bcT^{\best}\bfun^{\best}\big) = & \Berr(\bfun) - \Berr(\bfun^{\best}) - \| \bcT^{\best} \bfun - \bcT^{\star} \bfun \|_{\bmu}^2 + \| \bcT^{\best} \bfun^{\best} - \bcT^{\star} \bfun^{\best} \|_{\bmu}^2 \\
	\geq & \Berr(\bfun) - \Berr(\bfun^{\best}) - \epsilon, \end{aligned} \end{equation}
	which implies
	\[ \Berr(\empbfun) \leq \min_{\bfun \in \bFun} \Berr(\bfun) + \Big( \E_{\bmu} L_{\text{MM}} \big(\empbfun, \bcT^{\best}\empbfun\big) - \E_{\bmu} L_{\text{MM}} \big(\bfun^{\best}, \bcT^{\best}\bfun^{\best}\big) \Big) + \epsilon. \]
	By virtue of \cref{eq:bound_f},
	\begin{equation} \label{eq:Err_decom_1} \Berr(\empbfun) \leq \min_{\bfun \in \bFun} \Berr(\bfun) + \alpha  \Big( \hat{L}_{\text{MM}} \big(\empbfun, \bcT^{\best}\empbfun\big) - \hat{L}_{\text{MM}} \big(\bfun^{\best},\bcT^{\best}\bfun^{\best}\big) \Big) + Err_f + \epsilon. \end{equation}
	In the following,
	we leverage \cref{eq:bound_ga} to estimate $  \hat{L}_{\text{MM}} \big(\empbfun, \bcT^{\best}\empbfun\big) - \hat{L}_{\text{MM}} \big(\bfun^{\best},\bcT^{\best}\bfun^{\best}\big)$.
	
	We use the definition of $L_{\text{MM}}$ and find that
	\begin{equation} \label{eq:minimax_err_decomp_0} \begin{aligned} 
	& \hat{L}_{\text{MM}} \big(\empbfun, \bcT^{\best}\empbfun\big) - \hat{L}_{\text{MM}} \big(\bfun^{\best},\bcT^{\best}\bfun^{\best}\big) 
	= \hat{\ell} (\empbfun, \empbfun) - \hat{\ell} (\bcT^{\best} \empbfun, \empbfun) - \hat{\ell} (\bfun^{\best}, \bfun^{\best}) + \hat{\ell} (\bcT^{\best} \bfun^{\best}, \bfun^{\best}) \\
	= & \big( \hat{L}_{\text{MM}} \big(\empbfun, \widehat{\bcT} \empbfun\big)+ \hat{\ell} (\widehat{\bcT} \empbfun, \empbfun) \big) - \hat{\ell} (\bcT^{\best}\empbfun, \empbfun) - 
	\big(\hat{L}_{\text{MM}} \big(\bfun^{\best}, \widehat{\bcT} \bfun^{\best}\big) + \hat{\ell} (\widehat{\bcT}\bfun^{\best}, \bfun^{\best}) \big)
	+ \hat{\ell} (\bcT^{\best} \bfun^{\best}, \bfun^{\best}) \\
	= & \big( \hat{L}_{\text{MM}} \big(\empbfun, \widehat{\bcT} \empbfun\big)
	- \hat{L}_{\text{MM}} \big(\bfun^{\best}, \widehat{\bcT} \bfun^{\best}\big) \big)
	+ \big( \hat{\ell} (\widehat{\bcT} \empbfun, \empbfun) 
	- \hat{\ell} (\bcT^{\best} \empbfun, \empbfun) \big)
	- \big( \hat{\ell} (\widehat{\bcT} \bfun^{\best}, \bfun^{\best}) 
	- \hat{\ell} (\bcT^{\best} \bfun^{\best}, \bfun^{\best}) \big)
	. \end{aligned} \end{equation}
	Since $(\bfun,\bg) := (\empbfun,\widehat{\bcT}\empbfun)$ solves the minimax optimizaiton problem \cref{eq:Minimax}, we have $\hat{L}_{\text{MM}} \big(\empbfun, \widehat{\bcT} \empbfun\big)
	\leq \hat{L}_{\text{MM}} \big(\bfun^{\best}, \widehat{\bcT} \bfun^{\best}\big)$. Due to the optimality of $\widehat{\bcT}$, it also holds that $\hat{\ell} (\widehat{\bcT} \empbfun, \empbfun) 
	\leq \hat{\ell} (\bcT^{\best} \empbfun, \empbfun)$. To this end, \cref{eq:minimax_err_decomp_0} reduces to
	\begin{align} \label{eq:minimax_err_decomp_1} \hat{L}_{\text{MM}} \big(\empbfun, \bcT^{\best}\empbfun\big) - \hat{L}_{\text{MM}} \big(\bfun^{\best},\bcT^{\best}\bfun^{\best}\big) \leq - \big( \hat{\ell} (\widehat{\bcT} \bfun^{\best}, \bfun^{\best}) 
	- \hat{\ell} (\bcT^{\best} \bfun^{\best}, \bfun^{\best}) \big). \end{align}
	Additionally, \cref{eq:bound_ga} implies \[ \hat{\ell} (\widehat{\bcT} \bfun^{\best}, \bfun^{\best}) 
	- \hat{\ell} (\bcT^{\best} \bfun^{\best}, \bfun^{\best}) \geq \alpha ^{-1}\Big( \E_{\bmu} \ell \big(\widehat{\bcT} \bfun^{\best}, \bfun^{\best} \big) - \E_{\bmu} \ell \big(\bcT^{\best}\bfun^{\best}, \bfun^{\best} \big) \Big) - \alpha^{-1} Err_g. \]
	Note that $\E_{\bmu} \ell \big(\widehat{\bcT} \bfun^{\best}, \bfun^{\best} \big) - \E_{\bmu} \ell \big(\bcT^{\best}\bfun^{\best}, \bfun^{\best} \big) = \| \widehat{\bcT} \bfun^{\best} - \bcT^{\star} \bfun^{\best} \|_{\bmu}^2 - \| \bcT^{\best} \bfun^{\best} - \bcT^{\star} \bfun^{\best} \|_{\bmu}^2$ and $\| \widehat{\bcT} \bfun^{\best} - \bcT^{\star} \bfun^{\best} \|_{\bmu} \geq \| \bcT^{\best} \bfun^{\best} - \bcT^{\star} \bfun^{\best} \|_{\bmu}$ by definition of $\bcT^{\best}$, therefore,
	\[ \hat{\ell} (\widehat{\bcT} \bfun^{\best}, \bfun^{\best}) 
	- \hat{\ell} (\bcT^{\best} \bfun^{\best}, \bfun^{\best}) \geq  - \alpha^{-1} Err_g. \]
	It then follows from \cref{eq:minimax_err_decomp_1} that
	\begin{equation} \label{eq:Err_decom_2} \hat{L}_{\text{MM}} \big(\empbfun, \bcT^{\best}\empbfun\big) - \hat{L}_{\text{MM}} \big(\bfun^{\best},\bcT^{\best}\bfun^{\best}\big) \leq \alpha^{-1} Err_g. \end{equation}
	
	Combining \cref{eq:Err_decom_1} and \cref{eq:Err_decom_2}, we obtain \cref{eq:minimax_err_decomp}.
\end{proof}

\subsection{Analyzing Minimax Algorithm with Rademacher Complexity} \label{app:proof_minimax_Rad}
In what follows, we develop uniform concentration inequalities \cref{eq:bound_f,eq:bound_ga} using Rademacher complexities.
\begin{lemma} \label{lemma:Minimax_Rad_f}
	With probability at least $1 - \delta$,
	\[ \E_{\bmu} L_{\text{MM}} \big(\bfun, \bcT^{\best}\bfun\big) - \E_{\bmu} L_{\text{MM}} \big(\bfun^{\best},\bcT^{\best}\bfun^{\best}\big) \leq \big( \hat{L}_{\text{MM}} \big(\bfun, \bcT^{\best}\bfun\big) - \hat{L}_{\text{MM}} \big(\bfun^{\best},\bcT^{\best}\bfun^{\best}\big) \big) + Err_f \qquad \text{for any $\bfun \in \bFun$}, \]
	where
	\[ Err_f := c \sum_{h=1}^H \big( \popR_n^{\mu_h} (\cF_h) + \popR_n^{\mu_h} (\cG_h) + \popR_n^{\nu_h} ( V_{\cF_{h+1}}) \big) + 4H^2 \sqrt{\frac{2 \log(2/\delta)}{n}} \]
	for some universal constant $c > 0$.
\end{lemma}

\begin{proof}
	Note that $\big| L_{\text{MM}}\big(\bfun, \bcT^{\best}\bfun\big) - L_{\text{MM}}\big(\bfun^{\best}, \bcT^{\best}\bfun^{\best}\big) \big| \leq 8H^2$. We apply \Cref{lemma:RC} and find that
	\[ \begin{aligned} & \E_{\bmu} L_{\text{MM}} \big(\bfun, \bcT^{\best}\bfun\big) - \E_{\bmu} L_{\text{MM}} \big(\bfun^{\best},\bcT^{\best}\bfun^{\best}\big) \\ \leq & \big( \hat{L}_{\text{MM}} \big(\bfun, \bcT^{\best}\bfun\big) - \hat{L}_{\text{MM}} \big(\bfun^{\best},\bcT^{\best}\bfun^{\best}\big) \big) \\ & + 2 \popR_n \big( \big\{ L_{\text{MM}}\big(\bfun, \bcT^{\best}\bfun\big) - L_{\text{MM}}\big(\bfun^{\best}, \bcT^{\best}\bfun^{\best}\big) \, \big| \, \bfun \in \bFun \big\} \big) + 16H^2 \sqrt{\frac{2 \log(2/\delta)}{n}}. \end{aligned} \]
	Due to the symmetry of Rademacher random variables, we have
	\[ \popR_n \big( \big\{ L_{\text{MM}}\big(\bfun, \bcT^{\best}\bfun\big) - L_{\text{MM}}\big(\bfun^{\best}, \bcT^{\best}\bfun^{\best}\big) \, \big| \, \bfun \in \bFun \big\} \big) = \popR_n \big( \big\{ L_{\text{MM}}\big(\bfun, \bcT^{\best}\bfun\big) \, \big| \, \bfun \in \bFun \big\} \big). \]
	We now use \Cref{lemma:contraction_RC} to simplify the term $\popR_n \big( \big\{ L_{\text{MM}}\big(\bfun, \bcT^{\best}\bfun\big) \, \big| \, \bfun \in \bFun \big\} \big)$.
	
	Note that
	\[ \begin{aligned} L_{\text{MM}}\big(\bfun, \bcT^{\best}\bfun\big) = &  \frac{1}{H} \sum_{h=1}^H \boldsymbol{\phi}_h(\bfun)^{\top} \boldsymbol{A} \boldsymbol{\phi}_h(\bfun), \quad
	\text{where }\boldsymbol{A}:= \left(\!\! \begin{array}{ccc} 1 & -1 \\ -1 & 0 \end{array} \!\! \right), \\ &\qquad \boldsymbol{\phi}_h(\bfun) := \! \big( f_h(s_h,a_h) - \mathcal{T}_h^{\best} f_{h+1}(s_h,a_h), r_h + V_{f_{h+1}}(s_h') - \mathcal{T}_h^{\best} f_{h+1}(s_h,a_h) \big)^{\top}. \end{aligned} \]
	Since $\| \boldsymbol{\phi}_h(\bfun) \|_2 \! \leq \! \sqrt{2} H$ and $\|\boldsymbol{A}\|_2 \!=\! \frac{\sqrt{5}+1}{2}$, we learn that $\boldsymbol{\phi}_h(\bfun)^{\top} \boldsymbol{A} \boldsymbol{\phi}_h(\bfun)$ is \big($\frac{\sqrt{5}+1}{\sqrt{2}}H$\big)-Lipschitz with respect to $\boldsymbol{\phi}_h(\bfun)$.
	According to Lemma \ref{lemma:contraction_RC},
	\[ \begin{aligned} & \popR_n \big( \big\{ L_{\text{MM}}\big(\bfun, \bcT^{\best}\bfun\big) \, \big| \, \bfun \in \bFun \big\} \big) = \frac{1}{H} \sum_{h=1}^H \popR_n \big( \big\{\boldsymbol{\phi}_h(\bfun)^{\top} \boldsymbol{A} \boldsymbol{\phi}_h(\bfun) \, \big| \, \bfun \in \bFun \big\} \big) \\ \leq & (\sqrt{5} + 1) \sum_{h=1}^H \Big( \popR_n \big(\big\{ \phi_{h,1}(\bfun) \, \big| \, \bfun \in \bFun \big\} \big) + \popR_n \big(\big\{ \phi_{h,2}(\bfun) \, \big| \, \bfun \in \bFun \big\} \big) \Big). \end{aligned} \]
	Here,
	\[ \begin{aligned}
	\popR_n \big(\big\{ \phi_{h,1}(\bfun) \, \big| \, \bfun \in \bFun \big\} \big) = & \popR_n \big(\big\{ f_h - \mathcal{T}_h^{\best} f_{h+1} \, \big| \, f_h \in \cF_h, f_{h+1} \in \cF_{h+1} \big\} \big) \\ \leq & \popR_n \big(\big\{ f_h - g_h \, \big| \, f_h \in \cF_h, g_h \in \cG_h \big\} \big) \leq \popR_n^{\mu_h}(\cF_h) + \popR_n^{\mu_h}(\cG_h), \\ 
	\popR_n \big(\big\{ \phi_{h,2}(\bfun) \, \big| \, \bfun \in \bFun \big\} \big) = & \popR_n \big( \big\{ r_h + V_{f_{h+1}} - \mathcal{T}_h^{\best}f_{h+1} \, \big| \, f_{h+1} \in \cF_{h+1} \big\} \big) \\ \leq & \popR_n^{\nu_h} ( V_{\cF_{h+1}} ) + \popR_n \big( \big\{ \mathcal{T}_h^{\best}f_{h+1} \, \big| \, f_{h+1} \in \cF_{h+1} \big\} \big) \leq \popR_n^{\nu_h} (V_{\cF_{h+1}}) + \popR_n^{\mu_h}(\cG_h).
	\end{aligned} \]
	Integrating the pieces, we finish the proof of \Cref{lemma:Minimax_Rad_f}.
\end{proof}

%%%%%%%%%%%%%%%%%%%%%%%%%%%%%%%%%%%%%%%%%%%%%%%%%%%%

\begin{lemma} \label{lemma:Minimax_Rad_g}
	With probability at least $1 - \delta$, for any $\bg \in \bG$,
	\[	\E_{\bmu} \ell \big(\bg, \bfun^{\best} \big) - \E_{\bmu} \ell \big(\bcT^{\best}\bfun^{\best}, \bfun^{\best} \big) \leq \big( \hat{\ell} \big(\bg, \bfun^{\best} \big) - \hat{\ell} \big(\bcT^{\best}\bfun^{\best}, \bfun^{\best} \big) \big) + Err_g, \]
	where
	\[ Err_g := 8 \sum_{h=1}^H \popR_n^{\mu_h} ( \cG_h ) + 4H^2 \sqrt{\frac{2 \log(2/\delta)}{n}}. \]
\end{lemma}
\begin{proof}
	Note that $\big| \ell(\bg, \bfun^{\best}) - \ell\big(\bcT^{\best}\bfun^{\best}, \bfun^{\best}\big) \big| \leq 2H^2$. By \Cref{lemma:RC}, with probability at least $1 - \delta$,
	for any $\bg \in \bG$,
	\begin{equation} \label{eq:Minimax_Rad_ga_0} \begin{aligned}
	\E_{\bmu} \ell \big(\bg, \bfun^{\best} \big) - \E_{\bmu} \ell \big(\bcT^{\best}\bfun^{\best}, \bfun^{\best} \big) \leq & \big( \hat{\ell} \big(\bg, \bfun^{\best} \big) - \hat{\ell} \big(\bcT^{\best}\bfun^{\best}, \bfun^{\best} \big) \big) \\ & + 2\popR_n \big( \big\{\ell(\bg, \bfun^{\best}) - \ell\big(\bcT^{\best}\bfun^{\best}, \bfun^{\best} \big) \bigm| \bg \in \bG \big\} \big) + 4H^2 \sqrt{\frac{2 \log(2/\delta)}{n}}.
	\end{aligned} \end{equation}
	We observe that
	\begin{equation} \label{eq:Minimax_Rad_ga_1}
	\begin{aligned} 
	\popR_n \big( \big\{\ell(\bg, \bfun^{\best}) - \ell\big(\bcT^{\best}\bfun^{\best}, \bfun^{\best} \big) \, \big| \, \bg \in \bG \big\} \big)
	= \popR_n \big( \big\{ \ell(\bg, \bfun^{\best}) \, \big| \, \bg \in \bG \big\} \big) \leq \frac{1}{H} \sum_{h=1}^H \popR_n \big( \big\{ \ell(g_h,f_{h+1}^{\best}) \, \big| \, g_h \in \cG_h \big\} \big).
	\end{aligned}
	\end{equation}
	Similar to \cref{Lip}, we can show that $\ell(g_h, f_{h+1}^{\best})$ is ($4H$)-Lipschitz with respect to $g_h$, therefore,
	\begin{equation} \label{eq:Minimax_Rad_ga_2} \popR_n \big( \big\{ \ell(g_h, f_{h+1}^{\best}) \, \big| \, g_h \in \cG_h \big\}\big) \leq 4H \popR_n^{\mu_h}(\cG_h). \end{equation}
	Combining \cref{eq:Minimax_Rad_ga_0} - \cref{eq:Minimax_Rad_ga_2}, we complete the proof.
\end{proof}

\subsection{Analyzing Minimax Algorithm with Local Rademacher Complexity} \label{app:proof_minimax_locRad}
In this part, \Cref{lemma:Minimax_LRC_f,lemma:Minimax_LRC_f'} are devoted to the uniform concentration of $\hat{L}_{\text{MM}} \big(\bfun, \bcT^{\best}\bfun\big) - \hat{L}_{\text{MM}} \big(\bfun^{\best},\bcT^{\best}\bfun^{\best}\big)$ and \Cref{lemma:Minimax_LRC_ga} is concerned with $ \hat{\ell} \big(\bg, \bfun^{\best} \big) - \hat{\ell} \big(\bcT^{\best}\bfun^{\best}, \bfun^{\best} \big)$. The proof of \Cref{theorem:Minimax_LRC'} uses \Cref{lemma:Minimax_LRC_f,lemma:Minimax_LRC_ga}, while \Cref{theorem:Minimax_LRC} uses \Cref{lemma:Minimax_LRC_f',lemma:Minimax_LRC_ga}.

\paragraph{Concentration inequality \cref{eq:bound_f}, $\hat{L}_{\text{MM}} \big(\bfun, \bcT^{\best}\bfun\big) - \hat{L}_{\text{MM}} \big(\bfun^{\best},\bcT^{\best}\bfun^{\best}\big)$}
\Cref{lemma:Minimax_LRC_f} below will be used as a buiding block of the proof of \Cref{theorem:Minimax_LRC'}.
\begin{lemma} \label{lemma:Minimax_LRC_f}
	There exists a universal constant $c > 0$ such that under \Cref{as:approx_gcomp_FF}, for any fixed parameter $\theta > 1$, with probability at least $1 - \delta$, we have
	\begin{equation} \label{eq:Minimax_LRC_f}  \E_{\bmu} L_{\text{MM}} \big(\bfun, \bcT^{\best}\bfun\big) - \E_{\bmu} L_{\text{MM}} \big(\bfun^{\best},\bcT^{\best}\bfun^{\best}\big) \leq \frac{\theta}{\theta-1} \big( \hat{L}_{\text{MM}} \big(\bfun, \bcT^{\best}\bfun\big) - \hat{L}_{\text{MM}} \big(\bfun^{\best},\bcT^{\best}\bfun^{\best}\big) \big) + Err_f \end{equation}
	for any $\bfun \in \bFun$,
	with
	\[ Err_f := c \theta H^2 r_L^{\star} + c \theta H^2 \frac{\log(1/\delta)}{n} + \frac{c}{\theta-1}\big(\Berr(\bfun^{\best}) + \epsilon\big). \]
\end{lemma}

\begin{proof}% [Proof of Lemma \ref{lemma:Minimax_LRC_f}]
	We consider using Theorem \ref{theorem:LRC} to analyze the concentration of $\hat{L}_{\text{MM}} \big(\bfun, \bcT^{\best}\bfun\big) - \hat{L}_{\text{MM}} \big(\bfun^{\best},\bcT^{\best}\bfun^{\best}\big)$.
	Similar to \cref{Lip}, we can show that for any $\bfun \in \bFun$,
	\[ \big| L_{\text{MM}} \big(\bfun, \bcT^{\best}\bfun\big) \big| \leq 2\sum_{h=1}^H \big|f_h(s_h,a_h) - \mathcal{T}_h^{\best}f_{h+1}(s_h,a_h)\big|. \]
	By Cauchy-Schwarz inequality,
	\begin{equation} \label{eq:E_squareL} \begin{aligned} \mathbb{E} \big[ L_{\text{MM}} \big(\bfun, \bcT^{\best}\bfun\big)^2 \big] \leq 4H^2 \big\| \bfun - \bcT^{\best} \bfun \big\|_{\bmu}^2 \leq 8H^2 \Big( \big\| \bfun - \bcT^{\star} \bfun \big\|_{\bmu}^2 + \big\| \bcT^{\best} \bfun - \bcT^{\star} \bfun \big\|_{\bmu}^2 \Big) \leq 8H^2 \big( \Berr(\bfun) + \epsilon \big), \end{aligned} \end{equation}
	where we have used \Cref{as:approx_gcomp_FF}.
	It follows that
	\[ \begin{aligned} & {\rm Var}\big[ L_{\text{MM}}\big(\bfun, \bcT^{\best}\bfun\big) - L_{\text{MM}}\big(\bfun^{\best},\bcT^{\best}\bfun^{\best}\big) \big] \leq \mathbb{E}\big[ (L_{\text{MM}}\big(\bfun, \bcT^{\best}\bfun\big) - L_{\text{MM}}\big(\bfun^{\best},\bcT^{\best}\bfun^{\best}\big))^2 \big] \\ \leq & 2 \mathbb{E} \big[ L_{\text{MM}}\big(\bfun, \bcT^{\best}\bfun\big)^2 \big] + 2 \mathbb{E} \big[ L_{\text{MM}}\big(\bfun^{\best},\bcT^{\best}\bfun^{\best}\big)^2 \big] \leq 16H^2 \big( \Berr(\bfun) + \Berr(\bfun^{\best}) + 2 \epsilon \big). \end{aligned} \]
	We also learn from \cref{eq:Lmu} that
	\begin{equation} \label{eq:E_L} \mathbb{E} \big[ L_{\text{MM}}\big(\bfun, \bcT^{\best}\bfun\big) - L_{\text{MM}}\big(\bfun^{\best}, \bcT^{\best}\bfun^{\best}\big) \big] \geq \Berr(\bfun) - \Berr(\bfun^{\best}) - \epsilon. \end{equation}
	We combine \cref{eq:E_squareL} and \cref{eq:E_L} and find that
	\[ {\rm Var}\big[ L_{\text{MM}}\big(\bfun, \bcT^{\best}\bfun\big) - L_{\text{MM}}\big(\bfun^{\best},\bcT^{\best}\bfun^{\best}\big) \big] \leq 16H^2 \big( \mathbb{E} \big[ L_{\text{MM}}\big(\bfun, \bcT^{\best}\bfun\big) - L_{\text{MM}}\big[\bfun^{\best}, \bcT^{\best}\bfun^{\best} \big] \big] + 2 \Berr(\bfun^{\best}) + 3\epsilon \big). \]
	
	We now apply \Cref{theorem:LRC} and aim to find a sub-root function $\psi_L$ such that $\psi_L(r) \geq \widetilde{\psi}(r)$ for
	\begin{equation} \label{eq:def_tildpsi} \begin{aligned} \widetilde{\psi}(r) := & 16H^2 \popR_n \Big( \Big\{ L_{\text{MM}}\big(\bfun, \bcT^{\best}\bfun\big) - L_{\text{MM}}\big(\bfun^{\best},\bcT^{\best}\bfun^{\best} \big) \Bigm| \bfun \in \bFun, \\ & \qquad \qquad \qquad \qquad \qquad 16H^2 \big( \mathbb{E} \big[ L_{\text{MM}}\big(\bfun, \bcT^{\best}\bfun\big) - L_{\text{MM}}\big(\bfun^{\best}, \bcT^{\best}\bfun^{\best}\big) \big] + 2\Berr(\bfun^{\best}) + 3\epsilon \big) \leq r \Big\} \Big) \\ = & 16H^2 \popR_n \Big( \Big\{ L_{\text{MM}}\big(\bfun, \bcT^{\best}\bfun\big) \, \Big| \, \bfun \in \bFun, \\ & \qquad \qquad \qquad \qquad \qquad 16H^2 \big( \mathbb{E} \big[ L_{\text{MM}}\big(\bfun, \bcT^{\best}\bfun\big) - L_{\text{MM}}\big(\bfun^{\best}, \bcT^{\best}\bfun^{\best}\big) \big] + 2\Berr(\bfun^{\best}) + 3\epsilon \big) \leq r \Big\} \Big). \end{aligned} \end{equation}
	
	Note that by \cref{eq:E_squareL,eq:E_L}, we have
	\[ 16H^2 \big( \mathbb{E} \big[ L_{\text{MM}}\big(\bfun, \bcT^{\best}\bfun\big) - L_{\text{MM}}\big(\bfun^{\best}, \bcT^{\best}\bfun^{\best}\big) \big] + 2\Berr(\bfun^{\best}) + 3\epsilon \big) \geq 2\mathbb{E} \big[ L_{\text{MM}}\big(\bfun, \bcT^{\best}\bfun\big)^2 \big], \]
	therefore,
	\begin{align*} \widetilde{\psi}(r) \leq & 16H^2 \popR_n \Big( \Big\{ L_{\text{MM}}\big(\bfun, \bcT^{\best}\bfun\big) \, \Big| \, \bfun \in \bFun, 2\mathbb{E} \big[ L_{\text{MM}}\big(\bfun, \bcT^{\best}\bfun\big)^2 \big] \leq r \Big\} \Big) \leq 16H^2 \psi_L\Big(\frac{r}{2}\Big), \end{align*}
	where
	\[ \psi_L(r) = \popR_n \Big( \Big\{ L_{\text{MM}}\big(\bfun, \bcT^{\best}\bfun\big) \, \Big| \, \bfun \in \bFun, \mathbb{E} \big[ L_{\text{MM}}\big(\bfun, \bcT^{\best}\bfun\big)^2 \big] \leq r \Big\} \Big). \]
	Let $r_L^{\star}$ be the positive fixed point of $\psi_L$.
	\Cref{lemma:subroot} implies the positive fixed point of mapping $r \mapsto 16H^2 \psi_L\big(r/2\big)$ is upper bounded by $128H^4 r_L^{\star}$. We then obtain \cref{eq:Minimax_LRC_f} by applying \cref{eq:theoremLRC1} in \Cref{theorem:LRC}.
\end{proof}

While \Cref{lemma:Minimax_LRC_f} above uses the local Rademacher complexity of a composite function $L_{\text{MM}}(\bfun, \bcT^{\best} f)$, \Cref{lemma:Minimax_LRC_f'} below provides an alternative concentration inequality for $ \hat{L}_{\text{MM}} \big(\bfun, \bcT^{\best}\bfun\big) - \hat{L}_{\text{MM}} \big(\bfun^{\best},\bcT^{\best}\bfun^{\best}\big)$, which involves the complexities of $\cF_h$, $\cG_h$ and $V_{\cF_{h+1}}$.

\begin{lemma} \label{lemma:Minimax_LRC_f'}
	Suppose Assumptions~\ref{as:approx_gcomp_FF}~and~\ref{as:Ctilde} hold. There exists a universal constant $c > 0$ such that for any fixed parameter $\theta > 1$, with probability at least $1 - \delta$,
	\begin{align} \label{eq:Minimax_LRC_f'}  & \E_{\bmu} L_{\text{MM}} \big(\bfun, \bcT^{\best}\bfun\big) - \E_{\bmu} L_{\text{MM}} \big(\bfun^{\best},\bcT^{\best}\bfun^{\best}\big) \leq \frac{\theta}{\theta-1} \big( \hat{L}_{\text{MM}} \big(\bfun, \bcT^{\best}\bfun\big) - \hat{L}_{\text{MM}} \big(\bfun^{\best},\bcT^{\best}\bfun^{\best}\big) \big) + Err_f \end{align} for any $\bfun \in \bFun$, with \begin{align} Err_f := c\theta \widetilde{C} H^3 \sum_{h=1}^H \Big( r_{f,h}^{\star} + r_{g,h}^{\star} + \widetilde{r}_{f,h+1}^{\star} + \sqrt{\epsilon r_{g,h}^{\star}/\widetilde{C}} \Big) + c\theta H^2 \frac{\log(1/\delta)}{n} + \frac{c}{\theta-1}\big(\Berr(\bfun^{\best}) + \epsilon\big). \notag \end{align}
	Here, $\widetilde{C}$ is the concentrability coefficient in \Cref{as:Ctilde}.
\end{lemma}

\begin{proof}
	In this proof, we estimate the critical radius of $\widetilde{\psi}(r)$ in \cref{eq:def_tildpsi} in an alternative way. In particular, we use parameters $r_{f,h}^{\star}$, $r_{g,h}^{\star}$ and $\widetilde{r}_{f,h}^{\star}$ defined in the statement of \Cref{theorem:Minimax_LRC}. The key step is to upper bound $\widetilde{\psi}(r)$ by the local Rademacher complexities $\popR^{\mu_h}_n \big( \big\{ f_h \in \cF_h  ~\big|~ \norm{f_h - f^{\best}_h}_{\mu_h}^2 \leq r \big\} \big)$,
	$\popR^{\mu_h}_n \big( \big\{ g_h \in \cG_h  ~\big|~ \norm{g_h - g_h^{\best}}_{\mu_h}^2 \leq r \big\}\big)$ and
	$\popR^{\nu_h}_n \big( \big\{ V_{f_{h+1}} ~\big|~ f_{h+1} \in \cF_{h+1}, \norm{f_{h+1} - f^{\best}_{h+1}}^2_{\nu_h \times \Uniform(\cA)} \le r \big\} \big)$.
	
	We take a shorthand $\bFun(r) := \big\{ \bfun \in \bFun \, \big| \, 16H^2 \big( \mathbb{E} \big[ L_{\text{MM}}\big(\bfun, \bcT^{\best}\bfun\big) - L_{\text{MM}}\big(\bfun^{\best}, \bcT^{\best}\bfun^{\best}\big) \big] + 2\Berr(\bfun^{\best}) + 3\epsilon \big) \leq r \big\}$ and rewrite $\widetilde{\psi}(r)$ as $\widetilde{\psi}(r) = 16H^2 \popR_n \big( \big\{ L_{\text{MM}}\big(\bfun, \bcT^{\best}\bfun\big) \, \big| \, \bfun \in \bFun(r) \big\} \big)$.
	Similar to \Cref{lemma:Minimax_Rad_f}, one can show that there exists a univeral constant $c > 0$ such that
	\begin{align*} & \widetilde{r} \leq c H^2 \sum_{h=1}^H \big( \psi_{h,1}(r) + \psi_{h,2}(r) + \psi_{h,3}(r) \big). \end{align*}
	where $\psi_{h,1}(r) := \popR_n^{\mu_h}\big( \big\{ f_h \bigm| \bfun \in \bFun(r) \big\} \big)$, $\psi_{h,2}(r) := \popR_n^{\mu_h}\big( \big\{ \cT_h^{\best} f_{h+1} \bigm| \bfun \in \bFun(r) \big\} \big)$ and $\psi_{h,3}(r) := \popR_n^{\nu_h}\big( \big\{ V_{f_{h+1}} \bigm| \bfun \in \bFun(r) \big\} \big)$.
	In the sequel, we simplify $\psi_{h,1}$, $\psi_{h,2}$ and $\psi_{h,3}$.
	
	For any $\bfun \in \bFun(r)$, due to \cref{eq:E_L}, we have
	\begin{align*}
		\big\| (\bfun - \bcT^{\star} \bfun) - (\bfun^{\best} - \bcT^{\star} \bfun^{\best}) \big\|_{\bmu}^2 \leq & 2 \big\| \bfun - \bcT^{\star} \bfun \big\|_{\bmu}^2 + 2 \big\| \bfun^{\best} - \bcT^{\star} \bfun^{\best} \big\|_{\bmu}^2 = 2 \Berr(\bfun) + 2 \Berr(\bfun^{\best}) \\ \leq & 2 \E\big[ L_{\text{MM}}\big[\bfun,\bcT^{\best}\bfun\big] - L_{\text{MM}}\big(\bfun^{\best},\bcT^{\best}\bfun^{\best}\big) \big] + 4 \Berr(\bfun^{\best}) + 2 \epsilon \leq \frac{r}{8H^2}.
	\end{align*}
	We use \Cref{lemma:VQ-Vf} and find that under Assumptions~\ref{as:approx_gcomp_FF}~and~\ref{as:Ctilde}, for any $\bfun \in \bFun$,
	\begin{align*}
		& \big\| f_h - f^{\best}_h \big\|_{\mu_h}^2 \leq \frac{\widetilde{C}r}{8}, \\
		& \begin{aligned} \big\| \cT_h^{\best} f_{h+1} - \cT_h^{\best} f_{h+1}^{\best} \big\|_{\mu_h}^2 \leq & \big( \big\| \cT_h^{\star} f_{h+1} - \cT_h^{\star} f_{h+1}^{\best} \big\|_{\mu_h} + 2 \sqrt{\epsilon} \big)^2 \\ \leq & 2 \big\| \cT_h^{\star} f_{h+1} - \cT_h^{\star} f_{h+1}^{\best} \big\|_{\mu_h}^2 + 8 \epsilon \leq \frac{\widetilde{C}r}{4} + 8 \epsilon, \end{aligned} \\
		& \big\| f_{h+1} - f^{\best}_{h+1} \big\|_{\nu_h \times \Uniform(\cA)}^2 \leq \frac{\widetilde{C}r}{8}.
	\end{align*}
	It follows that
	\begin{align*}
		& \psi_{h,1}(r) = \popR_n^{\mu_h}\big( \big\{ f_h \Bigm| \bfun \in \bFun(r) \big\} \big) \leq \popR_n^{\mu_h}\Big( \Big\{ f_h \in \cF_h \Bigm| \big\| f_h - f^{\best}_h \big\|_{\mu_h}^2 \leq \frac{\widetilde{C}r}{8} \Big\} \Big), \\
		& \psi_{h,2}(r) = \popR_n^{\mu_h}\big( \big\{ \cT_h^{\best} f_{h+1} \bigm| \bfun \in \bFun(r) \big\} \big) \leq \popR_n^{\mu_h}\Big( \Big\{ g_h \in \cG_h \Bigm| \big\| g_h - \cT_h^{\best} f_{h+1}^{\best} \big\|_{\mu_h}^2 \leq \frac{\widetilde{C}r}{4} + 8\epsilon \Big\} \Big), \\ & \psi_{h,3}(r) = \popR_n^{\nu_h}\big( \big\{ V_{f_{h+1}} \bigm| \bfun \in \bFun(r) \big\} \big) \leq \popR_n^{\nu_h}\Big( \Big\{ V_{f_{h+1}} \Bigm| f_{h+1} \in \cF_{h+1}, \big\| f_{h+1} - f^{\best}_{h+1} \big\|_{\nu_h \times \Uniform(\cA)}^2 \leq \frac{\widetilde{C}r}{8} \Big\} \Big).
	\end{align*}
	Recall that $r_{f,h}^{\star}$, $r_{g,h}^{\star}$ and $\widetilde{r}_{f,h+1}^{\star}$ are respectively the fixed points of
	\begin{align*}
		& \psi_{f,h}(r) = \popR_n^{\mu_h} \big( \big\{ f_h \in \cF_h \bigm| \big \| f_h - f_h^{\best} \|_{\mu_h}^2 \leq r \big\} \big), \\ & \psi_{g,h}(r) = \popR_n^{\mu_h} \big( \big\{ g_h \in \cG_h \bigm| \big \| g_h - \cT_h^{\best} f_{h+1}^{\best} \|_{\mu_h}^2 \leq r \big\} \big) \qquad \text{and} \\ &
		\widetilde{\psi}_{f,h}(r) = \popR_n^{\nu_h} \big( \big\{ V_{f_{h+1}} \bigm| f_{h+1} \in \cF_{h+1}, \big\| f_{h+1} - f^{\best}_{h+1} \big\|_{\widetilde{\mu}_{h+1}}^2 \leq r \big\} \big).
	\end{align*}
	According to \Cref{lemma:subroot}, the positive fixed points of $\psi_{h,1}$, $\psi_{h,2}$ and $\psi_{h,3}$ are upper bounded by $8\widetilde{C} r_{f,h}^{\star}$, $4 \widetilde{C}r_{g,h}^{\star} + \sqrt{32\epsilon \widetilde{C}r_{g,h}^{\star}}$ and $8\widetilde{C}\widetilde{r}_{f,h}$, therefore, the critical radius $\widetilde{r}^{\star}$ of $\widetilde{\psi}(r)$ satisfies
	\begin{align*} \widetilde{r}^{\star} \leq & c^2H^4 \Bigg(\sum_{h=1}^H \Big( \sqrt{8\widetilde{C}r_{f,h}^{\star}} + \sqrt{4\widetilde{C}r_{g,h}^{\star}} + \sqrt[4]{32 \epsilon \widetilde{C}r_{g,h}^{\star}} + \sqrt{8\widetilde{C}\widetilde{r}_{f,h}^{\star}} \Big) \Bigg)^2 \\ \leq & c' \widetilde{C} H^5 \sum_{h=1}^H \Big( r_{f,h}^{\star} + r_{g,h}^{\star} + \widetilde{r}_{f,h}^{\star} + \sqrt{\epsilon r_{g,h}^{\star}/\widetilde{C}}\Big), \end{align*}
	where $c, c' > 0$ are universal constants.
	
	We then apply \cref{eq:theoremLRC1} in \Cref{theorem:LRC} and obtain \cref{eq:Minimax_LRC_f'}.
\end{proof}

\paragraph{Concentration inequality \cref{eq:bound_ga}, $\hat{\ell} \big(\bg, \bfun^{\best} \big) - \hat{\ell} \big(\bcT^{\best}\bfun^{\best}, \bfun^{\best} \big)$}
\begin{lemma} \label{lemma:Minimax_LRC_ga}
	Suppose Assumption \ref{as:approx_gcomp_FF} holds. Then there exists a universal constant $c>0$ such that for any fixed parameter $\theta > 1$, with probability at least $1 - \delta$,
	\begin{align} \label{eq:Minimax_LRC_ga} & \E_{\bmu} \ell \big(\bg, \bfun^{\best} \big) - \E_{\bmu} \ell \big(\bcT^{\best}\bfun^{\best}, \bfun^{\best} \big) \leq \frac{\theta}{\theta-1} \big( \hat{\ell} \big(\bg, \bfun^{\best} \big) - \hat{\ell} \big(\bcT^{\best}\bfun^{\best}, \bfun^{\best} \big) \big) + Err_g, \\
	& \text{with} \qquad Err_g := c\theta H \sum_{h=1}^H r_{g,h}^{\star} + c\theta H^2 \frac{\log(H/\delta)}{n} + \frac{c\epsilon}{\theta-1}. \notag \end{align}
\end{lemma}

\begin{proof}
	Note that
	\begin{align*} \ell\big(\bg, \bfun^{\best}\big) - \ell\big(\bcT^{\best}\bfun^{\best}, \bfun^{\best}\big) = & \frac{1}{H} \sum_{h=1}^H \big( \ell(g_h, f_{h+1}^{\best}) - \ell\big(\mathcal{T}_h^{\best}f_{h+1}^{\best}, f_{h+1}^{\best}\big) \big). \end{align*}
	We can analyze the concentration of $\ell(g_h, f_{h+1}^{\best}) - \ell\big(\mathcal{T}_h^{\best}f_{h+1}^{\best}, f_{h+1}^{\best}\big)$ in a way similar to \Cref{theorem:FQI_localRC}. It follows that for any $h \in [H]$, with probability at least $1 - \delta$,
	\begin{align*} & \E_{\bmu} \ell (g_h, f_{h+1}^{\best}) - \E_{\bmu} \ell \big(\mathcal{T}_h^{\best}f_{h+1}^{\best}, f_{h+1}^{\best}\big) \\ \leq & \frac{\theta}{\theta-1} \big( \hat{\ell} (g_h, f_{h+1}^{\best}) - \hat{\ell} \big(\mathcal{T}_h^{\best}f_{h+1}^{\best}, f_{h+1}^{\best}\big) \big) 
	+ 8 c_1 \theta H^2 r_{g,h}^{\star} + (2c_2 + 8 c_3 \theta) H^2\frac{\log(1/\delta)}{n} + \frac{2\epsilon}{\theta-1}, \end{align*}
	for any $g_h \in \cG_h$,
	where $c_1, c_2, c_3$ are the constants in \Cref{theorem:LRC}. By union bound, we can further derive \cref{eq:Minimax_LRC_ga}.
\end{proof}

\subsection{Proof of Theorems~ \ref{theorem:Minimax_RC}, \ref{theorem:Minimax_LRC}~and~\ref{theorem:Minimax_LRC'}} \label{app:proof_minimax_conclude}

\begin{proof}[Proof of \Cref{theorem:Minimax_RC}]
	Combining \Cref{lemma:Minimax_ErrorDecomp,lemma:Minimax_Rad_f,lemma:Minimax_Rad_g}, we obtain Theorem \ref{theorem:Minimax_RC}.
\end{proof}
\begin{proof}[Proof of \Cref{theorem:Minimax_LRC,theorem:Minimax_LRC'}]
	Plugging Lemmas \Cref{lemma:Minimax_LRC_f,lemma:Minimax_LRC_ga} into \Cref{lemma:Minimax_ErrorDecomp} yields that with probability at least $1 - \delta$,
	\[ \begin{aligned} \Berr(\empbfun) \leq & \min_{\bfun \in \bFun} \Berr(\bfun) + \epsilon + c \theta H^2 \Bigg( r_L^{\star} + \frac{1}{H} \sum_{h=1}^H r_{g,h}^{\star} + \frac{\log(H/\delta)}{n} \Bigg) + \frac{c}{\theta-1}\big(\Berr(\bfun^{\best}) + \epsilon\big) \end{aligned} \]
	for a universal constant $c > 0$.
	By letting
	\[ \theta := 1 + \sqrt{\frac{\Berr(\bfun^{\best}) + \epsilon}{cH^2\big( r_L^{\star} + \frac{1}{H}\sum_{h=1}^H r_{g_h}^{\star} + \frac{\log(H/\delta)}{n}\big)}}, \]
	we have
	\[ \begin{aligned} \Berr(\empbfun) \leq & \min_{\bfun \in \bFun} \Berr(\bfun) + \epsilon + cH^2 \Bigg( r_L^{\star} + \frac{1}{H} \sum_{h=1}^H r_{g_h}^{\star} + \frac{\log(H/\delta)}{n} \Bigg) \\ & + cH \sqrt{\big(\min_{\bfun \in \bFun} \Berr(\bfun^{\best}) + \epsilon\big)\Bigg(r_L^{\star} + \frac{1}{H}\sum_{h=1}^H r_{g_h}^{\star} + \frac{\log(H/\delta)}{n}\Bigg)}, \end{aligned} \]
	which finishes the proof of \Cref{theorem:Minimax_LRC'}.
	
	Similarly, by combining \Cref{lemma:Minimax_ErrorDecomp,lemma:Minimax_LRC_f',lemma:Minimax_LRC_ga}, we prove Theorem~\ref{theorem:Minimax_LRC}.
\end{proof}

\section{Examples (Propositions~\ref{prop:R_bound_finite}~to~\ref{prop:R_bound_sparse})}

In this part, we provide estimates for the (local) Rademacher complexities of four special function spaces, namely function class with finite elements, linear function space, kernel class and sparse linear space. The results presented here slightly generalize \Cref{prop:R_bound_finite,prop:R_bound_linear,prop:R_bound_kernel,prop:R_bound_sparse}.

\subsection{Function class with finite elements (Proposition~\ref{prop:R_bound_finite})}

\begin{lemma}[Full version of \Cref{prop:R_bound_finite}] \label{lemma:finite}
	Suppose $\cF$ is a discrete function class with $|\cF| < \infty$ and $f \in [0,D]$ for any $f \in \cF$. Then for any distribution $\rho$, \begin{equation} \label{finite_generalRC} \popR_n^{\rho} ( \cF ) \leq 2 D \max \bigg\{ \sqrt{\frac{\log|\cF|}{n}}, \frac{\log |\cF|}{n} \bigg\}. \end{equation}
	For any function $f^{\circ}$ with range in $[0,D]$, we have \begin{equation} \label{finite_0} \popR_n^{\rho} \big( \big\{ f \in \cF  \bigm|  \|f - f^{\circ}\|_{\rho}^2 \leq r \big\} \big) \leq \psi(r), \quad \text{where }\psi(r) := 2 \max \bigg\{ \sqrt{\frac{r\log|\cF|}{n}}, \frac{D \log |\cF|}{n} \bigg\}. \end{equation}
	$\psi$ is a sub-root function with positive fixed point
	\[ r^{\star} = \frac{2(D \vee 2) \log|\cF|}{n}. \]
\end{lemma}

We remark that \Cref{prop:R_bound_finite} is a corollary of \Cref{lemma:finite} with $D := H$.

%\finite*

In order to prove \Cref{lemma:finite}, we first present a preliminary lemma that will be used later. See \Cref{lemma:MGF}.
\begin{lemma} \label{lemma:MGF}
	Suppose a random variable $X$ satisfies $|X|\le D$ and $\E[X] = 0$. Then for any $\lambda >0$, we have 
	\begin{equation} \label{MGF_0} \E[e^{\lambda X}] \le \exp\left\{\lambda^2{\rm Var}[X] \left( \frac{e^{\lambda D}-1-\lambda D}{\lambda^2D^2}\right)\right\}.\end{equation}
\end{lemma}
\begin{proof}% [Proof of \Cref{lemma:MGF}] 
	Note that $X \leq D$ and the mapping $x \mapsto \frac{e^x-1-x}{x^2}$ is nondecreasing, therefore, $\frac{e^{\lambda X} - 1 - \lambda X}{\lambda^2 X^2} \leq \frac{e^{\lambda D} - 1 - \lambda D}{\lambda^2 D^2}$. It follows that
	\begin{equation} \label{MGF_1} \mathbb{E}[e^{\lambda X}] = 1 + \lambda \mathbb{E}[X] +  \lambda^2 \mathbb{E}\bigg[ X^2 \bigg(\frac{e^{\lambda X} - 1 - \lambda X}{\lambda^2 X^2}\bigg)\bigg] \leq 1 + \lambda^2 {\rm Var}[X] \bigg(\frac{e^{\lambda D} - 1 - \lambda D}{\lambda^2 D^2}\bigg), \end{equation}
	where we have used the fact $\mathbb{E}[X] = 0$. Since $1 + x \leq e^x$ for any $x \in \mathbb{R}$, \cref{MGF_1} implies \cref{MGF_0}.
\end{proof}
We are now ready to prove \Cref{lemma:finite}.
\begin{proof}[Proof of \Cref{lemma:finite}]
	We can easily see that \cref{finite_generalRC} is a corollary of \cref{finite_0} by letting $f^{\circ} = 0$ and $r = D^2$, therefore, we focus on proving \cref{finite_0}.
	By definition of Rademacher complexity and the symmetry of Rademacher variables, we have
	\begin{align*} \popR_n^{\rho} \big( \big\{ f \in \cF  \bigm|  \|f - f^{\circ}\|_{\rho}^2 \leq r \big\} \big) = & \popR_n^{\rho} \big( \big\{ f - f^{\circ} \in \cF  \bigm|  \|f - f^{\circ}\|_{\rho}^2 \leq r \big\} \big) \\ = & \mathbb{E} \max\Bigg\{ \frac{1}{n} \sum_{i=1}^n \sigma_i \big( f(X_i) - f^{\circ}(X_i) \big)  \biggm|  f \in \cF, \|f - f^{\circ}\|_{\rho}^2 \leq r \Bigg\}. \end{align*}
	For any $\lambda > 0$, it holds that
	\begin{equation} \label{finite_1} \begin{aligned} \popR_n^{\rho} \big( \big\{ f \in \cF  \bigm|  \|f - f^{\circ}\|_{\rho}^2 \leq r \big\} \big) = & \frac{1}{\lambda n} \mathbb{E} \log \max_{\begin{subarray}{c}f \in \cF : \\ \|f-f^{\circ}\|_{\rho}^2 \leq r\end{subarray}} \exp\Bigg\{ \lambda \sum_{i=1}^n \sigma_i \big(f(X_i) - f^{\circ}(X_i) \big)\Bigg\} \\ \leq & \frac{1}{\lambda n} \mathbb{E} \log \sum_{\begin{subarray}{c}f \in \cF : \\ \|f-f^{\circ}\|_{\rho}^2 \leq r\end{subarray}} \exp\Bigg\{ \lambda \sum_{i=1}^n \sigma_i \big(f(X_i) - f^{\circ}(X_i) \big)\Bigg\} \\ \leq & \frac{1}{\lambda n} \log \sum_{\begin{subarray}{c}f \in \cF : \\ \|f-f^{\circ}\|_{\rho}^2 \leq r\end{subarray}} \mathbb{E} \exp\Bigg\{ \lambda \sum_{i=1}^n \sigma_i \big(f(X_i) - f^{\circ}(X_i) \big)\Bigg\}, \end{aligned} \end{equation}
	where the last line is due to Jensen's inequality. Since $(\sigma_1, X_1), \ldots, (\sigma_n, X_n)$ are {\it i.i.d.} samples,
	\begin{equation} \label{finite_2} \mathbb{E} \exp\Bigg\{ \lambda \sum_{i=1}^n \sigma_i \big(f(X_i) - f^{\circ}(X_i) \big)\Bigg\} = \Big( \mathbb{E} \exp \big\{ \lambda \sigma_1 \big( f(X_1) - f^{\circ}(X_1) \big) \big\} \Big)^n. \end{equation}
	Note that $\big| \sigma_1 \big( f(X_1) - f^{\circ}(X_1) \big) \big| \leq D$ and $\mathbb{E}\big[ \sigma_1 \big( f(X_1) - f^{\circ}(X_1) \big) \big] = 0$ since $\mathbb{E}[\sigma_1] = 0$. For any $f \in \cF$ such that $\|f-f^{\circ}\|_{\rho}^2 \leq r$, we have ${\rm Var}\big[ \sigma_1 \big( f(X_1) - f^{\circ}(X_1) \big) \big] = \mathbb{E} \big[ \big( f(X_1) - f^{\circ}(X_1) \big)^2 \big] = \|f-f^{\circ}\|_{\rho}^2 \leq r$. We apply \Cref{lemma:MGF} and derive that
	\begin{equation} \label{finite_3} \mathbb{E} \exp \big\{ \lambda \sigma_1 \big( f(X_1) - f^{\circ}(X_1) \big) \big\} \leq \exp\bigg\{ \lambda^2 r \left( \frac{e^{\lambda D}-1-\lambda D}{\lambda^2D^2}\right) \bigg\}. \end{equation}
	Combining \cref{finite_1,finite_2,finite_3}, we obtain
	\begin{equation} \label{finite_4} \begin{aligned} \popR_n^{\rho} \big( \big\{ f \in \cF  \bigm|  \|f - f^{\circ}\|_{\rho}^2 \leq r \big\} \big) \leq & \frac{1}{\lambda n} \log \sum_{\begin{subarray}{c}f \in \cF : \\ \|f-f^{\circ}\|_{\rho}^2 \leq r\end{subarray}} \Big( \mathbb{E} \exp \big\{ \lambda \sigma_1 \big( f(X_1) - f^{\circ}(X_1) \big) \big\} \Big)^n \\ \leq & \frac{1}{\lambda n} \log \Bigg( |\cF| \exp \bigg\{ n \lambda^2 r \bigg( \frac{e^{\lambda D} - 1 - \lambda D}{\lambda^2 D^2} \bigg) \bigg\} \Bigg) \\ = & \frac{\log |\cF|}{\lambda n} + \lambda r \bigg( \frac{e^{\lambda D} - 1 - \lambda D}{\lambda^2 D^2} \bigg). \end{aligned} \end{equation}
	
	For $r \geq \frac{D^2 \log|\cF|}{n}$, by letting $\lambda := \sqrt{\frac{\log |\cF|}{r n}}$, \cref{finite_4} implies
	$\popR_n^{\rho} \big( \big\{ f \in \cF  \bigm|  \|f - f^{\circ}\|_{\rho}^2 \leq r \big\} \big) \leq 2 \sqrt{\frac{r \log |\cF|}{n}}$, where we have used the fact $\frac{e^x-1-x}{x^2} \leq 1$ for any $x \leq 1$.
	When $0 \leq r < \frac{D^2 \log|\cF|}{n}$, by letting $\lambda := \frac{1}{D}$, \cref{finite_4} ensures
	$\popR_n \big( \big\{ f \in \cF  \bigm|  P(f - f^{\circ})^2 \leq r \big\} \big) \leq \frac{2D \log|\cF|}{n}$.
	Integrating the pieces, we complete the proof of \cref{finite_0}.
	
	It is easy to see that the right hand side of \cref{finite_0} is a sub-root function with positive fixed point $\frac{2(D \vee 2) \log|\cF|}{n}$.
\end{proof}

\subsection{Linear Space (Proposition~\ref{prop:R_bound_linear})}

\begin{lemma}[Full version of \Cref{prop:R_bound_linear}] \label{lemma:linear}
	Let $\phi: \cS \times \cA \rightarrow \mathbb{R}^d$ be a feature map to $d$-dimensional Euclidean space and $\rho$ be a distribution over $\cS \times \cA$. Consider a function class
	\[ \cF = \big\{ f = w^{\top} \phi \bigm| w \in \mathbb{R}^d, \|f\|_{\rho}^2 \leq B \big\}, \]
	where $B > 0$. It holds that
	\[ \popR_n^{\rho}(\cF) \leq \sqrt{\frac{2Bd}{n}}. \]
	For any $f^{\circ} \in \cF$, we have
	\[ \popR_n^{\rho}\big( \big\{ f \in \cF \bigm| \|f - f^{\circ}\|_{\rho}^2 \leq r \big\} \big) \leq \psi(r), \quad \text{where $\psi(r) := \sqrt{\frac{2rd}{n}}$}. \]
	$\psi$ is sub-root and has a positive fixed point
	\[ r^{\star} = \frac{2d}{n}. \]
\end{lemma}
\Cref{prop:R_bound_linear} in \Cref{sec:examples} is a corollary to \Cref{lemma:linear}.
In \Cref{prop:R_bound_linear}, conditions $\|w\| \leq H$ and $\| \phi(s,a) \| \leq 1$ ensure $\|f\|_{\infty} \leq H$ for $f = w^{\top} \phi$ and therefore $\|f\|_{\rho}^2 \leq H^2$. By letting $B := H^2$ in \Cref{lemma:linear}, we obtain \Cref{prop:R_bound_linear}.

%\linear*

\begin{proof}[Proof of \Cref{lemma:linear}]
	\Cref{lemma:linear} can be viewed as a consequence of \Cref{lemma:kernel_FQI} in \Cref{sec:kernel}. Without loss of generality, suppose that $\phi$ is orthonormal in $L^2(\rho)$, that is, $\int_{\cS \times \cA} \phi_i(s,a) \phi_j(s,a) \rho(s,a) {\rm d}s {\rm d}a = \begin{cases}
		1 & \text{if $i = j$}, \\ 0 & \text{if $i \neq j$}.
	\end{cases}$
	Define a kernel function $k\big( (s,a), (s',a') \big) = \phi(s,a)^{\top} \phi(s',a')$. The RKHS associated with kernel $k$ is the linear space spanned by $\phi$ endorsed with inner product $\langle f, f' \rangle_{\mathcal{K}} := w^{\top} w'$ for $f = \phi^{\top} w$, $f' = \phi^{\top} w'$. In this way, we have $\|\cdot\|_{\rho} = \|\cdot\|_{\mathcal{K}}$. For any $f \in \cF$, $\|f\|_{\rho}^2 \leq B$ implies $\|f\|_{\mathcal{K}} \leq \sqrt{B}$. We apply the results in \Cref{lemma:kernel_FQI} with $D = \sqrt{B}$. It follows that $\popR_n^{\rho}(\cF) \leq \sqrt{\frac{2B}{n} \sum_{i=1}^{\infty} 1 \wedge (4\lambda_i)} \leq \sqrt{\frac{2Bd}{n}}$ and $\popR_n^{\rho}\big( \big\{ \bfun \in \bFun \bigm| \|f - f^{\circ}\|_{\rho}^2 \leq r \big\} \big) \leq \sqrt{\frac{2}{n} \sum_{i=1}^{\infty} r \wedge \big( 4B\lambda_k \big)} \leq \sqrt{\frac{2rd}{n}}$ since $\lambda_i = 0$ for $i > d$.
\end{proof}

\subsection{Kernel Class (Proposition~\ref{prop:R_bound_kernel})} \label{sec:kernel}

We now consider kernel class, that is, a sphere in an RKHS $\mathcal{H}$ associated with a positive definite kernel $k: \cX \times \cX \rightarrow \mathbb{R}$. In our paper, $\cX = \cS \times \cA$. Let $\rho$ be a distribution over $\cX$. We are interested in Rademacher complexities of function class
\begin{align} \label{eq:def_kernel_F} \cF = \big\{ f \in \mathcal{H}  \bigm|  \|f\|_{\mathcal{K}} \leq D, \|f\|_{\rho}^2 \leq B \big\}. \end{align}
Here, $\|\cdot\|_{\mathcal{K}}$ denotes the RKHS norm and $D, B \geq 0$ are some constants. Suppose that $\E_{\rho} k(X,X) < \infty$ for $X \sim \rho$. We define an integral operator $\mathscr{T}: L^2(\rho) \rightarrow L^2(\rho)$ as 
\[ \mathscr{T}f = \int k(\cdot, y) f(y) \rho(y) {\rm d}y. \]
It is easy to see that $\mathscr{T}$ is positive semidefinite and trace-class. Let $\{ \lambda_i \}_{i=1}^{\infty}$ be the eigenvalues of $\mathscr{T}$, arranging in a nonincreasing order. By using these eigenvalues, we have an estimate for (local) Rademacher complexities of $\cF$
in \Cref{lemma:kernel_FQI} below.

\begin{lemma}[Full version of \Cref{prop:R_bound_kernel}] \label{lemma:kernel_FQI}
	For function class $\cF$ defined in \cref{eq:def_kernel_F}, we have
	\begin{equation} \label{eq:linear_general} \popR_n^{\rho} ( \cF ) \leq \sqrt{\frac{2}{n} \sum_{i=1}^{\infty} B \wedge ( 4D^2 \lambda_i)}.
	\end{equation}
	Let $f^{\circ}$ be an arbitrary function in $\cF$. The local Rademacher complexity around $f^{\circ}$ satisfies
	\begin{align} \label{eq:linear_local} \popR_n^{\rho} \big( \big\{ f \in \cF  \bigm|  \|f - f^{\circ}\|_{\rho}^2 \leq r \big\} \big) \leq \psi(r), \quad \text{where } \psi(r) := \sqrt{\frac{2}{n} \sum_{i=1}^{\infty} r \wedge \big( 4D^2 \lambda_i \big)}. \end{align}
	$\psi$ is a sub-root function with positive fixed point \vspace{-.2em}
	\begin{equation} \label{eq:linear} r^{\star} \leq 2 \min_{j \in \mathbb{N}} \left\{ \frac{j}{n} + D \sqrt{\frac{2}{n} \sum_{i=j+1}^{\infty} \lambda_i} \right\}. \end{equation}
\end{lemma}

	In \Cref{prop:R_bound_kernel}, we assume that $k(x,x) \leq 1$ for any $x \in \cX$ and $\|f\|_{\mathcal{K}} \leq H$ for any $f \in \cF$. It is then guaranteed that $|f(x)| = \big| \langle f, k(\cdot,x) \rangle_{\mathcal{K}} \big| \leq \| f \|_{\mathcal{K}} \big\| k(\cdot,x) \big\|_{\mathcal{K}} = \| f \|_{\mathcal{K}} \sqrt{ k(x,x) } \leq H$, which further implies $\|f\|_{\rho}^2 \leq H^2$. To this end, \Cref{prop:R_bound_kernel} is a consequence of \cref{lemma:kernel_FQI} by taking $D := H$ and $B := H^2$.
	
%	\kernel*

	We remark on the rate of $r^{\star}$ with respect to sample size $n$. 
	Firstly, it is evident that $r^{\star} \lesssim n^{-\frac{1}{2}}$.
	When $\lambda_i \lesssim i^{-\alpha}$ for $\alpha>1$, $r_h^{\star}$ has order $n^{-\frac{\alpha}{1+\alpha}}$ which is typical in nonparametric estimation.
	When the eigenvalues $\{ \lambda_i \}_{i=1}^{\infty}$ decay exponentially quickly, {\it i.e.} $\lambda_i \lesssim \exp(-\beta i^{\alpha})$ for $\alpha, \beta > 0$, $r^{\star}$ can be of order $n^{-1}(\log n)^{1/\alpha}$.

Our proof of \Cref{lemma:kernel_FQI} is based on a classical result shown in \Cref{theorem:kernel}.
\begin{theorem}[Theorem 41 in \citet{mendelson2002geometric}] \label{theorem:kernel}
	For every $r > 0$, we have
	\[ \popR_n^{\rho}\big(\big\{ f \in \mathcal{H}  \bigm|  \|f\|_{\mathcal{K}} \leq 1, \|f\|_{\rho}^2 \leq r \big\}\big) \leq \sqrt{ \frac{2}{n} \sum_{i=1}^{\infty} r \wedge \lambda_i}. \]
\end{theorem}
Now we are ready to prove \Cref{lemma:kernel_FQI}.
\begin{proof}[Proof of \Cref{lemma:kernel_FQI}]
	Since \cref{eq:linear_general} is a corollary of \cref{eq:linear_local} by setting $r = B$, we only consider \cref{eq:linear_local,eq:linear}.
	
	Due to the symmetry of Rademacher random variables,
	\begin{equation} \label{eq:kernel_1} \popR_n^{\rho} \big( \big\{ f \in \cF  \bigm|  \| f - f^{\circ} \|_{\rho}^2 \leq r \big\} \big) = \popR_n^{\rho} \big( \big\{ f - f^{\circ} \bigm| f \in \cF, \| f - f^{\circ} \|_{\rho}^2 \leq r \big\} \big). \end{equation}
	Since $\|f\|_{\mathcal{K}} \leq D$ implies $\|f - f^{\circ}\|_{\mathcal{K}} \leq 2D$, we have $\cF \subseteq \big\{ f \in \mathcal{H}  \bigm|  \| f - f^{\circ} \|_{\mathcal{K}} \leq 2 D \big\}$. It follows that
	\[ \begin{aligned} \popR_n^{\rho} \big( \big\{ f \in \cF \bigm| \| f - f^{\circ} \|_{\rho}^2 \leq r \big\} \big) \leq & \popR_n^{\rho} \big( \big\{ f - f^{\circ} \bigm| f \in \mathcal{H}, \|f-f^{\circ}\|_{\mathcal{K}} \leq 2D, \| f - f^{\circ} \|_{\rho}^2 \leq r \big\} \big) \\ = & \popR_n^{\rho} \big( \big\{ f \in \mathcal{H} \bigm|  \|f\|_{\mathcal{K}} \leq 2D, \| f \|_{\rho}^2 \leq r \big\} \big) \\ \overset{f_h':=f_h/(2D)}{=} & 2D \cdot \popR_n^{\rho} \bigg( \bigg\{ f' \in \mathcal{H}  \biggm|  \|f'\|_{\mathcal{K}} \leq 1, \| f' \|_{\rho}^2 \leq \frac{r}{4D^2} \bigg\} \bigg), \end{aligned} \]
	where we have used the translational symmetry of RKHS $\mathcal{H}$. We apply \Cref{theorem:kernel} and derive that
	\[ \begin{aligned} \popR_n^{\rho} \big( \big\{ f \in \cF  \bigm|  \| f - f^{\circ} \|_{\rho}^2 \leq r \big\} \big) \leq & 2D  \sqrt{ \frac{2}{n} \sum_{i=1}^{\infty} \frac{r}{4D^2} \wedge \lambda_i} = \sqrt{ \frac{2}{n} \sum_{i=1}^{\infty} r \wedge \big(4D^2\lambda_i\big)} = \psi(r). \end{aligned} \]
	It is evident that $\psi$ is sub-root. In the following, we estimate the positive fixed point $r^{\star}$ of $\psi$.
	
	If $r \leq r^{\star}$, then $r \leq \psi(r)$, which implies
	\[ r^2 \leq \frac{2}{n} \sum_{i=1}^{\infty} r \wedge \big(4D^2 \lambda_i\big) \leq \frac{2}{n} \Bigg( j r + 4D^2 \sum_{i=j+1}^{\infty} \lambda_i \Bigg) \qquad \text{for any $j \in \mathbb{N}$}. \]
	Solving the quadratic inequality yields
	\[ r \leq \frac{2j}{n} + 2D \sqrt{\frac{2}{n} \sum_{i=j+1}^{\infty} \lambda_i} \qquad \text{for any $j \in \mathbb{N}$}. \]
	It ensures that
	\[ r^{\star} \leq 2 \min_{j \in \mathbb{N}} \left\{ \frac{j}{n} + D \sqrt{\frac{2}{n} \sum_{i=j+1}^{\infty} \lambda_i} \right\}. \]
\end{proof}

\subsection{Sparse Linear Class (Proposition~\ref{prop:R_bound_sparse})}

Let $\phi: \cS \times \cA \rightarrow \mathbb{R}^d$ be a $d$-dimensional feature map and $\rho$ be a distribution over $\cS \times \cA$. We are interested in function class
\[ \cF_s = \big\{ f = w^{\top} \phi  \bigm|  w \in \mathbb{R}^d, \|w\|_0 \leq s, \| f \|_{\rho}^2 \leq B \big\}. \]
In the following, we provide an estimate for (local) Rademacher complexities of $\cF_s$ based on the transportation $T_2$ inequality. \Cref{prop:R_bound_sparse} would be a special case of our result in this part since Gaussian distributions always satisfy $T_2$ inequality.

\paragraph{Notations}
We denote by $\alpha \subseteq [d]$ an index set with $s$ elements.
Let $\mathcal{I} := \big\{ \alpha \subseteq [d]  \bigm|  |\alpha| = s \big\}$.
Note that $|\mathcal{I}| \leq d^s$. For any $\alpha \in \mathcal{I}$, let $\phi_\alpha$ be the subvector of $\phi$ with $ \phi_{\alpha} := (\phi_{\alpha_1}, \phi_{\alpha_2}, \ldots, \phi_{\alpha_s})^{\top}$. Denote covariance matrix $\Sigma := \mathbb{E}_{\rho}\big[ \phi \phi^{\top} \big] \in \mathbb{R}^{d \times d}$. Let $\Sigma_\alpha := \mathbb{E}_{\rho}\big[ \phi_\alpha \phi_\alpha^{\top} \big] \in \mathbb{R}^{s \times s}$ be the principal submatrix of $\Sigma$ with indices given by $\alpha$. % Suppose $\Sigma_{\alpha} \succ 0$ for any $\alpha \in \mathcal{I}$.

We use Orlicz norms $\| \cdot \|_{\psi_1}$ and $\| \cdot \|_{\psi_2}$ in the spaces of random variables. For a real-valued random variable $X$, define $\| X \|_{\psi_1} := \inf\big\{ c > 0  \bigm|  \mathbb{E}[\exp(|X|/c)-1] \leq 1 \big\}$ and $\| X \|_{\psi_2} := \inf\big\{ c > 0  \bigm|  \mathbb{E}[\exp(X^2/c^2)-1] \leq 1 \big\}$. For a random vector $X \in \mathbb{R}^d$, define $\| X \|_{\psi_1} := \sup_{u \in \mathbb{S}^{d-1}} \| u^{\top} X \|_{\psi_1}$ and $\| X \|_{\psi_2} := \sup_{u \in \mathbb{S}^{d-1}} \| u^{\top} X \|_{\psi_2}$.

For any positive semidefinite (PSD) matrix $M \in \mathbb{R}^{d \times d}$, let $M^{\dagger}$ denote its Moore–Penrose inverse and $\sqrt{M^{\dagger}} \in \mathbb{R}^{d \times d}$ be the unique PSD matrix such that $\big(\sqrt{M^{\dagger}}\big)^2 = M^{\dagger}$. We define a $M^{\dagger}$-weighted vector norm $\|\cdot\|_{M^{\dagger}}$ as $\|{\bf x}\|_{M^{\dagger}} = \sqrt{{\bf x}^{\top} M^{\dagger} {\bf x}} := \big\| \sqrt{M^{\dagger}} {\bf x} \big\|_2$ for any ${\bf x} \in \mathbb{R}^d$.

For any two distributions $\mu$ and $\nu$ on a same metric space $(\mathbb{X}, d)$, we say a measure $p(X,Y)$ over $\mathbb{X} \times \mathbb{X}$ is a coupling of $\mu$ and $\nu$ if the marginal distributions of $p$ are $\mu$ and $\nu$ respectively, {\it i.e.} $p(\cdot,\mathbb{X}) = \mu$ and $p(\mathbb{X},\cdot) = \nu$.
The \emph{quadratic Wasserstein metric} of $\mu$ and $\nu$ is defined as
\[ W_2(\mu,\nu) := \inf_{p(X,Y) \in \mathcal{C}(\mu,\nu)} \sqrt{\E[d(X,Y)^2]}, \]
where $\mathcal{C}(\mu,\nu)$ is the collection of all couplings of $\mu, \nu$.

\paragraph{Main results}

Before the statement of main results, we first introduce the notion of $T_2$ property. See \Cref{def_T2} below.

\begin{definition}[$T_2(\sigma)$ distribution] \label{def_T2} Suppose that a probability measure $\rho$ on metric space $(\mathbb{X},d)$ satisfy the \emph{quadratic transportation cost} ($T_2$) inequality
\[ W_2(\rho,\nu) \leq \sqrt{2\sigma^2 \KL{\nu}{\rho}} \qquad \text{for all measures $\nu$ on $\mathbb{X}$}, \]
then we say $\rho$ is a $T_2(\sigma)$ distribution.
\end{definition}
We remark that $T_2$ is a broad class that contains many common distributions as special cases. For example, Gaussian distribution $\mathcal{N}(\cdot, M)$ satisfies $T_2\big(\sqrt{\|M\|_2}\big)$-inequality. Strongly log-concave distributions are $T_2$. Suppose $\rho$ is a continuous measure with a convex and compact support set. If its smallest density is lower bounded within the support, then $\rho$ is $T_2$.

We have an estimate of the (local) Rademacher complexities of $\cF_s$ in Lemma \ref{lemma:ex_sparse}.

\begin{lemma}[Full version of \Cref{prop:R_bound_sparse}] \label{lemma:ex_sparse}
	Suppose that for $X \sim \rho$, the distribution of $\phi_{\alpha}(X) \in \mathbb{R}^s$ satisfies $T_2\big(\sigma(\alpha)\big)$-inequality for any $\alpha \in \mathcal{I}$. Let $\sigma_{\min}^2(\alpha)$ be the smallest positive eigenvalue of $\Sigma_{\alpha} = \E_{\rho} \big[ \phi_{\alpha} \phi_{\alpha}^{\top} \big]$.
	Let $\eta_s$ be a constant such that $\eta_s \geq \sigma(\alpha)/\sigma_{\min}(\alpha)$ for any $\alpha \in \mathcal{I}$.
	There exists a universal constant $c > 0$ such that when $n \geq c s \log d$,
	\[ \popR_n^{\rho} (\cF_s) \leq c(1 + \eta_s) \sqrt{\frac{B s \log d}{n}}. \]
	Moreover, when $n \geq c s \log d$, for any $f^{\circ} \in \cF_s$, the local Rademacher complexity of $\cF_s$ satisfies
	\[ \popR_n^{\rho} \big( \big\{ f \in \cF_s  \bigm|  \|f - f^{\circ}\|_{\rho}^2 \leq r \big\} \big) \leq \psi(r), \qquad \text{with }\psi(r) := c\sqrt{r}(1 + \eta_s) \sqrt{\frac{s\log d}{n}}. \]
	Here, $\psi(r)$ is a sub-root function with a unique positive fixed point \[r^{\star} = c^2 (1 + \eta_s)^2 \cdot \frac{ s \log d}{n}. \]
\end{lemma}

When $\phi(X)$ follows a non-degenerated Gaussian distribution with covariance matrix $\Sigma \in \mathbb{R}^{d \times d}$, we have $\sigma(\alpha) \leq \sqrt{\lambda_{\max}(\Sigma_{\alpha})}$. Since $\E_{\rho}\big[\phi_{\alpha} \phi_{\alpha}^{\top} \big] \succeq \Sigma_{\alpha}$, it also holds that $\sigma_{\min}(\alpha) = \sqrt{\lambda_{\min}\big(\E_{\rho}\big[\phi_{\alpha} \phi_{\alpha}^{\top} \big]\big)} \geq \sqrt{\lambda_{\min}(\Sigma_{\alpha})}$. According to \Cref{lemma:ex_sparse}, we take a parameter $\kappa_s(\Sigma)$ such that $\kappa_s(\Sigma) \geq {\lambda_{\max}(\Sigma_{\alpha})/\lambda_{\min}(\Sigma_{\alpha})} \geq 1$ for all $\alpha \in \mathcal{I}$. In this way, the result in \Cref{lemma:ex_sparse} holds for $\eta_s = \sqrt{\kappa_s(\Sigma)}$ and reduces to \Cref{prop:R_bound_sparse}.

\paragraph{Proof of main results}

In the sequel, we prove \Cref{lemma:ex_sparse}. We first present some preliminary results.

\begin{lemma} \label{lemma:sparse_1}
	For arbitrary random variables $X_1, X_2, \ldots, X_m \geq 0$ ($m \geq 2$) satisfying $\| X_i \mathds{1}_{\{|X_i| \leq R\}} \|_{\psi_2} \leq \kappa_2$ and $\| X_i \mathds{1}_{\{|X_i| > R\}} \|_{\psi_1} \leq \kappa_1$ for $i = 1,2,\ldots,m$ and some parameter $R \geq 1$, we have
	\[ \E \max_{1 \leq i \leq m} X_i \leq c \Big( \kappa_2 \sqrt{\log m} + m (\kappa_1 + R) e^{-cR/\kappa_1}\Big), \]
	where $c > 0$ is a universal constant.
\end{lemma}
\begin{proof}
	We first note that $\E \big[\max_{1 \leq i \leq m} X_i\big] \leq U + V$ with $U := \E\big[\max_{1 \leq i \leq m} X_i \mathds{1}_{\{X_i \leq R\}}\big]$ and $V := \E\big[\max_{1 \leq i \leq m} X_i \mathds{1}_{\{|X_i| \geq R\}}\big]$. In what follows, we analyze $U$ and $V$ separately.
	
	By definition of $\psi_2$-norm and our assumption $\big\|X_i \mathds{1}_{\{ |X_i| \leq R \}}\big\|_{\psi_2} \leq \kappa_2$, we have $\E\big[\exp(X_i^2\mathds{1}_{\{|X_i| \leq R\}}/\kappa_2^2)-1\big] \leq 1$ for $i = 1,2,\ldots,m$. It follows that
	\begin{align*} \E \bigg[ \max_{1 \leq i \leq m} \frac{X_i^2 \mathds{1}_{\{|X_i| \leq R\}}}{\kappa_2^2} \bigg] \overset{\begin{subarray}{c} \text{Jensen's} \\ \text{inequality} \end{subarray}}{\leq} & \log \E \bigg[ \max_{1 \leq i \leq m} \exp\Big( \frac{X_i^2 \mathds{1}_{\{|X_i| \leq R\}}}{\kappa_2^2} \Big) \bigg] \\ \leq & \log \Bigg( \sum_{i=1}^m \E \bigg[ \exp\Big(\frac{X_i^2 \mathds{1}_{\{|X_i| \leq R\}}}{\kappa_2^2}\Big) \bigg] \Bigg) \leq \log(2m) \leq 2\log m. \end{align*}
	Therefore, by Jensen's inequality $U = \E\big[ \max_{1 \leq i \leq m} X_i \mathds{1}_{\{|X_i| \leq R\}} \big] \leq \sqrt{\E\big[ \max_{1 \leq i \leq m} X_i^2 \mathds{1}_{\{|X_i| \leq R\}} \big]} \leq \kappa_2\sqrt{2\log m}$.
	
	Recall that $\| X_i \mathds{1}_{\{|X_i|>R\}} \|_{\psi_1} \leq \kappa_1$, which implies there exists a universal constant $c \geq 1$ such that $\P\big( |X_i| \mathds{1}_{\{|X_i| \geq R\}} > t \big) \leq c e^{-ct/\kappa_1}$. Using this fact, we find that
	\begin{align*}
	    V \leq & \E\bigg[ \max_{1 \leq i \leq m} |X_i| \mathds{1}_{\{|X_i| > R\}} \bigg] = \bigg( \int_0^R +\int_R^{\infty} \bigg) \P\Big(\max_{1 \leq i \leq m} |X_i| \mathds{1}_{\{|X_i|>R\}} \geq t \Big) {\rm d}t
	    \\ = & R \P\Big( \max_{1 \leq i \leq m} |X_i| \mathds{1}_{\{|X_i| > R\}} \geq R \Big) + \int_R^{\infty} \P\Big( \max_{1 \leq i \leq m} |X_i| \mathds{1}_{|X_i| > R\}} \geq t \Big) {\rm d}t \\ \overset{\begin{subarray}{c}\text{union} \\ \text{bound} \end{subarray}}{\leq} & m R \P\big( |X_i| \mathds{1}_{\{|X_i| > R\}} \geq R \big) + m \int_R^{\infty} \P\big( |X_i| \mathds{1}_{\{|X_i| > R\}} \geq t \big) {\rm d}t \\ \leq & m R \cdot c e^{-cR/\kappa_1} + m \int_R^{\infty} c e^{-ct/\kappa_1} {\rm d}t = c mR e^{-cR/\kappa_1} + m \kappa_1 e^{-cR/\kappa_1} \leq c m(\kappa_1 + R)e^{-cR/\kappa_1}.
	\end{align*}
	
	Integrating the pieces, we finish the proof.
\end{proof}

\begin{lemma} \label{lemma:sparse_2}
    Let $X_1, X_2, \ldots, X_n \in \mathbb{R}^d$ be {\it i.i.d.} random vectors satisfying $T_2(\sigma)$-inequality and $\E[X_1X_1^{\top}] = M \in \mathbb{R}^{d \times d}$. Suppose that $n \geq d$. Let $\sigma_1, \sigma_2, \ldots, \sigma_n$ be Rademacher random variables independent of $X_1, X_2, \ldots, X_n$. Then $Y := \big\| \frac{1}{n} \sum_{k=1}^n \sigma_k X_k \big\|_{M^{\dag}}$ satisfies
    \[ \big\| |Y - \E[Y]| \mathds{1}\big\{|Y-\E[Y]| \leq \big(1+\sigma\textstyle{\sqrt{\|M^{\dagger}\|_2}}\big)\big\} \big\|_{\psi_2} \leq \displaystyle{c\bigg(\frac{1}{\sqrt{n}}+\sigma\sqrt{\frac{\|M^{\dagger}\|_2}{n}}\bigg)} \]
    \[ \text{and} \qquad \qquad \big\| |Y - \E[Y]| \mathds{1}\big\{|Y-\E[Y]| > \big(1+\sigma\textstyle{\sqrt{\|M^{\dagger}\|_2}}\big)\big\} \big\|_{\psi_1} \leq \displaystyle{c\bigg(\frac{1}{n}+\frac{\sigma\sqrt{\|M^{\dagger}\|_2}}{n}\bigg)}. \]
\end{lemma}
\begin{proof}
    We take shorthands ${\bf X} := [ X_1, X_2, \ldots, X_n ] \in \mathbb{R}^{d \times n}$, $\boldsymbol{\sigma} := (\sigma_1, \ldots, \sigma_n)^{\top} \in \mathbb{R}^n$ and rewrite $Y$ as $Y = \frac{1}{n} \| {\bf X} \boldsymbol{\sigma} \|_{M^{\dagger}} $. Note that $Y - \E Y = \big(Y - \E_{\boldsymbol{\sigma}}[Y \mid {\bf X}]\big) + \big(\E_{\boldsymbol{\sigma}}[Y \mid {\bf X}] - \E Y\big)$. In the following, we analyze these two terms separately.
    
    Note that $\nabla_{\sigma} Y = n^{-1} \|{\bf X}\boldsymbol{\sigma}\|_{M^{\dagger}}^{-1} {\bf X}^{\top} M^{\dagger} {\bf X} \boldsymbol{\sigma}$ and $\| \nabla_{\boldsymbol{\sigma}} Y \|_2 \leq n^{-1} \big\|\sqrt{M^{\dagger}}{\bf X}\big\|_2$, therefore, $Y$ is {($n^{-1} \big\|\sqrt{M^{\dagger}}{\bf X}\big\|_2$)-Lipschitz} with respect to $\sigma$ in the Euclidean norm. Moreover, $Y$ is convex in $\boldsymbol{\sigma}$ and the Rademacher random variables are independent and bounded. We use Talagrand's inequality (See Theorem~4.20 and Corollary~4.23 in \cite{van2014probability}.) and obtain that there exists a universal constant $c > 0$ such that
    \begin{align} \label{eq:sparse_1} \P\Big( \big| Y - \E_{\boldsymbol{\sigma}}[Y \mid {\bf X}] \big| \geq t_1 n^{-1} \big\| \sqrt{M^{\dagger}}{\bf X} \big\|_2 \Bigm| {\bf X} \Big) \leq c e^{-ct_1^2} \qquad \text{for any $t_1 > 0$}. \end{align}
    We next consider the concentration of $\big\| \sqrt{M^{\dagger}}{\bf X} \big\|_2$. For random vector $X$, we define $$\|X\|_{\psi_2} := \sup_{{\bf u}\in\mathbb{R}^d, \|{\bf u}\|_2 \leq 1} \| {\bf u}^{\top} X \|_{\psi_2}.$$ Since $X$ satisfies $T_2(\sigma)$-inequality, according to Gozlan's theorem (Theorem~4.31 in \citet{van2014probability}), we find that $\big\| \sqrt{M^{\dagger}} (X - \E X) \big\|_{\psi_2} \leq c \sigma \sqrt{\|M^{\dagger}\|_2}$ for some universal constant $c > 0$.
    Additionally, we have $\big\|\sqrt{M^{\dagger}} \E X\big\|_2 \leq \sqrt{\big\| \E\big[ \sqrt{M^{\dagger}} XX^{\top} \sqrt{M^{\dagger}} \big] \big\|_2} = 1$. Therefore, $\big\| \sqrt{M^{\dagger}} X \big\|_{\psi_2} \leq \big\| \sqrt{M^{\dagger}} (X - \E X) \big\|_{\psi_2} + \big\| \sqrt{M^{\dagger}} \E X \big\|_2 \leq 1 + c\sigma\sqrt{\|M^{\dagger}\|_2}$. We now apply Theorem~5.39 in \citet{vershynin2010introduction} and obtain that
    \[ \P\Big(\big\| \sqrt{M^{\dagger}} {\bf X} \big\|_2 \geq \sqrt{n} + c(\sqrt{d}+t)\big\|\sqrt{M^{\dagger}}X\big\|_{\psi_2}\Big) \leq ce^{-ct^2}, \]
    which further implies
    \begin{align} \label{eq:sparse_2} \P\Big(\big\| \sqrt{M^{\dagger}} {\bf X} \big\|_2 \geq \sqrt{n} +  c\big(\sqrt{d}+t_2\big) \big(1+\sigma\textstyle{\sqrt{\|M^{\dagger}\|_2}}\big) \Big) \leq ce^{-ct_2^2} \qquad \text{for all $t_2 > 0$}. \end{align}
    Combining \cref{eq:sparse_1} and \cref{eq:sparse_2}, we learn that
    \begin{align} \label{eq:sparse_A} \P\Big( \big| Y - \E_{\boldsymbol{\sigma}}[Y \mid {\bf X}] \big| \geq t_1n^{-\frac{1}{2}} + c t_1 n^{-1} \big(\sqrt{d}+t_2\big) \big(1+\sigma\textstyle{\sqrt{\|M^{\dagger}\|_2}}\big) \Big) \leq c \big(e^{-ct_1^2}+e^{-ct_2^2}\big). \end{align}
    
    As for the second term $\E_{\boldsymbol{\sigma}}[Y \mid {\bf X}] - \E Y$, we use the $T_2(\sigma)$ property of sample distribution and Gozlan's theorem (Theorem~4.31 in \citet{van2014probability}). We first show that $\E_{\boldsymbol{\sigma}}[Y \mid {\bf X}]$ is $\sqrt{\frac{\|M^{\dagger}\|_2}{n}}$-Lipschitz with respect to Frobenius norm $\|\cdot\|_F$. In fact,
    \begin{align*}
        & \big| \E_{\sigma}[Y \mid {\bf X}] - \E_{\boldsymbol{\sigma}}[Y \mid {\bf X}'] \big| = \frac{1}{n} \big| \E_{\sigma} \| {\bf X} \boldsymbol{\sigma} \|_{M^{\dagger}} - \E_{\boldsymbol{\sigma}}\| {\bf X}' \sigma \|_{M^{\dagger}} \big| \leq \frac{1}{n} \E_{\boldsymbol{\sigma}} \big| \| {\bf X} \boldsymbol{\sigma} \|_{M^{\dagger}} - \| {\bf X}' \boldsymbol{\sigma} \|_{M^{\dagger}} \big| \\ \leq & \frac{1}{n} \E_{\boldsymbol{\sigma}} \big\| ({\bf X} - {\bf X}') \boldsymbol{\sigma} \big\|_{M^{\dagger}} \leq \frac{1}{n} \sqrt{\| M^{\dagger} \|_2} \| {\bf X} - {\bf X}'\|_2 \E_{\boldsymbol{\sigma}} \|\boldsymbol{\sigma}\|_2 \leq \sqrt{\frac{\|M^{\dagger}\|_2}{n}} \|{\bf X} - {\bf X}'\|_F.
    \end{align*}
    We then apply Gozlan's theorem and find that there exists a universal constant $c > 0$ such that
    \begin{align} \label{eq:sparse_B} \P\Bigg( \big| \E_{\sigma}[Y \mid {\bf X}] - \E[Y] \big| \geq t_1 \sigma \sqrt{\frac{\|M^{\dagger}\|_2}{n}} \Bigg) \leq c e^{-ct_1^2} \qquad \text{for any $t_1 > 0$}. \end{align}
    
    Integrating \cref{eq:sparse_A} and \cref{eq:sparse_B} and using the condition $n \geq d$, we find that
    \[ \P\Big(\big| Y - \E[Y] \big| \geq t_1 n^{-\frac{1}{2}} \big(1 + c n^{-\frac{1}{2}} t_2\big) \big(1+\sigma\textstyle{\sqrt{\|M^{\dagger}\|_2}}\big) \Big) \leq c \big(e^{-ct_1^2}+e^{-ct_2^2}\big). \]
    If $0 \leq t_1 \leq \sqrt{n}$, then by letting $t_2 = \sqrt{n}$, we have
    \[ \P\Big( \big| Y - \E [Y] \big| \geq c t_1 n^{-\frac{1}{2}} \big(1+\sigma\textstyle{\sqrt{\|M^{\dagger}\|_2}}\big) \Big) \leq c e^{-ct_1^2}. \]
    Otherwise, when $t_1 > \sqrt{n}$, we take $t_2 = t_1$ and obtain
    \[ \P\Big( \big|Y - \E[Y]\big| \geq c t_1^2 n^{-1} \big(1+\sigma\textstyle{\sqrt{\|M^{\dagger}\|_2}}\big) \Big) \leq c e^{-ct_1^2}. \]
    We then finish the proof by combining these two cases.
\end{proof}

We are now ready to prove Lemma \ref{lemma:ex_sparse}.
\begin{proof}[Proof of Lemma \ref{lemma:ex_sparse}]
	Note that $\popR_n^{\rho} \big( \big\{ f \in \cF_s  \bigm|  \|f - f^{\circ}\|_{\rho}^2 \leq r \big\} \big) = \popR_n^{\rho} \big( \big\{ f - f^{\circ} \bigm| f \in \cF_s,  \|f - f^{\circ}\|_{\rho}^2 \leq r \big\} \big) \leq \popR_n^{\rho} \big( \big\{ f \in \cF_{2s} \bigm|  \|f\|_{\rho}^2 \leq r \big\} \big)$. Therefore, we can easily obtain upper bounds for $\popR_n^{\rho} \big( \big\{ f \in \cF_s  \bigm|  \|f - f^{\circ}\|_{\rho}^2 \leq r \big\} \big)$ by analyzing $\popR_n^{\rho} \big( \big\{ f \in \cF_s  \bigm|  \|f\|_{\rho}^2 \leq r \big\} \big)$. To this end, in the following, we focus on the local Rademacher complexity \[ \popR_n^{\rho} \big( \big\{ f \in \cF_s  \bigm|  \|f\|_{\rho}^2 \leq r \big\} \big). \] 

	To simplify the notation, we write $x := (s,a)$.
	Note that
	\[ \begin{aligned} \popR_n^{\rho} \big( \big\{ f \in \cF_s \bigm|  \|f\|_{\rho}^2 \leq r \big\} \big) = & \mathbb{E} \sup \Bigg\{ \frac{1}{n} \sum_{k=1}^K \sigma_k f(x_k)  \biggm|  f \in \cF, \|f\|_{\rho}^2 \leq r \Bigg\} \\ = & \mathbb{E} \sup \Bigg\{ \frac{1}{n} \sum_{k=1}^n \sigma_k \phi_{\alpha}(x_k)^{\top} w  \biggm|  \alpha \in \mathcal{I}, w \in \mathbb{R}^s, w^{\top} \Sigma_{\alpha} w \leq r \Bigg\}. \end{aligned} \]
	We fix $\alpha$, $\{\sigma_k\}_{k=1}^n$ and $\{ x_k \}_{k=1}^n$ and then optimize $w \in \mathbb{R}^s$. Since $x_k \in {\rm supp}(\rho)$, one always has $\frac{1}{n} \sum_{k=1}^K \sigma_k \phi_{\alpha}(x_k) \in {\rm range}(\Sigma_{\alpha})$ with probability one. The supremum is therefore acheived at \[ w := \frac{\sqrt{r} \Sigma_{\alpha}^{\dag}\big[ \frac{1}{n} \sum_{k=1}^n \sigma_k \phi_{\alpha}(x_k) \big]}{\big\| \frac{1}{n} \sum_{k=1}^n \sigma_k \phi_{\alpha}(x_k) \big\|_{\Sigma_{\alpha}^{\dag}}}. \]
	It follows that
	\[ \popR_n^{\rho} \big( \big\{ f \in \cF_s  \bigm|  \|f\|_{\rho}^2 \leq r \big\} \big) = \sqrt{r} \mathbb{E} \max_{\alpha \in \mathcal{I}} Y_{\alpha}, \qquad \text{where }Y_{\alpha} := \Bigg\| \frac{1}{n} \sum_{k=1}^n \sigma_k \phi_{\alpha}(x_k) \Bigg\|_{\Sigma_{\alpha}^{\dag}}. \]
	
	We further upper bound the local Rademacher complexity by
	\begin{equation} \label{eq:sparse_5'} \popR_n^{\rho} \big( \big\{ f \in \cF_s  \bigm|  \|f\|_{\rho}^2 \leq r \big\} \big) \leq \sqrt{r} \bigg(\underbrace{\max_{\alpha \in \mathcal{I}} \mathbb{E}[Y_{\alpha}]}_{E_1} + \underbrace{\mathbb{E} \Big[ \max_{\alpha \in \mathcal{I}} \big( Y_{\alpha} - \mathbb{E}[Y_{\alpha}] \big) \Big]}_{E_2}\bigg). \end{equation}
	In the following, we estimate the two terms in the right hand side of \cref{eq:sparse_5'} separately.
	
	Define $\boldsymbol{\sigma} := (\sigma_1, \ldots, \sigma_n)^{\top} \in \mathbb{R}^n$ and $\Phi_{\alpha} := \big[ \phi_{\alpha}(x_1), \ldots, \phi_{\alpha}(x_n) \big] \in \mathbb{R}^{s \times n}$. We reform $Y_{\alpha}$ as $Y_{\alpha} = n^{-1} \| \Phi_{\alpha} \boldsymbol{\sigma} \|_{\Sigma_{\alpha}^{\dag}}$. It follows that
	\[ \mathbb{E}[Y_{\alpha}^2] = \frac{1}{n^2} {\mathbb{E} \big[\| \Phi_{\alpha} \boldsymbol{\sigma} \|_{\Sigma_{\alpha}^{\dag}}^2\big]}  = \frac{1}{n^2} {\mathbb{E} \big[ (\Phi_{\alpha} \boldsymbol{\sigma})^{\top} \Sigma_{\alpha}^{\dag} (\Phi_{\alpha} \boldsymbol{\sigma}) \big]} = \frac{1}{n^2} {\mathbb{E} \big[ {\rm Tr}( \Sigma_{\alpha}^{\dag} \Phi_{\alpha} \boldsymbol{\sigma}\boldsymbol{\sigma}^{\top}\Phi_{\alpha}^{\top}) \big]}. \]
	We use the relations $\frac{1}{n} \mathbb{E}\big[ \Phi_{\alpha} \Phi_{\alpha}^{\top} \big] = \Sigma_{\alpha}$ and $\mathbb{E}[\boldsymbol{\sigma} \boldsymbol{\sigma}^{\top}] = I_s$ where $I_r$ represents the identity matrix in $\mathbb{R}^{s \times s}$. The inequality above is then reduced to
	\begin{equation} \label{eq:sparse_4'} \mathbb{E}[Y_{\alpha}] \leq \sqrt{\mathbb{E}[Y_{\alpha}^2]} \leq \frac{1}{\sqrt{n}} \sqrt{{\rm rank}(\Sigma_{\alpha})} \leq \sqrt{\frac{s}{n}}. \end{equation}
	To this end, we have $E_1 \leq \sqrt{s/n}$.
	
	Now we focus on $E_2$.
	Since $\phi_{\alpha}(x)$ satisfies $T_2\big(\sigma(\alpha)\big)$-inequality. Applying \Cref{lemma:sparse_2}, we find that if $n \geq s$,
	\[ \big\| |Y_{\alpha} - \E[Y_{\alpha}]| \mathds{1}\big\{|Y_{\alpha}-\E[Y_{\alpha}]| \leq \big(1+\sigma(\alpha)/\sigma_{\min}(\alpha)\big)\big\} \big\|_{\psi_2} \leq \frac{c}{\sqrt{n}}\bigg(1+\frac{\sigma(\alpha)}{\sigma_{\min}(\alpha)}\bigg) \leq \frac{c}{\sqrt{n}}(1+\eta_s) \]
    \[ \text{and} \qquad \qquad \big\| |Y_{\alpha} - \E[Y_{\alpha}]| \mathds{1}\big\{|Y_{\alpha}-\E[Y_{\alpha}]| > \big(1+\sigma(\alpha)/\sigma_{\min}(\alpha)\big)\big\} \big\|_{\psi_1} \leq \frac{c}{n} \bigg(1+\frac{\sigma(\alpha)}{\sigma_{\min}(\alpha)}\bigg) \leq \frac{c}{n}(1+\eta_s). \]
    We further use \Cref{lemma:sparse_2} and obtain
    \begin{align*} \E \max_{\alpha \in \mathcal{I}} |Y_{\alpha} - \E[Y_{\alpha}]| \leq & c(1+\eta_s)\Big( n^{-\frac{1}{2}} \sqrt{\log|\mathcal{I}|} + |\mathcal{I}| e^{-cn} \Big) \\ \leq & c(1+\eta_s)\Big( n^{-\frac{1}{2}} \sqrt{s\log d} + \exp\big(-cn + s \log d\big) \Big). \end{align*}
    If $n \geq c' s \log d$ for some sufficiently large constant $c'$, then
    \begin{align} \label{eq:sparse_3} E_2 = \E \max_{\alpha \in \mathcal{I}} \big|Y_{\alpha} - \E[Y_{\alpha}]\big| \leq c(1+\eta_s)\sqrt{\frac{s\log d}{n}}. \end{align}
    
    Plugging \cref{eq:sparse_4',eq:sparse_3} into \cref{eq:sparse_5'}, we complete our proof.
\end{proof}

% !TEX root = main.tex

\section{Proof of Lower Bound (Theorem~\ref{theorem:lb})} 
\label{app:proof_thm_lb}

In this section, we will prove a stronger version of \Cref{theorem:lb}, which is \Cref{thm:lower_bound_main2}. In \Cref{thm:lower_bound_main2}, we show that in the same setting as \Cref{theorem:lb}, even if additionally assuming \Cref{as:coverage} holds with $C=1$, i.e., $\mu_h$ is the true marginal distribution of the single-action MDP, and the algorithm knows $\{\mu_h\}_{h=1}^H$, it still takes $\Omega(\frac{\sqrt{S}}{\eps^2})$ samples for the learning algorithm $\mathfrak{A}$ to achieve $\eps$ optimality gap for Bellman error. This further justifies the necessity of \Cref{as:approx_comp_FF} and \Cref{as:approx_gcomp_FF} in the single sampling regime.

\lb*

\begin{theorem}\label{thm:lower_bound_main2}
For any $\eps<0.5$ and $S\ge 2$, there is a family of single-action, $S+5$-state MDPs ($H=3$) with the same underlying distributions $\mu_h$ (satisfying \Cref{as:coverage} with $C=1$) and the same reward function (thus the MDPs only differ in probabiilty transition matrices) and a function class $\cF$ of size $2$, such that all learning algorithm $\mathfrak{A}$ that takes $n$ pairs of states $(s,a,r,s')$ and output a value function in $\cF$ must suffer $\Omega(\eps^2)$ expected optimality gap in terms of mean-squared bellman error w.r.t $\mu$ if $n=O(\frac{\sqrt{S}}{\eps^2})$.

Mathematically, it means for any learning algorithm $\mathfrak{A}$, there is a single-action, $S+5$-state MDP defined above, such that for $D=\cup_h\{(s_i,a_i,r_i,s'_i,h)\}_{i=1}^n$ sampled from $\cM$ and $\mu$, if $n=O(\frac{\sqrt{S}}{\eps^2})$, we have
 \[  \E_{D}\left[\Berr_{\cM}\left(\mathfrak{A}(D)\right)\right]\ge \min_{f\in\cF}\Berr_{\cM}\left(f\right)  + \Omega(\eps^2). \]
\end{theorem}

Below we will prove \Cref{thm:lower_bound_main2}. To better illustrate the idea of the hard instance, we will first prove a slightly weaker version with $C=2$ (\Cref{thm:lower_bound_main}) in \Cref{subsec:C=2} and in \Cref{subsec:proof_full_lb} we will prove \Cref{thm:lower_bound_main2} by slightly twisting the proof in \Cref{subsec:C=2}.

\subsection{Warm-up with $C=2$}\label{subsec:C=2}
\newcommand{\startstate}{s_{\text{start}}}

We construct the hard instances for single sampling in the following way. 

\paragraph{Hard Instance Construction:} We first generate a uniform random bit $c\in\{-1,1\}$, and a Radamacher vector  $\vsigma\in\{\pm1\}^{S}$. For each $c,\vsigma$, we define MDP $\cM^\eps_{c,\vsigma} = (\cS,\cA,H,\P^\eps_{c,\vsigma},r)$ below, where $0< \eps < 0,5$. The claim is the distribution of $\cM^\eps_{c,\vsigma}$ serves as the distribution of hard instances.   Note that only $\P^\eps_{c,\vsigma}$ in the tuple defining $\cM^\eps_{c,\vsigma}$ depends on $c$ and $\vsigma$. Here the probability transition matrix $\cM^\eps_{c,\vsigma}$ is the same for all $h=1,2,\ldots,H$.

Let $\mathcal{S} = \{\startstate\}\cup \{1,\ldots,{S}\}\cup \{t_{j,k}\}_{j,k\in\{-1,1\}}$, $H=2$, $|\cA|=1$ and the initial state is $\startstate$. Since there's only one action, below we will just drop the dependence on action and thus simplify the notation. We will always define the probability transition matrix in the way such that in the $2$nd step, we will reach some state among $1,\ldots,S$ and in the $3$rd step, we will reach some state among $t_{j,k}$.

\begin{figure}[!htbp]
     \centering
     \vspace{-.6cm}
     \begin{subfigure}[b]{0.5\textwidth}
         \centering
         \includegraphics[width=\textwidth]{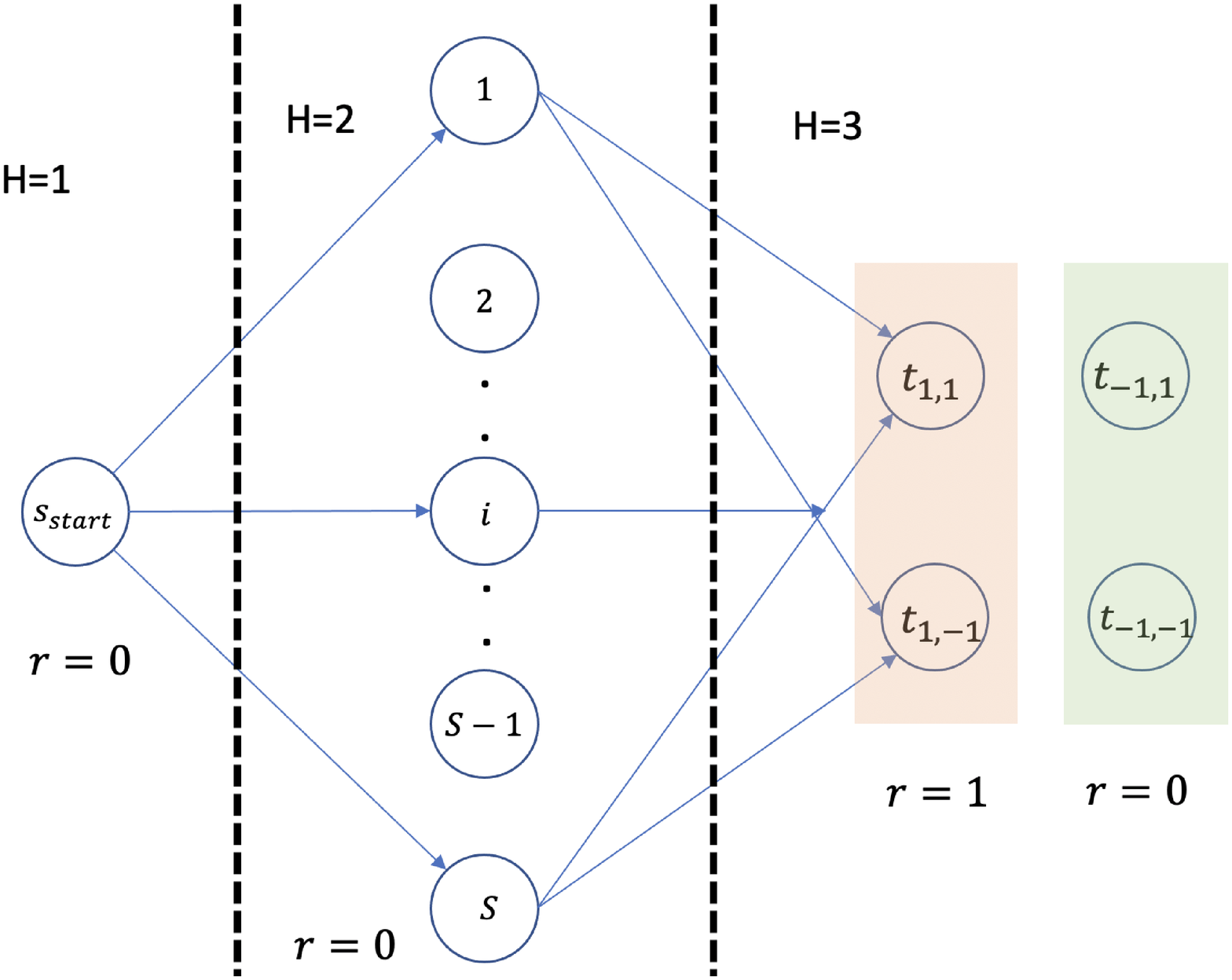}
		\caption{Illustration of the $3$-stage, single action MDP. Each state can be visited for at most one $h=1,2,3$. $r$ is the reward for each state.(action omitted since there's only one)}
         \label{fig:model}
     \end{subfigure}
     \hspace{3em}
     \begin{subfigure}[b]{0.35\textwidth}
         \centering
         \includegraphics[width=0.7\textwidth]{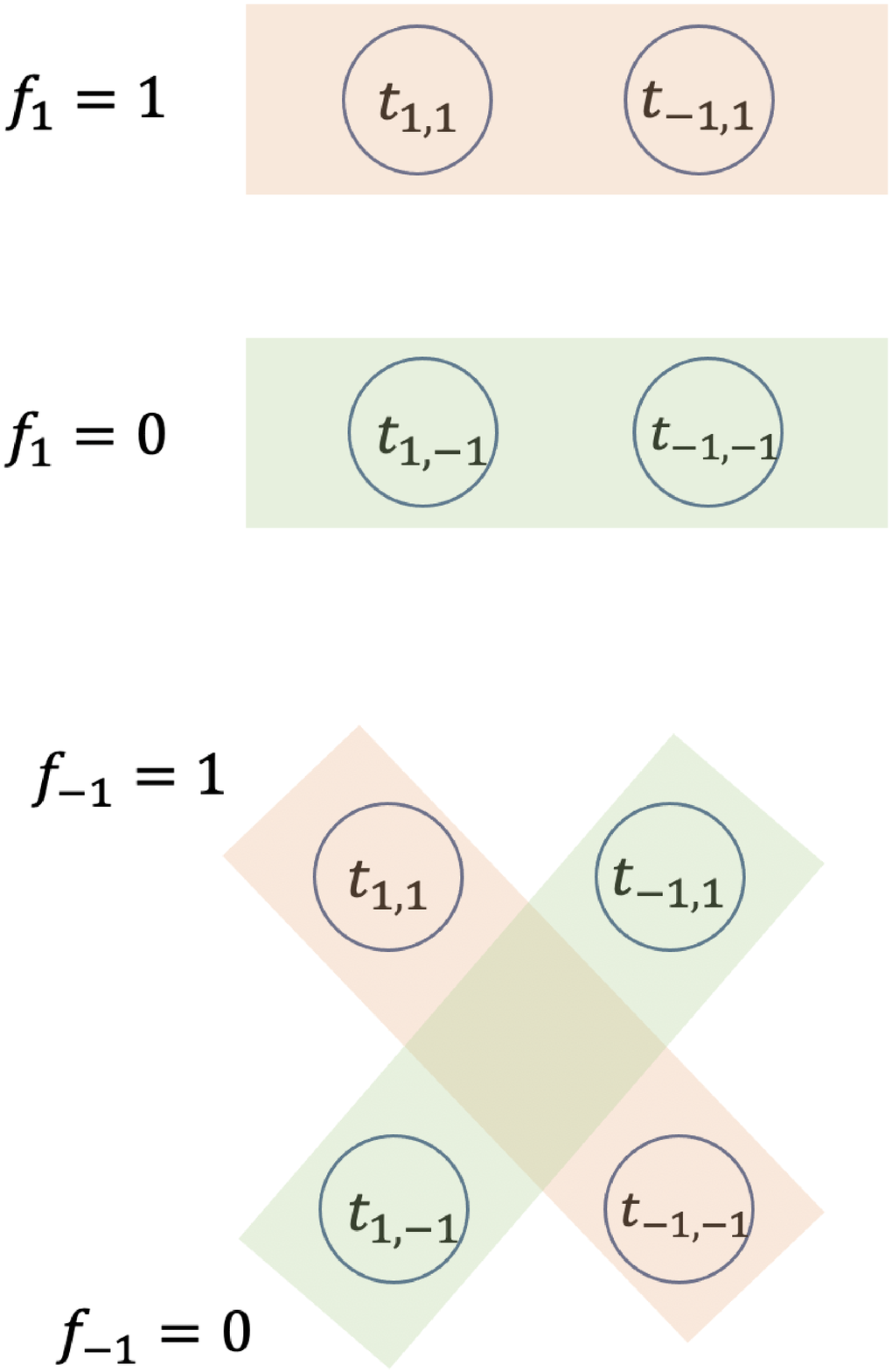}
         \caption{Illustration of $f_1$ and $f_{-1}$. They only differ on $t_{-1,1}$ and $t_{-1,-1}$. For $h=3$, the Bellman error $\norm{f_{c'}-\cT^3_{c,\vsigma} f_{c'}}_{2,\mu_3}^2=0.5$, regardless of $c'$ and $c$.}
         \label{fig:f}
     \end{subfigure}
     \vspace{0.2cm}
     \begin{subfigure}[b]{0.72\textwidth}
         \centering
         \includegraphics[width=\textwidth]{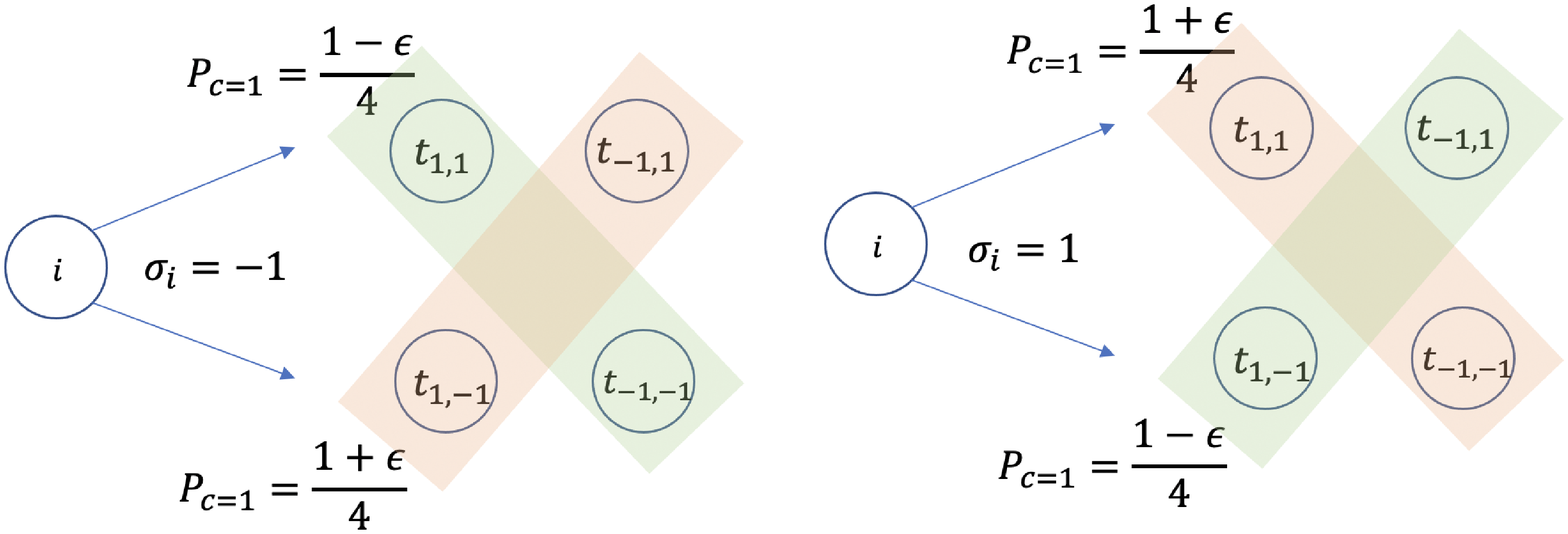}
         \caption{Illustration of $\overline{\P}^\eps_{1,\sigma_i}$. When $c=1$, there are two different but equally likely types of state $i$, depending on their probability transition matrix for the next step.}
         \label{fig:P_1}
     \end{subfigure}
     \vspace{0.2cm}
     \begin{subfigure}[b]{0.72\textwidth}
         \centering
         \includegraphics[width=\textwidth]{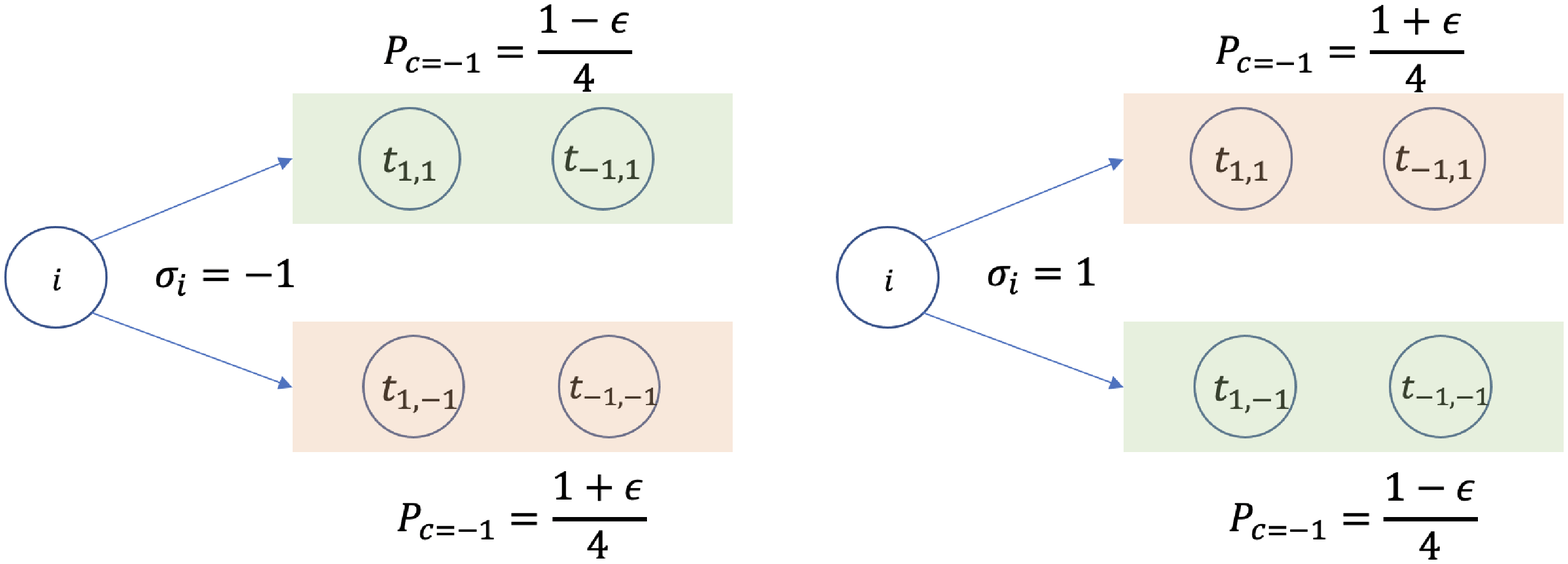}
         \caption{Illustration of $\overline{\P}^\eps_{-1,\sigma_i}$. When $c=1$, there are two different but equally likely types of state $i$, depending on their probability transition matrix for the next step.}
         \label{fig:P_-1}
     \end{subfigure}
     \caption{Graphical illustration of the hard instances $\cM^\eps_{c,\vsigma}$. As shown in \Cref{eq:c_sigma_bellman_error}, the total Bellman error is only determined by the Bellman error for $h=2$, which is equal to optimal error + $\frac{\eps^2}{12} \one{c\neq c'}  $ if $f_{c'}$ is the returned function. The main idea of the proof is to show it's difficult to guess $c$ via the observed dataset $D$ if $D$ only contains single-sampled data. As a sanity check, for any $c$ and sample $(i,t_{j,k})$, if $\sigma_i\overset{unif}{\sim}\{\pm 1\}$, the marginal distribution of $t_{j,k}$ is always uniform, but for double sampling of form $(i,t_{j,k},t_{j',k'})$, we can decide $c$ by simply looking at histogram of $(t_{j,k},t_{j',k'})$.}
\end{figure}

\paragraph{Function class:} $\cF = \{f_1,f_{-1}\}$, where  $f_c(\startstate) = \frac{1}{2}$, $f_c(i)=\frac{1}{2},\forall 1\le  i \le S$ and $f_{c}(t_{j,k})= \frac{k\max(c,j )+1}{2}$, 
 $\forall c,j,k\in\{\pm 1\}$. Compared to the notation in the main paper, we drop the dependency on $h$ for $f\in \cF$. This is because the MDP will reach a disjoint set of states for each step $h$ (see below).
  
\paragraph{Probability Transition Matrix:} 
We define the probability transition matrix below.  Specifically, for $i\in \{1,\ldots,S\}$ and $j,k\in\{\pm 1\}$, \(\P^\eps_{c,\vsigma}(t_{j,k}\mid i) \equiv \overline{\P}^\eps_{c,\sigma_i}(t_{j,k}):=0.25(1 + \eps k \max(-c,j)\sigma_i ).\)

\begin{table}[!htbp]
\centering
\begin{tabular}{|l|c|c|c|c|}
\hline
 \diagbox{From}{To} & $\startstate$ &  $i (i=1,\ldots,S)$ & $t_{j,k}$ & $s_{\text{end}}$\\
  \hline
  $\startstate$&  0 &  $\frac{1}{S}$ & 0 & 0 \\
  \hline
  $i (i=1,\ldots,S)$ & 0  & 0  & $ \overline{\P}^\eps_{c,\sigma_i}(t_{j,k}):=0.25(1 + \eps k \max(-c,j)\sigma_i )$ & 0\\
  \hline 
  $t_{j,k}$ & $0$ & $0$ & $0$ & 1\\
  \hline
\end{tabular}
\caption{Probability Transition Matrix $\P^\eps_{c,\vsigma}$ for MDP $\cM^\eps_{c,\vsigma}$. Starting from $\startstate$, the process terminates as it reaches $s_{\text{end}}$ in the $4$th step.}
\end{table}

\paragraph{Reward Function:} $r(\startstate) =0, r(i) =0,\forall 1\le i\le S$, $r(t_{j,k})=\frac{j+1}{2}, \forall j,k\in \{-1,1\}$. 
 
\paragraph{Underlying distribution:} We define the underlying distribution for batch data $\mu$ as $\mu_{2}(i) = \frac{1}{S}$ and $\mu_{3}(t_{j,k}) = \frac{1}{4}$, we can check that \Cref{as:coverage} is satisfied with $C=2$ as $\eps<0.5$. Define $\cT^1_{c,\vsigma}, \cT^2_{c,\vsigma},\cT^3_{c,\vsigma}$ be the Bellman operator of $\cM^\eps_{c,\sigma}$, we have $\forall \vsigma\in \{-1,1\}^S$, $\forall c,c'\in \{-1,1\}$, 
$$\norm{f_{c'}-\cT^3_{c,\vsigma} f_{c'}}_{2,\mu_3}^2 = \norm{f_{c'}-r}_{2,\mu_3}^2 =  \pr{j\neq k\max(c',j)} =0.5,$$
\[ \begin{aligned} \norm{f_{c'}-\cT^2_{c,\vsigma} f_{c'}}_{2,\mu_2}^2 = & \Bigg\|\sum_{j,k\in\{\pm 1\}} \P_{c,\vsigma}(t_{j,k}\mid i) f_{c'}(t_{j,k})\Bigg\|_{2,\mu_2}^2 \\ = & \frac{1}{64} \bigg\| \sum_{j,k}\eps\sigma_i k^2\max(j,c)\max(j,-c') \bigg\|_{2,\mu_2}^2 = \frac{\eps^2}{4} \one{c\neq c'}, \end{aligned} \]  
$$\norm{f_{c'}-\cT^1_{c,\vsigma} f_{c'}}_{2,\mu_3}^2 = \norm{f_c(\startstate) - f_c(i)}_{2,\mu_1}^2 = 0.$$

Thus
\begin{equation}\label{eq:c_sigma_bellman_error}
\Berr_{c,\vsigma}(f_{c'})\equiv \Berr_{\cM_{c,\vsigma}}(f_{c'})  = \frac{1}{3}\sum_{h=1}^3\norm{f_{c'}-\cT^h_{c,\vsigma} f_{c'}}_{2,\mu_h}^2 =\frac{1}{3}(0.5 + \frac{\eps^2}{4} \one{c\neq c'}).	
\end{equation}

From \cref{eq:c_sigma_bellman_error} we can see minimizing Bellman error in this case is equivalent to predict $-c$. And any algorithm predicts $c$ wrongly, i.e., outputs $f_{c'}$ with $c'\neq c$ with constant probability, will suffer $\Omega(\eps^2)$ expected optimality gap. More specifically, we can show that for random $\vsigma$, it's information-theoretically hard to predict $c$ correctly given $D$, which leads to the following theorem.

\begin{theorem}\label{thm:lower_bound_main}
For $c\overset{iid}{\sim} \{-1,1\}$, $\vsigma\overset{iid}{\sim} \{-1,1\}^S$, $D=\cup_{h=1}^3\{(s_i,a_i,r_i,s'_i,h)\}_{i=1}^n$ sampled from $\cM^\eps_{c,\vsigma}$ and $\mu$, we have for any learning algorithm $\mathfrak{A}$ with $n=O(\frac{\sqrt{S}}{\eps^2})$ samples,
\[ \E_{c,\vsigma}\E_{D }\left[\Berr_{c,\vsigma}\left(\mathfrak{A}(D)\right)\right] \ge \E_{c,\vsigma}\left[\min_{c'\in\{-1,1\}}\Berr_{c,\vsigma}\left(f_{c'}\right)\right] + \Omega(\eps^2). \]
Or equivalently (and more specifically), if we view $\widetilde{\mathfrak{A}}(D)$ as the modified version of $\mathfrak{A}$, whose range is $\{-1,1\}$ and satisfies $\mathfrak{A} = f_c$ with $c = \widetilde{\mathfrak{A}}(D)$. Then we have 
\[ \E_{c,\vsigma}\E_{D }\left[\one{\widetilde{\mathfrak{A}}(D) \neq c}\right] \ge  \Omega(\eps^2). \]
\end{theorem}

Towards proving \Cref{thm:lower_bound_main}, we need the following lower bound, where $\mu_2\circ \P^\eps_{c,\vsigma}$ is defined as the joint distribution of $(s,s')$, where $s\sim \mu_2$ and $s' \sim \P_{c,\vsigma| s}$. Note that when $\eps = 0$, $\P^0_{c,\vsigma}(\cdot\mid i)$ becomes uniform distribution for every $1\le i\le S$, and thus is independent of $c,\vsigma$, which could be denoted by $\P^0$ therefore.

\begin{lemma}\label{lem:TV_to_KL}
 If $n\le 0.1\frac{S^{0.5}}{\eps^{2}}$, then $\norm{\E_{\vsigma} \left(\mu_2\circ \P^\eps_{c,\vsigma}\right)^n -\left(\mu_2\circ \P^0\right)^n}_{TV}\le 0.1$, for all $c\in \{-1,1\}$.	
\end{lemma}
\begin{proof}
For convenience, we denote $\left(\mu_2\circ \P^0\right)^n$ by $P$ and $\E_{\vsigma} \left(\mu_2\circ \P^\eps_{c,\vsigma}\right)^n$ by $Q$. By Pinsker's inequality, we have $\norm{P-Q}_{TV}\le \sqrt{2KL(P,Q)}$, for any distribution $P,Q$. Thus it suffices to upper bound $KL(P,Q)$ by $0.05$.

We define $E_i$ as a random subset, i.e., $E_i = \{l|1\le l\le n, s_l = i\}$, given $D = \{(s_i,s_i')\}_{i=1}^n$.
 Then for both $\E_{\vsigma} \left(\mu_2\circ \P^\eps_{c,\vsigma}\right)^n$ and $\left(\mu_2\circ \P^0\right)^n$, $s_1,\ldots,s_n$ are i.i.d. distributed by $\mu_2$.  Note that 
\begin{equation}
\begin{split}
	&Q(s'_1,\ldots,s'_n\mid s_1,\ldots,s_n ) \\
	=& \sum_{\vsigma\in\{-1,1\}^S}p(\vsigma)Q(s'_1,\ldots,s'_n\mid s_1,\ldots,s_n,\vsigma) \\
	= &\sum_{\vsigma\in\{-1,1\}^S}\prod_{i=1}^S p(\sigma_i)\prod_{i=1}^S Q(s'_{E_i}|E_i,\sigma_i)\\
	= &\prod_{i=1}^S\left(\sum_{\sigma_i\in\{-1,1\}} p(\sigma_i)Q(s'_{E_i}|E_i,\sigma_i)\right),
\end{split}
\end{equation}
and 
\begin{equation}
\begin{split}
	P(s'_1,\ldots,s'_n\mid s_1,\ldots,s_n ) 	= \prod_{i=1}^S P(s'_{E_i}|E_i).\\
\end{split}
\end{equation}

For any tuple $(s_1,\ldots, s_n)$ and  subset $E\subset \{1,\ldots, n\}$, we define $s_{E}$ as the sub-tuple of $s$ with length $|E|$ selected by $E$. Define $P_{E_i},Q_{E_i}$ as the  distribution of $s'_{E_i}$ conditioned on $E_i$. In detail, $Q_{E_i}(s'_{E_i}) = \sum\limits_{\sigma_i\in\{-1,1\}} p(\sigma_i)Q(s'_{E_i}|E_i,\sigma_i)$ and $P_{E_i}(s'_{E_i}) = P(s'_{E_i}|E_i)$.  Note that for $Q$, $s'_{E_i}$ are i.i.d.  conditioned on $E_i$ and $\sigma_i$, i,.e., $ Q(s'_{E_i}|E_i,\sigma_i) =\prod_{l\in E_i}\overline{\P}^\eps_{c,\sigma_i}(s'_l)$. Therefore the distribution $Q_{E_i}$ only depends on $|E_i|$, so does $P_{E_i}$. 

Thus we can write the KL divergence as:

\begin{equation}
\begin{split}
	&\KL{\left(\mu_2\circ \P^0\right)^n}{\E_{\vsigma} \left(\mu_2\circ \P^\eps_{c,\vsigma}\right)^n} = \KL{P}{Q} = \Exp{D\sim P}{\log\frac{ P(D)}{ Q(D)}} \\
	=&  \Exp{D\sim P}{\log\frac{ P(s'_1,\ldots,s'_n\mid s_1,\ldots,s_n)}{  Q(s'_1,\ldots,s'_n\mid s_1,\ldots,s_n)} + \log\frac{P(s_1,\ldots,s_n)}{Q(s_1,\ldots,s_n)}}\\
	=&  \Exp{D\sim P}{\log\frac{  \prod_{i=1}^S P(s'_{E_i}|E_i)}{  \prod_{i=1}^S Q(s'_{E_i}|E_i)} + \log\frac{P(s_1,\ldots,s_n)}{Q(s_1,\ldots,s_n)}}  \quad  (P(s_1,\ldots,s_n) = Q(s_1,\ldots,s_n))\\
	=&  \Exp{D\sim P}{\sum_{i=1}^S \log\frac{P_{E_i}(s'_{E_i})}{  Q_{E_i}(s'_{E_i})} } \\
	=& \sum_{i=1}^S\Exp{D\sim P}{ \log\frac{P_{E_i}(s'_{E_i})}{  Q_{E_i}(s'_{E_i})} }.\\
\end{split}
\end{equation}

By the definition of $P$ and $Q$, given $c\in{\pm 1}$ and $\eps>0$, we can see that $\Exp{D\sim P}{ \log\frac{P_{E_i}(s'_{E_i})}{  Q_{E_i}(s'_{E_i})} }$ only a function of $|E_i|$, and we denote it by $G_{c,\eps}(|E_i|)$. Thus we have

\begin{equation}
\begin{split}
&\KL{\left(\mu_2\circ \P^0\right)^n}{\E_{\vsigma} \left(\mu_2\circ \P^\eps_{c,\vsigma}\right)^n} \\
=& \KL{P}{Q}\\
=&  \sum_{i=1}^S\sum_{m=0}^n \Exp{D\sim P}{\sum_{i=1}^S \log\frac{P_{E_i}(s'_{E_i})}{  Q_{E_i}(s'_{E_i})} \middle\vert |E_i| = m} P(|E_i| = m)\\
=&  \sum_{i=1}^S\sum_{m=0}^n G_{c,\eps}(m) P(|E_i| = m)\\
=&  S\sum_{m=0}^n G_{c,\eps}(m) P(|E_1| = m).\\
\end{split}
\end{equation}

The last step is because $E_i$ are i.i.d. distributed. For convenience, we will denote $P(|E_1| = m)$ by $P_N(m)$.

 It can be shown that $G_{c,\eps}(m)$ is independent of $c$, and thus we drop $c$ in the subscription. We could even simplify the expression of $G_{\eps}(m)$ by defining $R_{\sigma,\eps}(j) = 0.5(1+j\sigma\eps)$ over $\{-1,1\}$ (For $c=1$, this is effectively grouping $(t_{1,1},t_{-1,-1})$ into a state, say $1$, and  $(t_{1,-1},t_{-1,1})$ into another state, say $-1$.)

\[G_{\eps}(m) = \KL{\left(\textrm{Unif}\{-1,1\}\right)^m}{\frac{\left(R'_{-1,\eps} \right)^m+ \left(R'_{1,\eps}\right)^m}{2}}.\]

Below are some basic properties of $G_\eps(m)$.

\begin{itemize}
	\item $G_\eps(0) =0 $.
	\item $G_\eps(1) = 0$.
	\item $G_\eps(m) \le \frac{6m^2+m}{8}\eps^4 + \frac{m}{2}\frac{\eps^4}{1-\eps^2}\le 2m\eps^4$, for $\eps^2\le \frac{1}{2}$.
\end{itemize}

The first two properties can be verified by direct calculation, and the third property is proved in \Cref{lem:third_bullet}.

Now it remains to calculate $P_N(1)$ and $\Exp{P}{|E_1|^2}$. 
We have 

\[P_N(1) = n\frac{1}{S}(1-\frac{1}{S})^{n-1}\le \frac{n}{S}(1-\frac{n}{S}), \]
and 

\[\Exp{P}{|E_1|^2} = \Exp{P}{(\sum_{i=1^n}\one{s_i =1})^2} = \Exp{P}{\sum_{i=1}^{n}\one{s_i =1} + \sum_{i,j=1,\i\neq j}^{n}\one{s_i=s_j =1}} = \frac{n}{S} + \frac{n(n-1)}{S^2}.\]

Thus we conclude that 
\begin{align*}
	\KL{P}{Q} 
	= & S\sum_{m=2}^n P_N(m)  G_\eps(m) 
	\le S\sum_{m=2}^n P_N(m)\cdot 2m^2\eps^4
	= 2\big(\sum_{m=2}^n P_N(m)m^2\big) S\eps^4 \\
	= & 2\left(\Exp{P}{|E_1|^2} - P_N(1)\right)  S\eps^4
	= 2\left(\frac{n(n-1)}{S^2} +\frac{n^2}{S^2} \right) S\eps^4
	\le \frac{4n^2\eps^4}{S}.
\end{align*}

Since $n\le 0.1\frac{S^{0.5}}{\eps^{2}}$, we have $\norm{P-Q}_{TV}\le \sqrt{2KL(P,Q)} \le \sqrt{0.08} \le 0.1$, which completes the proof.
\end{proof}

\begin{lemma}\label{lem:third_bullet}
For $\eps^2\le \frac{1}{2}$, we have $$G_\eps(m) \le \frac{6m^2+m}{8}\eps^4 + \frac{m}{2}\frac{\eps^4}{1-\eps^2}\le 2m^2\eps^4.$$

\end{lemma}

\begin{proof}[Proof of \Cref{lem:third_bullet}]\label{proof:}
\newcommand{\sx}{|\vx|}
Let $x_1,\ldots, x_n\overset{i.i.d.}{\sim}\{-1,1\}$, we have 

\[G_\eps(m) = -\Exp{\vx}{\log \left(\prod_{i=1}^m (1-x_i\eps) + \prod_{i=1}^m (1+x_i\eps) \right)}.\]

For convenience, we define $\sx:= |\sum_{i=1}^m x_i|$. Note that 
\[\prod_{i=1}^m (1-x_i\eps) + \prod_{i=1}^m (1+x_i\eps) = \left((1-\eps)^{\sx} + (1+\eps)^{\sx}\right) (1-\eps^2)^{\frac{m-\sx}{2}}.\]

Thus 
\[ G_\eps(m) = -\Exp{\vx}{\log  \left((1-\eps)^{\sx} + (1+\eps)^{\sx}\right)} - \frac{m-\sx}{2}\Exp{\vx}{\log (1-\eps^2)}. \]

For the first term, we have 
\begin{align*}
&-\Exp{\vx}{\log  \left((1-\eps)^{\sx} + (1+\eps)^{\sx}\right)} \\
\le &-\Exp{\vx}{\log  \left(1 + \frac{\sx(\sx-1)}{2}\eps^2\right)}	\\
 \le &\Exp{\vx}{-\frac{\sx(\sx-1)}{2}\eps^2 + \frac{1}{2} \left(\frac{\sx(\sx-1)}{2}\eps^2\right)^2}\\
 \le &\Exp{\vx}{-\frac{\sx(\sx-1)}{2}\eps^2 + \frac{\sx^4}{8}\eps^4}\\
  = & -\frac{m}{2} \eps^2 + \Exp{\vx}{\frac{\sx}{2}}\eps^2 + \frac{6m^2+m}{8}\eps^4.\\
\end{align*}

For the second term, we have
\begin{align*}
	-\Exp{\vx}{\log (1-\eps^2)} = \Exp{\vx}{\log (1+ \frac{\eps^2}{1-\eps^2})}\le \frac{\eps^2}{1-\eps^2} = \eps^2 + \frac{\eps^4}{1-\eps^2}.
\end{align*}

Thus $G_\eps(m)$ only contains $\eps^4$ terms, i.e.,

\[ G_\eps(m) \le \frac{6m^2+m}{8}\eps^4 + \frac{m}{2}\frac{\eps^4}{1-\eps^2}\le 2m^2\eps^4, \]

the last step is by assumption $\eps^2\le \frac{1}{2}$.
\end{proof}

\begin{proof}[Proof of \Cref{thm:lower_bound_main}]

In our case, since $r$ is known and $|\cA|=1$, we can simplify the each data in $D$ into the form of $(s,s',h)$. Further since the probability transition matrix for $h=1$ and $h=3$ are known, below we will assume $D$ only contains $n$ pairs of $(s,s',2)$, and we will call these states by $\{s_i\}_{i=1}^ n $ and $\{s'_i\}_{i=1}^n$. Since $\norm{f_{c'}-\cT^3_{c,\vsigma} f_{c'}}_{2,\mu_2}^1$ and $\norm{f_{c'}-\cT^3_{c,\vsigma} f_{c'}}_{2,\mu_2}^1$ are constant for all $c,c'$, we only need to consider $\norm{f_{c'}-\cT^2_{c,\vsigma} f_{c'}}_{2,\mu_3}^2$ as our loss.

Recall we define $\mu_2\circ \P^\eps_{c,\vsigma}$ as the joint distribution of $(s,s')$, where $s\sim \mu_2$ and $s' \sim \P_{c,\vsigma| s}$. Thus the dataset $D$ can be viewed as sampled from $\E_{c,\vsigma} \left(\mu_2\circ \P_{c,\vsigma}\right)^n$, i.e., $D$ is sampled from a mixture of product measures.

By \Cref{lem:TV_to_KL}, we know 
\begin{align*}
&\big\| \E_{\vsigma} \left(\mu_2\circ \P^\eps_{1,\vsigma}\right)^n -\E_{\vsigma} \left(\mu_2\circ \P^\eps_{-1,\vsigma}\right)^n \big\|_{TV} 	\\
\le & \big\| \E_{\vsigma} \left(\mu_2\circ \P^\eps_{1,\vsigma}\right)^n -\E_{\vsigma} \left(\mu_2\circ \P^0\right)^n \big\|_{TV} + \big\| \E_{\vsigma} \left(\mu_2\circ \P^\eps_{-1,\vsigma}\right)^n -\E_{\vsigma} \left(\mu_2\circ \P^0\right)^n \big\|_{TV} 	\\
\le &0.2.
\end{align*}

Thus if we denote the distribution of $widetilde{\mathfrak{A}}(D)$ by $X_c$, where $D\sim \E_{\vsigma} \left(\mu_2\circ \P^\eps_{c,\vsigma}\right)^n$ and $\vsigma \sim \{-1,1\}^S$, and $\widetilde{\mathfrak{A}}$  can be random, the above inequality implies $\pr{X_{-1}\neq X_1}\le 0.2$, and therefore we have

\begin{equation}\label{eq:c_to_gap'}
\begin{split}
	\E_{c,\vsigma}\E_{D}\left[\Berr_{c,\vsigma}\left(\mathfrak{A}(D)\right)\right]  
	=&\frac{1}{2}\left( \E_{\vsigma,D}\left[\Berr_{1,\vsigma}\left(\mathfrak{A}(D)\right)\right] + \E_{\vsigma,D}\left[\Berr_{-1,\vsigma}\left(\mathfrak{A}(D)\right)\right]\right)\\
	=&\frac{1}{6} + \frac{\eps^2}{24} \left( \pr{X_1 \neq 1 } + \pr{X_{-1}\neq -1 } \right)\\
	=&\frac{1}{6} + \frac{\eps^2}{24} \left( \pr{X_1 \neq 1 } + \pr{X_{-1}\neq -1 } + \pr{X_1 \neq X_{-1}} \right) -\frac{\eps^2}{24}\pr{X_1 \neq X_{-1}} \\
	\ge& \frac{1}{6}+ \frac{\eps^2}{24} - \frac{\eps^2}{24}\pr{X_1 \neq X_{-1}}\\
	\ge& \frac{1}{6}+ \frac{\eps^2}{24} - \frac{\eps^2}{24}\times 0.2\\
	= & \frac{1}{6}+ \frac{\eps^2}{30}\\
	= & \E_{c,\vsigma}\left[\min_{c'\in\{-1,1\}}\Berr_{c,\vsigma}\left(f_{c'}\right)\right]  + \frac{\eps^2}{30}.
\end{split}
\end{equation}

\end{proof}

\subsection{Proof of Theorem~\ref{thm:lower_bound_main2}}\label{subsec:proof_full_lb}

Now we will prove \Cref{thm:lower_bound_main2} by slightly twisting the distribution of hard instances (MDPs) constructed in the previous subsection.

\begin{proof}[Proof of \Cref{thm:lower_bound_main2}]
W.O.L.G, we can assume $S$ is even and $S = 2S'$ (o.w. we can just abandon one state.) The only modification from the previous lower bound with $C=2$ is now the distribution of $\sigma$ is defined as the conditional distribution of $P$ on $\sum_{i=1}^{S} = 0$, i.e.,  $P'(\vsigma) = P(\vsigma|\sum_{i=1}^{S} \sigma_i = 0)$, where $P$ is the uniform distribution on $\{-1,1\}^S$. The main idea is that the data distribution (i.e., distribution of $(s,s')$) shouldn't be very different even if we add this additional `balancedness' restriction.
 We further define a metric $d$ on $\{-1,1\}^S$. In detail, for $\vsigma,\vsigma'\in\{-1,1\}^d$, we define $d(\vsigma,\vsigma') = \frac{\sum_{i=1}^S |\sigma_i-\sigma'_i|}{2S}$. We have the following lemma:
\begin{lemma}
\begin{equation}\label{eq:wasserstein}
W_1^d(P,P') = 	\frac{1}{2S}\E_{P}{\left|\sum_{i=1}^S \sigma_i\right|},
\end{equation}
where $W_1^d(P,P')$ is defined as $\min\limits_{\vsigma\sim P,\vsigma'\sim P'}\E[d(\vsigma,\vsigma')]$.

By Cauchy Inequality, we have
	\begin{equation}
		W_1^d(P,P') = \frac{1}{2S}\E_{P}{\left|\sum_{i=1}^S \sigma_i\right|}	\le \frac{1}{2S}\sqrt{\E_{P}{\left(\sum_{i=1}^S \sigma_i\right)^2}} = \frac{1}{2\sqrt{S}}
	\end{equation}
\end{lemma}

\begin{proof}
For even $S$, we define $B$ as the set of the ``balanced'' $\vsigma$, i.e., $B = \{\vsigma | \sum_{i=1}^S \sigma_i' = 0\}$. 
For every $\vsigma\in \{-1,1\}^S$, we define $Q_\vsigma$ as the uniform distribution on $U_\vsigma = \{\vsigma' \mid d(\vsigma,\vsigma') = \frac{|\sum_{i=1}^S\sigma_i|}{2S}\}\cap B$, i.e. $\vsigma'\in U_\vsigma$ if and only if $\vsigma'\in B $ and $d(\vsigma,\vsigma') = \min_{\vsigma'\in B}d(\vsigma,\vsigma')$.

Now we define $\Gamma(\vsigma,\vsigma') = P(\vsigma)Q_\vsigma(\vsigma')$. By definition the marginal distribution of $\Gamma$ on $\vsigma$ is $P$. By symmetry, the marginal distribution of $\vsigma'$ is $P'$. Thus by definition of $W_1$,

\[W_1^d(P,P') \le \Exp{\vsigma,\vsigma'\sim \Gamma}{d(\vsigma,\vsigma')} = \frac{1}{2S}\E_{P}{\left|\sum_{i=1}^S \sigma_i\right|}. \]
\end{proof}

\begin{lemma}\label{lem:TV_diff_sigma}
 \begin{equation}
\big\| \left(\mu_2\circ \P^\eps_{c,\vsigma}\right)^n -\left(\mu_2\circ \P^\eps_{c,\vsigma'}\right)^n \big\|_{TV} 	\le C\eps\sqrt{n d(\vsigma,\vsigma')}.
 \end{equation}
\end{lemma}

\begin{proof}
First, note that
\begin{equation}
	\begin{split}
		&\KL{\mu_2\circ \P^\eps_{c,\vsigma}}{\mu_2\circ \P^\eps_{c,\vsigma'}} \\
		= & \KL{\mu_2}{\mu_2} + \Exp{i\sim \mu_2}{\KL{\P^\eps_{c,\vsigma}(\cdot\mid i)}{\P^\eps_{c,\vsigma'}(\cdot \mid i)}}\\
		= & 0+ \Exp{i\sim \mu_2}{\KL{\P^\eps_{c,\sigma_i}}{\P^\eps_{c,\sigma'_i}}}\\
		= & \prob{i\sim \mu_2}{\sigma_i\neq \sigma'_i} \cdot \left( \frac{1+\eps}{2}\log \frac{1+\eps}{1-\eps}+ \frac{1-\eps}{2}\log \frac{1-\eps}{1+\eps}\right)\\
		= & \prob{i\sim \mu_2}{\sigma_i\neq \sigma'_i} \cdot \eps\log \frac{1+\eps}{1-\eps}\\
		= & \prob{i\sim \mu_2}{\sigma_i\neq \sigma'_i} \cdot \frac{2\eps^2}{1-\eps}\\
		\le & 4d(\vsigma,\vsigma') \eps^2.
	\end{split}
\end{equation}
Thus we have 
\begin{equation}
\begin{split}
	&\big\| \left(\mu_2\circ \P^\eps_{c,\vsigma}\right)^n -\left(\mu_2\circ \P^\eps_{c,\vsigma'}\right)^n\big\|_{TV} \\
\le &\sqrt{2 \KL{\left(\mu_2\circ \P^\eps_{c,\vsigma}\right)^n}{\left(\mu_2\circ \P^\eps_{c,\vsigma'}\right)^n}}\\
\le &\sqrt{2 n \KL{\mu_2\circ \P^\eps_{c,\vsigma}}{\mu_2\circ \P^\eps_{c,\vsigma'}}}\\
\le & \eps\sqrt{8m d(\vsigma,\vsigma')}.
\end{split}
\end{equation}

\end{proof}

Let $\Gamma(\vsigma,\vsigma')$ be the joint probabilistic distribution on $\{-1,1\}^S \times \{-1,1\}^S$ which attains the \cref{eq:wasserstein}. Therefore the marginal distribution of $\Gamma$ is $P$ and $P'$.  And thus we have for any $c\in \{-1,1\}$, 

\begin{equation}
\begin{split}
	&\norm{\Exp{\vsigma\sim P}{\left(\mu_2\circ \P^\eps_{c,\vsigma}\right)^n} -\Exp{\vsigma\sim P'}{\left(\mu_2\circ \P^\eps_{c,\vsigma'}\right)^n}}_{TV} \\
	\le & \Exp{\vsigma,\vsigma'\sim \Gamma}{\norm{ \left(\mu_2\circ \P^\eps_{c,\vsigma}\right)^n -\left(\mu_2\circ \P^\eps_{c,\vsigma'}\right)^n}_{TV} }\\
	\le &\Exp{\vsigma,\vsigma'\sim \Gamma}{\eps\sqrt{8n d(\vsigma,\vsigma')}}\\
	\le & \eps\sqrt{n\Exp{\vsigma,\vsigma'\sim \Gamma}{8d(\vsigma,\vsigma')}}\\
	 = & \eps \sqrt{8n W_1^d(P,P')}\\
	\le & 2\eps n^{0.5} S^{-0.25}.
\end{split}
\end{equation}

Therefore, when $n\le  \frac{\sqrt{S}}{400\eps^2}$, for any $c\in \{-1,1\}$,
\[\norm{\Exp{\vsigma\sim P}{\left(\mu_2\circ \P^\eps_{c,\vsigma}\right)^n} -\Exp{\vsigma\sim P'}{\left(\mu_2\circ \P^\eps_{c,\vsigma'}\right)^n}}_{TV}\le 0.1. \]

By \Cref{lem:TV_to_KL}, we have \[\norm{\Exp{\vsigma\sim P'}{\left(\mu_2\circ \P^\eps_{1,\vsigma'}\right)^n}-\Exp{\vsigma\sim P'}{\left(\mu_2\circ \P^\eps_{-1,\vsigma'}\right)^n}}_{TV} \le 0.1+0.1+0.1+0.1 = 0.4.\] 

Thus using the same argument in \cref{eq:c_to_gap'}, 
In detail, denote the distribution of $\mathfrak{A}(D)$ by $X_c$, where $D\sim \E_{\vsigma} \left(\mu_2\circ \P^\eps_{c,\vsigma}\right)^n$, $\vsigma \sim \{-1,1\}^S$,  the above inequality implies $\pr{X_{-1}\neq X_1}\le 0.4$, and therefore we have 
\begin{equation}\label{eq:c_to_gap}
\begin{split}
	\E_{c,\vsigma}\E_{D}\left[\Berr_{c,\vsigma}\left(\mathfrak{A}(D)\right)\right]  
	=&\frac{1}{2}\left( \E_{\vsigma,D}\left[\Berr_{1,\vsigma}\left(\mathfrak{A}(D)\right)\right] + \E_{\vsigma,D}\left[\Berr_{-1,\vsigma}\left(\mathfrak{A}(D)\right)\right]\right)\\
	=&\frac{1}{6} + \frac{\eps^2}{24} \left( \pr{X_1 \neq 1 } + \pr{X_{-1}\neq -1 } \right)\\
	=&\frac{1}{6}  + \frac{\eps^2}{24} \left( \pr{X_1 \neq 1 } + \pr{X_{-1}\neq -1 } + \pr{X_1 \neq X_{-1}} \right) -\frac{\eps^2}{24}\pr{X_1 \neq X_{-1}} \\
	\ge& \frac{1}{6} + \frac{\eps^2}{24} - \frac{\eps^2}{24} \times 0.4\\
	= & \frac{1}{6} + \frac{1}{40} \eps^2\\
	= & \E_{c,\vsigma}\left[\min_{c'\in\{-1,1\}}\Berr_{c,\vsigma}\left(f_{c'}\right)\right]  + \frac{\eps^2}{40}.
\end{split}
\end{equation}

\end{proof}
% !TEX root = main.tex

\section{Auxiliary Results} \label{append:aux}

In this section, we prove some auxiliary lemmas. \Cref{app:proof_lemma_surrogate} considers the relation between Bellman error and suboptimality in values (\Cref{lem:surrogate}). \Cref{app:proof_lemma_Ctilde} provides a supporting lemma used in the proof of \Cref{theorem:Minimax_LRC}. \Cref{app:proof_lemma_VF} presents a full version of \Cref{lemma:VF}.

\subsection{Connections between Bellman error and suboptimality in value (Lemma~\ref{lem:surrogate})} \label{app:proof_lemma_surrogate}

In this part, we present several possible ways to connect Bellman error $\Berr(\bfun)$ with the suboptimality gap $V_1^{\star}(s_1) - V_1^{\pi_{\bfun}}(s_1)$.

\paragraph{Via concentrability coefficient}

\gap*

\Cref{lem:surrogate} gives a feasible method to upper bound $V_1^{\star}(s_1) - V_1^{\pi_{\bfun}}(s_1)$ with $\Berr(\bfun)$ using the concentrability coefficient introduced in \Cref{as:coverage}. We provide the proof of \Cref{lem:surrogate} below.
\begin{proof}[Proof of Lemma \ref{lem:surrogate}]
	The proof of \Cref{lem:surrogate} is analogous to Theorem~2 in \cite{xie2020q}. We place it here for the self-containedness of our paper. In discussions below, we omit the subscript $h$ in policy $\pi_{f_h}$ and simply write $\pi_{\bfun}$ to ease the notation. We first note that since $\pi_f$ is greedy w.r.t $f$, therefore,
	\begin{equation} \label{eq:pre_0} V_1^{\star}(s_1) - V_1^{\pi_f}(s_1) \leq V_1^{\star}(s_1) - f_1\big(s_1, \pi^\star(s_1)\big) + 
	f_1\big(s_1, \pi_f(s_1)\big) - V_1^{\pi_f}(s_1).
	\end{equation}
	Consider any policy $\pi$. Since $f_{H+1} = 0$ and $V_1^{\pi}(s_1) = \E\big[ \sum_{h=1}^H r_h \, \big| \, s_1, \pi \big]$ by definition, we have
	\begin{equation*}
	f_1\big( s_1, \pi(s_1) \big) - V_1^{\pi}(s_1) = \E \Bigg[ \sum_{h=1}^H \Big( f_h(s_h,a_h) - \E_h^{\pi} \big[ r_h + f_{h+1}\big(s_{h+1},a_{h+1}\big) \, \big| \, s_h, a_h \big] \Big) \, \Bigg| \, s_1, \pi \Bigg].
	\end{equation*}
	Therefore, combined with the fact $\pi_f$ is the greedy policy w.r.t. $f$, we can show that
	\begin{align} \label{eq:pre_1} f_1\big( s_1, \pi^\star(s_1) \big) - V_1^{\star}(s_1) \geq& \E \Bigg[ \sum_{h=1}^H \big( f_h - \cT_h^{\star} f_{h+1} \big)(s_h,a_h) \, \Bigg| \, s_1, \pi^\star \Bigg], \\
	 \label{eq:pre_2} f_1(s_1, \pi_f(s_1)) - V_1^{\pi_f}(s_1) =& \E \Bigg[ \sum_{h=1}^H \big( f_h - \cT_h^{\star} f_{h+1} \big)(s_h,a_h) \, \Bigg| \, s_1, \pi_f \Bigg]. \end{align}
	Plugging \cref{eq:pre_1,eq:pre_2} into \cref{eq:pre_0} yields
	\[ V_1^{\star}(s_1) - V_1^{\pi_f}(s_1) \leq -\E\Bigg[ \sum_{h=1}^H \big( f_h - \cT_h^{\star} f_{h+1} \big)(s_h,a_h) \, \Bigg| \, s_1, \pi^\star \Bigg] + \E \Bigg[ \sum_{h=1}^H \big( f_h - \cT_h^{\star} f_{h+1} \big)(s_h,a_h) \, \Bigg| \, s_1, \pi_f \Bigg].
	 \]
	Under \Cref{as:coverage}, by Cauchy-Swartz inequality, it holds that for any policy $\pi$:
	\begin{align*}
	\Bigg|\E \Bigg[ \sum_{h=1}^H \big( f_h - \cT_h^{\star} f_{h+1} \big)(s_h,a_h) \, \Bigg| \, s_1, \pi \Bigg] \Bigg|
	\le& \sqrt{H \sum_{h=1}^H \E \Bigg[\big( f_h - \cT_h^{\star} f_{h+1} \big)^2(s_h,a_h) \, \Bigg| \, s_1, \pi \Bigg]}
	\\
	\leq & \sqrt{C} H \sqrt{\frac{1}{H} \sum_{h=1}^H \big\| f_h - \cT_h^{\star} f_{h+1} \big\|_{\mu_h}^2},
	\end{align*}
	which finishes the proof.
\end{proof}

\paragraph{Via a weaker concentrability assumption}

We observe that \Cref{lem:surrogate} does not necessarily need an assumption as strong as \Cref{as:coverage}. In fact, the inequality $V_1^{\star}(s_1) - V_1^{\pi_{\bfun}}(s_1) \leq 2H \sqrt{C \cdot \Berr(\bfun)}$ still holds if
\begin{align} \label{eq:concen2} \E \big[ \big(f_h - \cT_h^{\star} f_{h+1}\big)(s_h,a_h) \bigm| s_1, \pi \big] \leq \sqrt{C} \big\| f_h - \cT_h^{\star} f_{h+1} \big\|_{\mu_h} \qquad \text{for $\pi = \pi^{\star}$ or $\pi = \pi_{\bfun}$ for $\bfun \in \cF$}. \end{align}

If the function class $\bFun$ and $\bcT^{\star} \bFun = \big\{ \bcT^{\star} \bfun = (\cT_1^{\star} f_2, \ldots, \cT_H^{\star} f_{H+1}) \bigm| \bfun \in \bFun \big\}$ have good structures, we may have a tighter estimate of the required $C$.
For illustrative purpose, we take a simple example where $\cF_h$ is a subset of a finite dimensional linear space and $\cT_h^{\star} f_{h+1} \in \cF_h$ for any $f_{h+1} \in \cF_{h+1}$. Let $\phi: \cS \times \cA \rightarrow \mathbb{R}^d$ be a basis of $\cF_h$ with $\| \phi(s,a) \|_2 \leq 1$. Define $\Sigma_h := \E_{\mu_h} [ \phi \phi^{\top} ] \in \mathbb{R}^{d \times d}$. For any $f = w^{\top} \phi \in \cF_h$, $\|f\|_{\infty} \leq \| w \|_2 \leq \|\Sigma_h^{\frac{1}{2}}w\|_2 \sqrt{1/\lambda_{\min} (\Sigma_h)} = \|f\|_{\mu_h} \sqrt{1/\lambda_{\min} (\Sigma_h)}$. Therefore, \cref{eq:concen2} holds for $C = \max_{h \in [H]} \{ 1/ \lambda_{\min}(\Sigma_h) \}$.

\subsection{Proof of Supporting Lemmas in Minimax Algorithm Analysis} \label{app:proof_lemma_Ctilde}

\begin{lemma} \label{lemma:VQ-Vf}
	Suppose \Cref{as:Ctilde} holds. Denote $\bfun^{\best} := \min_{\bfun \in \bFun} \Berr(\bfun)$. For $h \in [H]$, it holds that
	\begin{equation} \label{eq:Q-f_norm} \big\| f_h - f^{\best}_h \big\|_{\rho_h}^2 \leq \widetilde{C}H(H-h+1) \big\| \big( \bfun - \bcT^{\star} \bfun\big) - \big(\bfun^{\best} - \bcT^{\star} \bfun^{\best}\big) \big\|_{\bmu}^2 \quad \text{for $\rho_h = \mu_h$ or $\nu_h \times \Uniform(\cA)$}, \end{equation}
	\begin{equation} \label{eq:VQ-Vf_norm}
	\text{and} \qquad \big\| \cT_h^{\star}f_{h+1} - \cT_h^{\star}f^{\best}_{h+1} \big\|_{\mu_h}^2 \leq \widetilde{C}H(H-h) \big\| \big( \bfun - \bcT^{\star} \bfun\big) - \big(\bfun^{\best} - \bcT^{\star} \bfun^{\best}\big) \big\|_{\bmu}^2.
	\end{equation}
\end{lemma}

\begin{proof}
	1. Let $\pi_{\bfun}$ be the greedy policy associated with $\bfun \in \bFun$.
	Since $f_{H+1} = f^{\best}_{H+1} = 0$, we have
	\begin{equation} \label{eq:Vf-VQ_2} \begin{aligned} 
	f_{h}(s, a) - f^{\best}_{h}(s, a)
	= & \E \Bigg[ \sum_{\tau=h}^H \bigg[ \Big( f_{\tau}\big(s_{\tau},a_{\tau}\big) - \E \big[ r_{\tau} + f_{\tau+1}\big(s_{\tau+1},a_{\tau+1}\big)  \, \big| \, s_{\tau}, a_{\tau} \big] \Big) \\
	& \qquad - \Big( f^{\best}_{\tau}\big(s_{\tau},a_{\tau}\big) - \E \big[ r_{\tau} + f^{\best}_{\tau+1}\big(s_{\tau+1},a_{\tau+1}\big) \, \big| \, s_{\tau}, a_{\tau} \big] \Big) \bigg] \, \Bigg| \, s_{h} = s, a_{h} = a, \pi_{\bfun} \Bigg].
	\end{aligned} \end{equation}
	Note that
	\begin{equation} \label{eq:Vf-VQ_3} \begin{aligned}
	& \E \big[ r_{\tau} + f_{\tau+1}\big(s_{\tau+1}, \pi_{f_{\tau+1}}(s_{\tau+1})\big)  \, \big| \, s_{\tau}, a_{\tau} \big] = \cT_{\tau}^{\star} f_{\tau+1}(s_{\tau}, a_{\tau}), \\
	& \E \big[ r_{\tau} + f^{\best}_{\tau+1}\big(s_{\tau+1},\pi_{f_{\tau+1}}(s_{\tau+1})\big) \, \big| \, s_{\tau}, a_{\tau} \big] \leq \cT_{\tau}^{\star} f^{\best}_{\tau+1}(s_{\tau}, a_{\tau}).
	\end{aligned} \end{equation}
	Combining \cref{eq:Vf-VQ_2,eq:Vf-VQ_3}, we learn that
	\begin{equation} \label{eq:Vf-VQ} 
	f_{h}(s, a) - f^{\best}_{h}(s, a)
	\leq \E \Bigg[ \sum_{\tau = h}^H \Big[ \big( f_{\tau} - \cT_{\tau}^{\star}f_{\tau+1}\big) - \big(f^{\best}_{\tau} - \cT_{\tau}^{\star} f^{\best}_{\tau+1}\big) \Big](s_{\tau},a_{\tau}) \, \Bigg| \, s_{h} = s, a_{h} = a, \pi_{\bfun} \Bigg].
	\end{equation}
	By symmetry, it also holds that
	\begin{equation} \label{eq:VQ-Vf} 
	f^{\best}_{h}(s, a) - f_{h}(s, a)
	\leq \E \Bigg[ \sum_{\tau = h}^H \Big[ \big(f^{\best}_{\tau} - \cT_{\tau}^{\star} f^{\best}_{\tau+1}\big) - \big( f_{\tau} - \cT_{\tau}^{\star}f_{\tau+1}\big) \Big](s_{\tau},a_{\tau}) \, \Bigg| \, s_{h} = s, a_{h} = a, \pi_{\bfun^{\best}} \Bigg].
	\end{equation}
	Under \Cref{as:Ctilde}, by Cauchy-Swartz inequality, for any policy $\pi$: 
	\begin{align*} & \E_{(s_h,a_h) \sim \mu_h} \Bigg( \E \Bigg[ \sum_{\tau = h}^H \Big[ \big( f_{\tau} - \cT_{\tau}^{\star}f_{\tau+1}\big) - \big(f^{\best}_{\tau} - \cT_{\tau}^{\star} f^{\best}_{\tau+1}\big) \Big](s_{\tau},a_{\tau}) \, \Bigg| \, s_{h}, a_{h}, \pi \Bigg] \Bigg)^2 \\ \leq & (H-h+1) \E \Bigg[ \sum_{\tau = h}^H \Big[ \big( f_{\tau} - \cT_{\tau}^{\star}f_{\tau+1}\big) - \big(f^{\best}_{\tau} - \cT_{\tau}^{\star} f^{\best}_{\tau+1}\big) \Big]^2(s_{\tau},a_{\tau}) \, \Bigg| \, (s_{h}, a_{h}) \sim \mu_h, \pi \Bigg] \\ \leq & \widetilde{C} (H-h+1) \sum_{\tau = h}^H \big\| \big( f_{\tau} - \cT_{\tau}^{\star}f_{\tau+1}\big) - \big(f^{\best}_{\tau} - \cT_{\tau}^{\star} f^{\best}_{\tau+1}\big) \big\|_{\mu_{\tau}}^2 \\ \leq & \widetilde{C} H (H-h+1) \big\| \big( \bfun - \bcT^{\star}\bfun\big) - \big(\bfun^{\best} - \bcT^{\star} \bfun^{\best}\big) \big\|_{\mu}^2. \end{align*}
	Therefore, \cref{eq:Vf-VQ,eq:VQ-Vf} imply \cref{eq:Q-f_norm}.
	
	2. We now consider $\big\| \cT_h^{\star} f_{h+1} - \cT_h^{\star} f_{h+1}^{\best} \big\|_{\mu_h}$. Take $\widetilde{\pi}_{h+1}(s) := \argmax_{a \in \mathcal{A}} \big\{ f_{h+1}(s,a) \vee f_{h+1}^{\best}(s,a) \big\}$. Then we have
	\[ \big| V_{f_{h+1}}(s) - V_{f_{h+1}^{\best}}(s) \big| \leq \big| f_{h+1}\big(s, \widetilde{\pi}_{h+1}(s)\big) - f^{\best}_{h+1}\big(s, \widetilde{\pi}_{h+1}(s)\big) \big|. \]
	It follows that
	\[ \begin{aligned} \big\| \cT_h^{\star} f_{h+1} - \cT_h^{\star} f_{h+1}^{\best} \big\|_{\mu_h} = & \big\| \E\big[V_{f_{h+1}}(s') - V_{f_{h+1}^{\best}}(s') \, \big| \, s,a\big] \big\|_{\mu_h} \\ \leq & \big\| V_{f_{h+1}} - V_{f_{h+1}^{\best}} \big\|_{\nu_h} \leq \big\| f_{h+1} - f_{h+1}^{\best} \big\|_{\nu_h \times \widetilde{\pi}_{h+1}}. \end{aligned} \]
	Similar to \cref{eq:Vf-VQ,eq:VQ-Vf}, we find that
	\begin{align*} & \big\| f_{h+1} - f_{h+1}^{\best} \big\|_{\nu_h \times \widetilde{\pi}_{h+1}}^2 \\ \leq & \max_{\pi = \pi_{\bfun} \text{ or } \pi_{\bfun^{\best}}} \E_{(s_{h+1},a_{h+1}) \sim \nu_h \times \widetilde{\pi}_{h+1}} \Bigg( \E \Bigg[ \sum_{\tau = h+1}^H \Big[ \big( f_{\tau} - \cT_{\tau}^{\star}f_{\tau+1}\big) - \big(f^{\best}_{\tau} - \cT_{\tau}^{\star} f^{\best}_{\tau+1}\big) \Big](s_{\tau},a_{\tau}) \, \Bigg| \, s_{h+1}, a_{h+1}, \pi \Bigg] \Bigg)^2 \\ \leq & (H-h) \max_{\pi = \pi_{\bfun} \text{ or } \pi_{\bfun^{\best}}} \E_{(s_{h+1},a_{h+1}) \sim \nu_h \times \widetilde{\pi}_{h+1}} \E \Bigg[ \sum_{\tau = h+1}^H \Big[ \big( f_{\tau} - \cT_{\tau}^{\star}f_{\tau+1}\big) - \big(f^{\best}_{\tau} - \cT_{\tau}^{\star} f^{\best}_{\tau+1}\big) \Big]^2(s_{\tau},a_{\tau}) \, \Bigg| \, s_{h+1}, a_{h+1}, \pi \Bigg] \\ \leq & \widetilde{C} (H-h) \sum_{\tau = h+1}^H \big\| \big( f_{\tau} - \cT_{\tau}^{\star}f_{\tau+1}\big) - \big(f^{\best}_{\tau} - \cT_{\tau}^{\star} f^{\best}_{\tau+1}\big) \big\|_{\mu_{\tau}}^2 \leq \widetilde{C} H(H-h) \big\| \big( \bfun - \bcT^{\star}\bfun\big) - \big(\bfun^{\best} - \bcT^{\star} \bfun^{\best}\big) \big\|_{\mu}^2. \end{align*}
	Therefore, we conclude that $\big\| \cT_h^{\star} f_{h+1} - \cT_h^{\star} f_{h+1}^{\best} \big\|_{\mu_h}^2 \leq \widetilde{C} H(H-h) \big\| \big( \bfun - \bcT^{\star}\bfun\big) - \big(\bfun^{\best} - \bcT^{\star} \bfun^{\best}\big) \big\|_{\mu}^2$.
\end{proof}

\subsection{Proof of Proposition~\ref{lemma:VF}} \label{app:proof_lemma_VF}

\begin{lemma}[Full version of \Cref{lemma:VF}] \label{lemma:Minimax_RC_prelim}
	Let $\widetilde{\cF}_{h+1}$ be any subset of $\cF_{h+1}$. 
	We have the following inequality,
	\[ \begin{aligned} \popR_n^{\mu_h} \big(\big\{ \cT_h^{\star} f_{h+1} \, \big| \, f_{h+1} \in \widetilde{\cF}_{h+1} \big\}\big) \leq \popR_n^{\nu_h} \big( V_{\widetilde{\cF}_{h+1}} \big) \leq \sqrt{2} A \popR_n^{\nu_h \times \Uniform(\cA)} \big( \widetilde{\cF}_{h+1} \big). \end{aligned} \]
\end{lemma}

\begin{proof}
	1. Due to the symmetry of Rademacher random variables,
	\[ \begin{aligned} \popR_n^{\mu_h} \big( \big\{ \cT_h^{\star} f_{h+1} \, \big| \, f_{h+1} \in \widetilde{\cF}_{h+1} \big\} \big) = & \popR_n^{\mu_h} \Big( \Big\{ r_h + \E\big[ V_{f_{h+1}}(s_h') \, \big| \, s_h,a_h\big] \, \Big| \, f_{h+1} \in \widetilde{\cF}_{h+1} \Big\} \Big) \\ = & \popR_n^{\mu_h} \Big( \Big\{ \E\big[ V_{f_{h+1}}(s_h') \, \big| \, s_h,a_h\big] \, \Big| \, f_{h+1} \in \widetilde{\cF}_{h+1} \Big\} \Big). \end{aligned} \]
	By definition,
	\[ \popR_n^{\mu_h} \Big( \Big\{ \E\big[ V_{f_{h+1}}(s_h') \, \big| \, s_h,a_h\big] \, \Big| \, f_{h+1} \in \widetilde{\cF}_{h+1} \Big\} \Big) = \E_{\mu_h} \Bigg[ \sup_{f_{h+1} \in \widetilde{\cF}_{h+1}} \sum_{k=1}^n \sigma_k \E\big[ V_{f_{h+1}}(s_{k,h}') \, \big| \, s_{k,h},a_{k,h}\big] \Bigg]. \]
	Switching the order of supremum and the inner expectation, we derive that
	\[ \begin{aligned} \popR_n^{\mu_h} \Big( \Big\{ \E\big[ V_{f_{h+1}}(s_h') \, \big| \, s_h,a_h\big] \, \Big| \, f_{h+1} \in \widetilde{\cF}_{h+1} \Big\} \Big) \leq \E \Bigg[ \sup_{f_{h+1} \in \widetilde{\cF}_{h+1}} \sum_{k=1}^n \sigma_k V_{f_{h+1}}(s_{k,h}') \Bigg] = \popR_n^{\nu_h} \big( V_{\widetilde{\cF}_{h+1}} \big). \end{aligned} \]
	
	2. For notational convenience, let $\mathcal{A} = [A]$. Consider a vector function $\vec{\bfun}_{h+1}: \mathcal{S} \rightarrow \mathbb{R}^A$ defined as $\vec{\bfun}_{h+1}(s) := \big( f_{h+1}(s,1), f_{h+1}(s,2), \ldots, f_{h+1}(s,A) \big)\!^{\top} \in \mathbb{R}^A$. Then for any $f_{h+1}, f'_{h+1} \in \cF_{h+1}$, $\big| V_{f_{h+1}}(s) - V_{f'_{h+1}}(s) \big| \leq \| \vec{\bfun}_{h+1} - \vec{\bfun}'_{h+1} \|_{\infty} \leq \| \vec{\bfun}_{h+1} - \vec{\bfun}'_{h+1} \|_2$, {\it i.e.} the mapping $\mathbb{R}^A \ni \vec{\bfun}_{h+1}(s) \mapsto V_{f_{h+1}}(s)$ is $1$-Lipschitz. By \Cref{lemma:contraction_RC}, we have
	\[ \popR_n^{\nu_h} \big( V_{\widetilde{\cF}_{h+1}} \big) \leq \sqrt{2} \E \Bigg[ \sup_{f_{h+1} \in \widetilde{\cF}_{h+1}} \sum_{k=1}^n \sum_{a \in \mathcal{A}} \sigma_{k,a} f_{h+1}(s_k',a) \Bigg], \]
	where $s_1', s_2', \ldots, s_n'$ are {\it i.i.d.} samples generated from $\nu_h$.
	Let $a_1', a_2', \ldots, a_n' \in \mathcal{A}$ be random variables such that $\mathbb{P}(a_k' = a \, | \, s_k') = A^{-1}$ for $a \in \cA$. It follows that
	\[ \begin{aligned} & \E \Bigg[ \sup_{f_{h+1} \in \widetilde{\cF}_{h+1}} \sum_{k=1}^n \sum_{a \in \mathcal{A}} \sigma_{k,a} f_{h+1}(s_k',a) \Bigg] \leq A \E \Bigg[ \frac{1}{A} \sum_{a \in \mathcal{A}} \sup_{f_{h+1} \in \widetilde{\cF}_{h+1}} \sum_{k=1}^n \sigma_{k,a} f_{h+1}(s_k',a) \Bigg] \\
	= & A \E \Bigg[ \sup_{f_{h+1} \in \widetilde{\cF}_{h+1}} \sum_{k=1}^n \sigma_{k,a_k'} f_{h+1}(s_k',a_k') \Bigg] = A \popR_n^{\nu_h \times \Uniform(\cA)} \big( \widetilde{\cF}_{h+1} \big). \end{aligned} \]
	Therefore, $\popR_n^{\nu_h} \big( V_{\widetilde{\cF}_{h+1}} \big) \leq \sqrt{2}A\popR_n^{\nu_h \times \Uniform(\cA)} \big( \widetilde{\cF}_{h+1} \big)$.
\end{proof}

% !TEX root = main.tex

\section{Useful Results for (Local) Rademacher Complexity}

In this section, we sumarize some useful results for (local) Rademacher complexity that are used throughout our analysis.

\subsection{Concentration with Rademacher Complexity} % ~ \vspace{.5em}

Lemma \ref{lemma:RC} below shows some uniform concentration inequalities with Rademacher complexity.

\begin{lemma} \label{lemma:RC}
	Let $\mathcal{F}$ be a class of functions with ranges in $[a,b]$. With probability at least $1 - \delta$,
	\[ Pf \leq P_nf + 2 \mathcal{R}_n (\mathcal{F}) + (b-a) \sqrt{\frac{2 \log(2/\delta)}{n}}, \qquad \text{for any $f \in \mathcal{F}$}. \]
	Also, with probability at least $1 - \delta$,
	\[ P_nf \leq Pf + 2 \mathcal{R}_n (\mathcal{F}) + (b-a) \sqrt{\frac{2 \log(2/\delta)}{n}}, \qquad \text{for any $f \in \mathcal{F}$}. \]
\end{lemma}
\begin{proof}
	Consider the empirical process $\sup_{f \in \mathcal{F}}(Pf-P_nf)$. By McDiarmid's inequality, with probability at least $1 - \delta$,
	\begin{equation} \label{McDiarmid} \sup_{f \in \mathcal{F}} \big( Pf - P_nf \big) \leq \mathbb{E} \sup_{f \in \mathcal{F}} \big( Pf - P_nf \big) + (b-a) \sqrt{\frac{2 \log(2/\delta)}{n}}. \end{equation}
	The basic property of Rademacher complexity ensures that
	\begin{equation} \label{RC_basic} \mathbb{E} \sup_{f \in \mathcal{F}} \big( Pf - P_nf \big) \leq 2 \mathcal{R}_n (\mathcal{F}). \end{equation}
	Combining \cref{McDiarmid,RC_basic}, we finish the proof.
\end{proof}

%\vspace{.5em}

\subsection{Concentration with Local Rademacher complexity} % ~ \vspace{.5em}

In this part, we present some auxiliary results regarding local Rademacher complexity. In particular, \Cref{lem:subroot} guarantees the well-definedness of critical radius, \Cref{theorem:LRC} provides concentration inequalities and \Cref{lemma:subroot} gives some useful properties of sub-root functions.

\subsubsection{Well-definedness of critical radius}

Recall that in \Cref{def:critrad}, the critical radius $r^{\star}$ of local Rademacher complexity $\popR_n^{\rho}( \{ f \in \cF \mid T(f) \leq r \} )$ is defined as the possitive fixed point of some sub-root functions $\psi(r)$. The following \Cref{lem:subroot} ensures that $r^{\star}$ exists and is unique.

\begin{lemma}[Lemma 3.2 in \citet{bartlett2005local}] \label{lem:subroot}
	If $\psi: [0,\infty) \rightarrow [0,\infty)$ is a nontrivial sub-root function, then it is continuous on $[0, \infty)$ and the equation $\psi(r) = r$ has a unique positive solution $r^{\star}$. Moreover, for all $r > 0$, $r \geq \psi(r)$ if and only if $r^{\star} \leq r$.
\end{lemma}

\subsubsection{Concentration inequalities}

Throughout the paper, we use \Cref{theorem:LRC} below to prove uniform concentration with local Rademacher complexity. \Cref{theorem:LRC} is a variant of Theorem~3.3 in \citet{bartlett2005local}.

\begin{theorem}[Corollary of Theorem 3.3 in \citet{bartlett2005local}] \label{theorem:LRC}
	Let $\mathcal{F}$ be a class of functions with ranges in $[a,b]$ and assume that there are some functional $T: \mathcal{F} \rightarrow \mathbb{R}^+$ and some constants $B$ and $\eta$ such that for every $f \in \mathcal{F}$, ${\rm Var}[f] \leq T(f) \leq B (P f + \eta)$. Let $\psi$ be a sub-root function and let $r^{\star}$ be the fixed point of $\psi$. Assume that $\psi$ satisfies, for any $r \geq r^{\star}$, $\psi(r) \geq B \mathcal{R}_n\big(\big\{ f \in \mathcal{F} \, \big| \, T(f) \leq r \big\}\big)$. Then for any $\theta > 1$, with probability at least $1 - \delta$,
	\begin{equation} \label{eq:theoremLRC1} Pf \leq \frac{\theta}{\theta-1}P_nf + \frac{c_1\theta}{B} r^{\star} + \big(c_2(b-a) + c_3B\theta\big)\frac{\log(1/\delta)}{n} + \frac{\eta}{\theta-1}, \qquad \text{for any $f \in \mathcal{F}$}. \end{equation}
	Also, with probability at least $1 - \delta$,
	\[ P_nf \leq \frac{\theta+1}{\theta}Pf + \frac{c_1\theta}{B} r^{\star} + \big(c_2(b-a) + c_3B\theta\big)\frac{\log(1/\delta)}{n} + \frac{\eta}{\theta}, \qquad \text{for any $f \in \mathcal{F}$}. \]
	Here, $c_1, c_2, c_3 > 0$ are some universal constants.
\end{theorem}

\begin{proof} % [Proof of Theorem \ref{theorem:LRC}]
	Theorem \ref{theorem:LRC} is proved in the same way as the first part of Theorem 3.3 in \citet{bartlett2005local}, by applying the following Lemma \ref{cite_Lemma3.8} instead of Lemma 3.8 in \citet{bartlett2005local}. 
\end{proof}

Given a class $\mathcal{F}$, $\lambda > 1$ and $r > 0$, let $w(f) := \min\big\{ r \lambda^k \, \big| \, k \in \mathbb{N}, r\lambda^k \geq T(f) \big\}$ and set $\mathcal{G}_r := \big\{ \frac{r}{w(f)}f \, \big| \, f \in \mathcal{F} \big\}$. Define $V_r^+ := \sup_{g \in \mathcal{G}_r} Pg - P_n g$ and $V_r^- := \sup_{g \in \mathcal{G}_r} P_ng - Pg$.
\begin{lemma}[Corollary of Lemma 3.8 in \citet{bartlett2005local}] \label{cite_Lemma3.8}
	Assume that there is a constant $B > 0$ such that for every $f \in \mathcal{F}$, $T(f) \leq B(Pf + \eta)$. Fix $\theta > 1$, $\lambda > 0$ and $r > 0$. If $V_r^+ \leq \frac{r}{\lambda B\theta}$, then $Pf \leq \frac{\theta}{\theta-1}P_nf + \frac{r}{\lambda B \theta} + \frac{\eta}{\theta-1}$. Also, if $V_r^- \leq \frac{r}{\lambda B\theta}$, then $P_nf \leq \frac{\theta+1}{\theta}Pf + \frac{r}{\lambda B \theta} + \frac{\eta}{\theta}$.
\end{lemma}
\begin{proof}
	When $V_r^+ \leq \frac{r}{\lambda B \theta}$, following the same reasoning as Lemma 3.8 in \citet{bartlett2005local}, we derive that $Pf \leq P_nf + \theta^{-1}(Pf + \eta)$ under the modified condition $T(f) \leq B(Pf + \eta)$. It immediately implies the first statement. Similarly, the second part is proved by showing that $P_nf \leq Pf + \theta^{-1} (Pf + \eta)$.
\end{proof}

\subsubsection{Properties of sub-root functions}

We apply the following \Cref{lemma:subroot} to simplify the forms of critical radii.

\begin{lemma} \label{lemma:subroot}
	If $\psi\!:\! [0,\infty) \!\rightarrow\! [0,\infty)$ is a nontrivial sub-root function and $r^{\star}$ is its positive fixed point, then
%	\vspace{-.6em}
	\begin{enumerate} \itemsep = -.2em
		\item $\psi(r) \leq \sqrt{r^{\star} r}$ for any $r \geq r^{\star}$.
		\item For any $c > 0$, $\widetilde{\psi}(r) := c \psi(c^{-1}r)$ is sub-root and its positive fixed point $\widetilde{r}^{\star}$ satisfies $\widetilde{r}^{\star} = c r^{\star}$.
		\item For any $C > 0$, $\widetilde{\psi}(r) := C \psi(r)$  is sub-root and its positive fixed point $\widetilde{r}^{\star}$ satisfies $\widetilde{r}^{\star} \leq (C^2 \vee 1) r^{\star}$.
		\item For any $\Delta r \!>\! 0$, $\widetilde{\psi}(r) := \psi(r\!+\!\Delta r)$  is sub-root and its positive fixed point $\widetilde{r}^{\star}$ satisfies $\widetilde{r}^{\star} \!\leq\! r^{\star} \!+\! \sqrt{r^{\star} \Delta r}$.
	\end{enumerate}
%	\vspace{-.5em}
	If $\psi_i\!\!:\!\! [0,\infty) \!\rightarrow\! [0,\infty)$, $\!i\!=\!\! 1,...,\!n$ are nontrivial sub-root functions and $r_i^{\star}\!$ is the positive fixed point of $\psi_i$, then
%	\vspace{-.8em}
	\begin{enumerate}
		\setcounter{enumi}{4}
		\item $\widetilde{\psi}(r) = \sum_{i=1}^n \psi_i(r)$ is sub-root and its positive fixed point $\widetilde{r}^{\star}$ satisfies $\widetilde{r}^{\star} \leq \big( \sum_{i=1}^n \sqrt{r_i^{\star}} \big)^2$.
	\end{enumerate}
\end{lemma}
\begin{proof}
	1. Since $\psi$ is a sub-root function, we have $\frac{\psi(r)}{\sqrt{r}} \leq \frac{\psi(r^{\star})}{\sqrt{r^{\star}}}$ for any $r \geq r^{\star}$. Note that $r^{\star} > 0$ is the fixed point and $\frac{\psi(r^{\star})}{\sqrt{r^{\star}}} = \sqrt{r^{\star}}$. Therefore, $\psi(r) \leq \sqrt{r^{\star} r}$ for $r \geq r^{\star}$.
	
	2. It is evident that $\widetilde{\psi}$ is sub-root. Additionally, if $r \geq c r^{\star}$, then by \Cref{lem:subroot}, we have $\widetilde{\psi}(r) = c \psi(c^{-1}r) \leq c (c^{-1} r) = r$. In contrast, if $0 < r < c r^{\star}$, then $\widetilde{\psi}(r) = c \psi(c^{-1} r) > c(c^{-1}r) = r$. To this end, we can conclude that $\widetilde{r}^{\star} = c r^{\star}$.
	
	3. We use part 1 and derive that if $\widetilde{r}^{\star} \geq r^{\star}$ then $\widetilde{r}^{\star} = \widetilde{\psi}(\widetilde{r}^{\star}) = C \psi(\widetilde{r}^{\star}) \leq C \sqrt{r^{\star}\widetilde{r}^{\star}}$, which further implies $\widetilde{r}^{\star} \leq C^2 r^{\star}$. Therefore, $\widetilde{r}^{\star} \leq (C^2 \vee 1) r^{\star}$.
	
	4. If $\widetilde{r}^{\star} + \Delta r \geq r^{\star}$, then we have $\widetilde{r}^{\star} = \widetilde{\psi}(\widetilde{r}^{\star}) = \psi(\widetilde{r}^{\star} + \Delta r) \leq \sqrt{r^{\star}(\widetilde{r}^{\star} + \Delta r)}$ due to part 1. It follows that $\widetilde{r}^{\star} \leq \frac{1}{2}\big(r^{\star} + \sqrt{(r^{\star})^2 + 4r^{\star}\Delta r}\big) \leq r^{\star} + \sqrt{r^{\star} \Delta r}$. 
		
	5. If $\widetilde{r}^{\star} \geq \max_{i \in [n]} r_i^{\star}$, then we apply part 1 and obtain $\widetilde{r}^{\star} = \widetilde{\psi}(\widetilde{r}^{\star}) = \sum_{i=1}^n \psi_i(\widetilde{r}^{\star}) \leq \sum_{i=1}^n \sqrt{r_i^{\star} \widetilde{r}^{\star}}$. Hence, $\widetilde{r}^{\star} \leq \big( \sum_{i=1}^n \sqrt{r_i^{\star}} \big)^2$.
\end{proof}

\subsection{Contraction property of Rademacher complexity}

Our analyses use contraction properties of Rademacher complexity. See \Cref{lemma:contraction_RC_0,lemma:contraction_RC}.

\begin{lemma}[Contraction property of Rademacher complexity,  \citet{ledoux2013probability}, Theorem A.6 in \citet{bartlett2005local}] \label{lemma:contraction_RC_0} Suppose $\mathcal{F} \subseteq \{ f: \mathcal{X} \rightarrow \mathbb{R} \}$. Let $\phi: \mathbb{R} \rightarrow \mathbb{R}$ be a contraction such that $|\phi(x) - \phi(y)| \leq |y-y'|$ for any $y,y' \in \mathbb{R}$. Then for any $X_1, X_2, \ldots, X_n \in \mathcal{X}$, \vspace{-.5em}
\[ \widehat{\mathcal{R}}_X(\phi \circ \mathcal{F}) = \mathbb{E}_{\boldsymbol{\sigma}}\bigg[ \sup_{f \in \mathcal{F}} \frac{1}{n} \sum_{i=1}^n \sigma_i \phi\big( f(X_i) \big) \bigg] \leq \mathbb{E}_{\boldsymbol{\sigma}}\bigg[ \sup_{f \in \mathcal{F}} \frac{1}{n} \sum_{i=1}^n \sigma_i f(X_i) \bigg] = \widehat{\mathcal{R}}_X(\mathcal{F}). \]
	
\end{lemma}

\begin{lemma}[Vector-form contraction property of Rademacher complexity, \citet{maurer2016vector}] \label{lemma:contraction_RC}
	Suppose $\boldsymbol{\mathcal{F}}$ is a collection of vector-valued functions $\boldsymbol{f}: \mathcal{X} \rightarrow \mathbb{R}^d$ and $h: \mathbb{R}^d \rightarrow \mathbb{R}$ is $L$-Lipschitz with respect to the Euclidean norm, {\it i.e.} $\big|h(y) - h(y')\big| \leq L \|y-y'\|_2$ for any $y,y' \in \mathbb{R}^d$. Then for any $X_1, X_2, \ldots, X_n \in \mathcal{X}$, \vspace{-.5em}
	\[ \begin{aligned} \widehat{\mathcal{R}}_X(h \circ \boldsymbol{\mathcal{F}}) = & \mathbb{E}_{\boldsymbol{\sigma}} \bigg[ \sup_{\boldsymbol{f} \in \boldsymbol{\mathcal{F}}} \frac{1}{n} \sum_{i=1}^n \sigma_i h\big(\boldsymbol{f}(X_i)\big) \bigg] \\ \leq & \sqrt{2} L \mathbb{E}_{\boldsymbol{\sigma}} \Bigg[ \sup_{\boldsymbol{f} \in \boldsymbol{\mathcal{F}}} \frac{1}{n} \sum_{i=1}^n \sum_{j=1}^d \sigma_{i,j}f_j(X_i) \Bigg] \leq \sqrt{2}L \sum_{j=1}^d \widehat{\mathcal{R}}_X\big(\big\{ f_j \, \big| \, \boldsymbol{f} \in \boldsymbol{\mathcal{F}} \big\}\big). \end{aligned} \]
\end{lemma}

\end{document}